\documentclass[12pt]{article}
\usepackage[labelfont=bf,font={small}]{caption}
\usepackage[utf8]{inputenc}
\usepackage{fullpage}
\usepackage[T1]{fontenc}
\usepackage{authblk}
\usepackage{natbib}
\usepackage[colorlinks=true,citecolor=blue]{hyperref}       
\usepackage{url}            
\usepackage{booktabs}       
\usepackage{amsfonts}       
\usepackage{nicefrac}       
\usepackage{microtype}      
\usepackage{lipsum}
\usepackage{setspace}
\usepackage{authblk}

\usepackage{mathtools}
\usepackage{graphicx}
\usepackage{amssymb}
\usepackage{amsthm}
\usepackage{xpatch}
\usepackage{mathrsfs}

\usepackage{url}
\usepackage{array}
\usepackage{wrapfig}
\usepackage{multirow}
\usepackage{tabularx}
\usepackage[normalem]{ulem} 

\usepackage[usenames,dvipsnames]{xcolor}
\newcommand\cut[1]{{}}  
\makeatletter
\xpatchcmd{\algorithmic}{\itemsep\z@}{\itemsep=.5ex plus.5pt}{}{}
\makeatother
\usepackage{mathtools}
\mathtoolsset{showonlyrefs}

\newcommand{\vI}{{\mathbf{I}}}

\newcommand{\vP}{{\mathbf{P}}}

\newcommand{\vU}{{\mathbf{U}}}
\newcommand{\vV}{{\mathbf{V}}}

\newcommand{\sA}{{\mathscr{A}}}
\newcommand{\cA}{{\mathcal{A}}}
\newcommand{\cB}{{\mathcal{B}}}

\newcommand{\cE}{{\mathcal{E}}}
\newcommand{\cF}{{\mathcal{F}}}

\newcommand{\cM}{{\mathcal{M}}}

\newcommand{\cO}{{\mathcal{O}}}

\newcommand{\cQ}{{\mathcal{Q}}}

\newcommand{\cS}{{\mathcal{S}}}

\newcommand{\cV}{{\mathcal{V}}}

\newcommand{\cX}{{\mathcal{X}}}
\newcommand{\cY}{{\mathcal{Y}}}

\newcommand{\RR}{\mathbb{R}}

\newcommand{\NN}{\mathbb{N}}

\newcommand{\Proj}{{\mathrm{Proj}}}

\renewcommand{\iota}{I}

\DeclareMathOperator*{\argmax}{argmax}

\newcommand{\bc}{\begin{center}}
\newcommand{\ec}{\end{center}}

\newcommand{\bdm}{\begin{displaymath}}
\newcommand{\edm}{\end{displaymath}}

\newcommand{\beq}{\begin{equation}}
\newcommand{\eeq}{\end{equation}}

\newcommand{\bfl}{\begin{flushleft}}
\newcommand{\efl}{\end{flushleft}}

\newcommand{\bt}{\begin{tabbing}}
\newcommand{\et}{\end{tabbing}}

\newcommand{\beqn}{\begin{align}}
\newcommand{\eeqn}{\end{align}}

\newcommand{\beqs}{\begin{align*}} 
\newcommand{\eeqs}{\end{align*}}  

\providecommand{\keywords}[1]
{
  \small	
  \textbf{\textit{Keywords---}} #1
}

\newtheorem{theorem}{Theorem}
\newtheorem{assumption}{Assumption}
\newtheorem{definition}{Definition}
\newtheorem{corollary}{Corollary}

\newtheorem{lemma}{Lemma}
\newtheorem{proposition}{Proposition}

\newtheorem{problem}{Problem}

\makeatletter
\renewcommand*\env@matrix[1][\arraystretch]{%
  \edef\arraystretch{#1}%
  \hskip -\arraycolsep
  \let\@ifnextchar\new@ifnextchar
  \array{*\c@MaxMatrixCols c}}
\makeatother

\newcommand{\dist}{\mathrm{dist}}
\newcommand{\poly}{\mathrm{poly}}
\graphicspath{{./image/}}
\usepackage{algorithm}
\usepackage{algpseudocode}

\newcommand{\lin}[1]{\textcolor{red}{[Lin: #1]}}
\newcommand{\fei}[1]{\textcolor{orange}{[Fei: #1]}}
\allowdisplaybreaks
\title{How Does an Approximate Model Help in Reinforcement Learning?}

\author[1]{Fei Feng\thanks{fei.feng@math.ucla.edu}}
\author[1]{Wotao Yin\thanks{ wotaoyin@math.ucla.edu}}
\author[2]{Lin F. Yang\thanks{linyang@ee.ucla.edu, corresponding author.}}
\affil[1]{Department of Mathematics, UCLA}
\affil[2]{Department of  Electrical and Computer Engineering, UCLA}

\date{}
\begin{document}
\maketitle

\begin{abstract}
One of the key approaches to save samples in reinforcement learning (RL) is to use knowledge from an approximate model such as its simulator.
However, \emph{how much does an approximate model help to learn a near-optimal policy of the true unknown model}?
Despite numerous empirical studies of transfer reinforcement learning, an answer to this question is still elusive.
In this paper,
we study the sample complexity of RL while an approximate model of the environment is provided.
For an unknown Markov decision process (MDP), we show that the approximate model can effectively reduce the complexity by eliminating sub-optimal actions from the policy searching space.
In particular, we provide an algorithm that uses $\widetilde{O}(N/(1-\gamma)^3/\varepsilon^2)$ samples in a generative model to learn an $\varepsilon$-optimal policy, where $\gamma$ is the discount factor and $N$ is the number of near-optimal actions in the approximate model. This can be much smaller than the learning-from-scratch complexity $\widetilde{\Theta}(SA/(1-\gamma)^3/\varepsilon^2)$, where $S$ and $A$ are  the sizes of state and action spaces respectively.
We also provide a lower bound showing that the above upper bound is nearly-tight if the value gap between near-optimal actions and sub-optimal actions in the approximate model is sufficiently large. Our results provide a very precise characterization of how an approximate model helps reinforcement learning when no additional assumption on the model is posed.

\end{abstract}


\section{Introduction}
Reinforcement learning (RL) is the framework of learning to control an unknown system through trial and error. Recently, RL achieves phenomenal empirical successes, e.g, AlphaGo \citep{silver2016mastering} defeated the best human player in Go, and OpenAI used RL to precisely and robustly control a robotic arm \citep{NIPS2017_7090}.
The RL framework is general enough such that it can capture a broad spectrum of topics, including health care, traffic control, and experimental design \citep{sutton1992reinforcement, esteva2019guide, si2001online, wiering2000multi}. However,
\emph{successful} applications of RL in these domains are still rare.
The major obstacle that prevents RL being widely used is its high sample complexity:
both the AlphaGo and OpenAI arm took nearly a thousand years of human-equivalent experiences to achieve good performances.

One way to reduce the number of training samples is to mimic how human beings learn -- borrow knowledge from previous experiences.
In robotics research, a robot may need to accomplish different tasks at different times.
Instead of learning every task from scratch, a more ideal situation is that the robot can utilize the similarities between the underlying models of these tasks and adapt them to future new jobs quickly.
Another example is that RL agents are often trained in simulators and then applied to the real world \citep{ng2006autonomous,itsuki1995soccer,dosovitskiy2017carla}.
It is still desirable to have their performance improved after seeing samples collected from the real world.
One might hope that agents from simulators (approximate models) can adapt to the real world (true model) faster than knowing nothing. Both examples lead to a natural question:
\begin{quote}
\centering
    \textit{How does an approximate model help in RL?}
\end{quote}
This paper focuses on answering the above question. Suppose the true unknown model is a Markov Decision Process (MDP)  $\cM$ and the RL agent is provided with a prior model $\cM_0$ with the same state and action spaces as $\cM$ but different transition and reward functions. In particular, we have
$$\dist(\cM_0, \cM)\le \beta,$$ where $\dist(\cdot, \cdot)$ is a statistical distance and  $\beta$ is a small scalar. We would like to study the sample complexity of learning a
policy $\pi$ for $\cM$ such that its error\footnote{The error of a policy is the difference between the values of the policy and the optimal policy.} is at most $\varepsilon$.

In this paper, we consider one of the most natural choices for $\dist(\cdot, \cdot)$, the total-variation (TV) distance between the transition kernels of $\cM_0$ and $\cM$ (see the formal definition in Equation \eqref{eq:TV}). Under such a distance, we establish both upper and lower sample complexity bounds. Specifically, we utilize the fact that if two models are close under TV-distance, their optimal value functions are close under $\|\cdot\|_{\infty}$. Based on this, given the prior knowledge, for every state in $\cM$ we can split its available actions into two sets: one contains all \emph{potential-optimal} actions and the other contains all \emph{sub-optimal} actions. An action is potential-optimal if it has a chance to be optimal in $\cM$; otherwise, it is sub-optimal.
We develop a transfer RL algorithm that first identifies potential-optimal actions using the prior knowledge, then focuses exploration only on them. 
To obtain an $\varepsilon$-optimal policy, our algorithm enjoys a sample bound $\widetilde{\cO}\big(\frac{\overline{N}}{(1-\gamma)^3\varepsilon^2}\log(1/\delta)\big)$, where $\overline{N}$ equals the total number of potential-optimal actions over all states. $\overline{N}$ depends on the distance parameter $\beta$, the accuracy requirement $\varepsilon$, and the optimal value function of $\cM_0$. When $\overline{N}$ is smaller than the total number of state-action pairs in $\cM$, transfer learning achieves a smaller sample complexity than learning-from-scratch.

To complement our result, we further develop a sample lower bound  $\Omega\big(\frac{\underline{N}}{(1-\gamma)^3\varepsilon^2}\log(1/\delta)\big)$, where $\underline{N}$ characterizes the number of actions we must explore or we fail to learn a near-optimal policy with high probability for some $\cM\in B_{\text{TV}}(\cM_0, \beta)$. Roughly speaking, the more near-optimal actions there are in $\cM_0$, the more we shall consider and the larger $\underline{N}$ is. In particular, we have $\overline{N}\approx \underline{N}$ when there exists a large value gap in $\cM_0$. The reader is referred to Section \ref{sec:meet} for more discussion.
This lower bound essentially shows that our lower upper bound is nearly tight.
To achieve the lower bound, we leverage techniques for proving hardness in the bandit literature (e.g. \citealt{mannor2004sample}) and reinforcement learning (e.g. \citealt{azar2013minimax}).

Our results do not rely on additional structure assumptions but only the intrinsic value property of MDPs. To the best of our knowledge, this is the first systematic theoretical answer to the aforementioned question in this setting.
\cut{

Reinforcement learning (RL) is the framework of learning to control an unknown system through trial and error. Recently, RL achieves phenomenal empirical successes, e.g, AlphaGo \citep{silver2016mastering} defeated the best human player in Go, 
and OpenAI used RL to precisely and robustly control a robotic arm \citep{NIPS2017_7090}.
The RL framework is general enough such that it can capture a broad spectrum of topics, including health care, traffic control, and experimental design \citep{sutton1992reinforcement, esteva2019guide, si2001online, wiering2000multi, denil2016learning}.
However,
\emph{successful} applications of RL in these domains are still rare.
The major obstacle that prevents RL being widely used is its high sample complexity:
both the AlphaGo and OpenAI arm took nearly a thousand years of human-equivalent experiences to achieve good performances.

One way to reduce the number of training samples is to mimic how human beings learn -- borrow knowledge from previous experiences. 
In robotics research, a robot may need to accomplish different tasks at different time. 
Instead of learning every task from scratch, a more ideal situation is that the robot can utilize the similarities between the underlying models of these tasks and adapt them to future new jobs quickly.
Another example is that RL agents are often trained in simulators and then applied to real-world \citep{ng2006autonomous,itsuki1995soccer,dosovitskiy2017carla}.
It is still desirable to have their performance improved after seeing samples collected from the real-world.
One might hope that agents from simulators (approximate models) can adapt to the real world (true model) faster than knowing nothing.
Both examples lead to a natural question:
\begin{quote}
\centering
    \textit{How does an approximate model help in RL?}
\end{quote}


This paper focuses on answering the above question.
Suppose the true unknown model is a Markov Decision Process (MDP)  $\cM$ and the RL agent is provided with an approximate model $\cM_0$ with
$$\dist(\cM_0, \cM)\le \beta,$$ where $\dist(\cdot, \cdot)$ is a statistic distance and  $\beta$ is a small scalar.
We would like to study the sample complexity of learning a
policy $\pi$ for $\cM$ such that its error\footnote{The error of a policy is the difference between the values of the policy and the optimal policy.} is at most $\varepsilon$. 

We consider one of the most natural choices for $\dist(\cdot,\cdot)$, the total-variation (TV) distance between the transition kernels of $\cM_0$ and $\cM$ (see the formal definition in Section \ref{sec:pre}). Then, from e.g. \cite{puterman2014markov}, we can deduce that the optimal action-value functions of $\cM_0$ and $\cM$ are at most $\cO(\beta/(1-\gamma)^2)$ apart. Thus, given the full knowledge of $\cM_0$, we can compute a preliminary estimate of the action values of $\cM$. Such prior knowledge helps to eliminate sub-optimal actions and in later learning stage, one only needs to explore the remaining potential-optimal actions. Denote by $N$ the number of remaining actions. Given $\cM_0$, $\beta$, and $\varepsilon$, we can build estimate of $N$, i.e. $N\in[N^1_{\cM_0}, N^2_{\cM^0}]$. On one hand, we combine the prior knowledge with some mini-max optimal RL algorithm (e.g. \cite{sidford2018variance}) to establish the upper sample complexity bound as $\widetilde{\cO}(N^2_{\cM_0}/(1-\gamma)^3/\varepsilon^2)$ when a generative model is provided; on the other hand, we show a general lower bound $\Omega(N^1_{\cM_0}/(1-\gamma)^3/\varepsilon^2)$ no matter what algorithm or learning setting is applied. In particular, we have $N^2_{\cM_0}\approx N^1_{\cM_0}$ when the value gap between near-optimal actions and sub-optimal actions in $\cM_0$ is large and therefore, the upper bound is nearly tight. Our results include two extreme cases: 1. $N^2_{\cM_0}=0$, then no learning is needed and direct policy transfer is optimal; 2. $N^1_{\cM_0}\approx SA$, then the prior knowledge is basically useless and the sample complexity is equivalent to learning-from-scratch.

Our analysis does not rely on any additional structural assumptions but only the intrinsic value property of $\cM_0$. The proof of the upper bound is straightforward. We leverage techniques for proving hardness in the bandit literature (e.g. \citealt{mannor2004sample}) and reinforcement learning (e.g. \citealt{azar2013minimax}) to achieve the lower bound. To the best of our knowledge, this is the first systematic theoretical answer to the aforementioned question.
}

\cut{To complement the lower bound, we further investigate the possible structural information of a model that provably helps knowledge transferring.
We show that if the unknown model is in the convex hull of a set of $K$ known base models, we are able to obtain a high-precision control with a number of samples significantly fewer than that of learning from scratch.
Specifically, the number of samples is proportional to \[\cO(\poly(K))\] rather than the much larger $|\cS|$, the number of states in the model.}

\subsection{Related Work}
Transfer learning is an important strategy to reduce sample complexity in RL.
There are many different learning regimes, e.g, multi-task RL \citep{wilson2007multi, lazaric:inria-00475214, brunskill2013sample, ammar2014online, calandriello2014sparse}, lifelong RL \citep{tanaka1997approach, brunskill2014pac, abel2018policy, zhan2017scalable}, and meta-RL \citep{schweighofer2003meta, al-shedivat2018continuous, gupta2018meta, saemundsson2018meta, rakelly2019efficient}. Please also see surveys in \citealt{taylor2009transfer, lazaric2012transfer} and \citealt{yang2020transfer}.

In the above settings, often more than one prior model (task) is considered. These models are assumed to share structural similarity with the to-be-learned model, or they are all generated from a common distribution. For instance, in \citealt{brunskill2013sample}, all models are assumed to be drawn from a finite set of MDPs;
in \citealt{abel2018policy}, all models share the same transition dynamics but reward functions change with a hidden distribution;
in \citealt{modi2019sample} and \citealt{ayoub2020model}, every model's transition kernel and reward function lie in the linear span of several known base models.

In terms of how one approximate model can help, there are several theoretical works. In \citealt{jiang2018pac},
the authors use the number of incorrect state-action pairs to characterize the difference between two models, which is another interesting direction to look at besides the TV-distance as we adopt. However, the statistical distance from such a model to the true model can be arbitrarily large. As the authors show, even if the difference is only one state-action pair, the approximate model could still be information-theoretically useless. Additional conditions are needed to achieve positive transfer. In \citealt{mann2012directed}, the authors analyzed action-value transfer via inter-task mappings to guide exploration. They do not measure the difference between models but show that as long as the optimal action function of the source task is (almost) larger than that of the target task component-wisely, one can use the former as an initialization for value-based RL on the true model and achieve a smaller sample complexity compared compared with learning-from-scratch. Their working mechanism is similar to ours: eliminate sub-optimal state-action pairs from later exploration. The difference is we directly identify and remove sub-optimal actions and they achieve this implicitly through value functions. No lower bound is established there.


Compared with the aforementioned works, our results do not rely on additional assumptions and are more complete with both upper and lower bounds. In particular, our upper and lower bounds enjoy a unifying structure that is explicitly characterized by the value function, the TV-distance parameter, and the learning precision parameters.

\cut{\lin{I found it is a bit hard to understand what's going on in the following paragraphs. I have rewritten it above.}
As for successfully sharing knowledge, our assumption \lin{what assumpiton?} controls the number of parameters to learn. Such an idea of restricting models to a simpler space is adopted by a large body of works. In \citealt{brunskill2013sample}, there are only finitely many models showing up repetitively\lin{I don't understand this.}. During learning, one can first identify the new model, then borrow previous samples. In \citealt{calandriello2014sparse}, all tasks' value functions can be accurately represented in a linear approximation space and the coefficients are jointly sparse. In \citealt{modi2019sample}, every model's transition and reward lies in the linear span of $K$ base models, where the coefficients depend on some state-related feature space. Our setting is a special example of \citealt{modi2019sample}. While we study infinite horizon MDPs with a generative model to take samples and our objective function uses $l_2-$distance instead of $l_1$.}


\cut{
Reducing sample complexity is a core research goal in RL.
Many related sub-branches of RL, e.g., multi-task RL \citep{brunskill2013sample, ammar2014online, calandriello2014sparse}, lifelong RL \citep{abel2018policy, brunskill2014pac}, and meta-RL \citep{al2017continuous}, provide different schemes to utilize experiences from previous tasks. Please also see a survey paper \citep{taylor2009transfer} for more related works.
However, these results focus on different special cases of knowledge utilization rather than understanding the fundamental question of whether an approximate model is useful for policy learning and what guarantees we can have.

In the area of Sim-to-Real\footnote{It stands for simulator-to-real-environment}, some works point out that an imperfect approximate model may degrade the performance of learning and efforts have been made to address this issue, e.g \citep{kober2013reinforcement, buckman2018sample,kalweit2017uncertainty,kurutach2018model}. There has been active empirical research, but little in theory is known. A more related work is \citealt{jiang2018pac},
who shows that even if the approximate model differs from the real environment in a single state-action pair (but which one is unknown), such an approximate model could still be information-theoretically useless. 
This is another interesting direction to look at. However, the statistic distance from such a model to the true model can be arbitrarily large and hence the policy of the approximate model does not have a guarantee on the true model.
The limitation of the benefit that an approximate model could bring can also be found in \citealt{jiang2015doubly}, where the authors build a policy value estimator and use the approximate model to reduce variance. However, they demonstrate that, if no extra knowledge is provided, only the part of variance arising from the randomness in policy can be eliminated rather than the stochasticity in state transitions.

In order to take more advantage of the previous experiences, additional structure information is needed. A number of structure settings have been studied in the literature.
For instance, in \citealt{brunskill2013sample}, all models are assumed to be drawn from a finite set of MDPs with identical state and action spaces, but different reward and/or transition probabilities; in \citealt{abel2018policy}, one study case requires that all models share the same transition dynamics and only reward functions change with a hidden distribution; in \citealt{mann2012directed}, a special mapping between the approximate model and the true model is assumed such that the approximate model can provide a good action-value initialization for the true model;
in \citealt{calandriello2014sparse}, all tasks can be accurately represented in a linear approximation space and the weight vectors are jointly sparse; in \citealt{modi2019sample}, every model's transition kernel and reward function lie in the linear span of $K$ known base models.
To complement our lower bound, we study an MDP model that shares a similar information structure as in  \citealt{modi2019sample}.
In contrast to \citealt{modi2019sample}, our model is of  infinite-horizon and the loss function is also  different.
Although not the main focus of this paper, our proposed model and algorithm provide another effective approach that supports knowledge transferring.
}
\cut{
As for successfully sharing knowledge, our assumption \lin{what assumpiton?} controls the number of parameters to learn. Such an idea of restricting models to a simpler space is adopted by a large body of works. In \citealt{brunskill2013sample}, there are only finitely many models showing up repetitively\lin{I don't understand this.}. During learning, one can first identify the new model, then borrow previous samples. In \citealt{calandriello2014sparse}, all tasks' value functions can be accurately represented in a linear approximation space and the coefficients are jointly sparse. In \citealt{modi2019sample}, every model's transition and reward lies in the linear span of $K$ base models, where the coefficients depend on some state-related feature space. Our setting is a special example of \citealt{modi2019sample}. While we study infinite horizon MDPs with a generative model to take samples and our objective function uses $l_2-$distance instead of $l_1$.}


\paragraph{Outline}
This paper is organized as follows.
In Section~\ref{sec:pre} we introduce the necessary background and the formal setting.
In Section~\ref{sec:main}, we formalize the problem and state the main results.
In Section~\ref{sec:analysis}, we provide analysis to establish the main results.
We have more graphical illustration in Section~\ref{sec:discussion} and we conclude in Section~\ref{sec:conclusion}.

\section{Preliminary}\label{sec:pre}
Given an integer $K$, we use $[K]$ to represent the set $\{1,2,\dots,K\}$. We use $|\cdot|$ to denote the cardinality of a set.
We
denote by $\cO, \Omega$ and $\Theta$ to denote leading orders in upper, lower, and minimax lower bounds, respectively; and we use $\widetilde{\cO}, \widetilde{\Omega}$ and $\widetilde{\Theta}$ to hide the polylog factors.

\paragraph{Markov Decision Process}
We consider an infinite-horizon discounted Markov Decision Process (MDP), $\cM:=(\cS, \{\cA^s\}_{s\in\cS}, p,r, \gamma)$, 
where $\mathcal{S}$ is a finite state space, $\cA^s$ is the set of available actions for state $s$, $p(s'|s,a)$ is the transition function, $r(s,a,s')\in [0,1]$ is the reward function, and $\gamma\in(0,1)$ is a discount factor. We denote by $\cS'$ the set of states with more than one available action.


At step $t$, the controller observes a state $s_t$ and selects an action $a_t \in \mathcal{A}^{s_t}$ according to a \emph{policy} $\pi$, which maps a state to one of its available action. The environment then transitions to a new state $s_{t+1}$ with probability $p(s_{t+1}|s_t,a_t)$ and the controller receives an instant reward $r(s_{t},a_t,s_{t+1})$. Given a policy $\pi$, we define its value function as: $V^{\pi}(s) := \mathbb{E}^\pi\big[\sum_{t=0}^\infty \gamma^{t} r(s_t ,a_t, s_{t+1})~|~s_0=s\big],$
where the expectation is taken over the trajectory following $\pi$. The objective of RL is to learn a policy $\pi^*$ that can maximize the value function, i.e., $\forall~ \pi, s\in\cS:~V^*:=V^{\pi^*}(s)\geq V^{\pi}(s).$
A policy $\pi$ is said to be \emph{$\varepsilon$-optimal} if
$~V^{\pi}(s) \geq V^*(s) - \varepsilon$ for all $s\in\cS$.
The action-value function (or $Q$-function) of a policy $\pi$ is defined as $$Q^{\pi}(s,a) := \sum_{s'\in\cS}p(s'|s,a) \cdot \big( r(s,a,s') +\gamma V^{\pi}(s')\big).$$
The optimal $Q$-function is denoted by $Q^* := Q^{\pi^*}$. By Bellman Optimality Equation, we have $$V^*(s) = \max_{a\in\cA^s} Q^*(s,a) = \max_{a\in\cA^s} \sum_{s'\in\cS}p(s'|s,a) \cdot \big( r(s,a,s') +\gamma V^{*}(s')\big).$$
\cut{\paragraph{Value Gap}
Given an MDP,
for every state $s$ we define a vector $q^{s}$ with
\begin{align}
   q^s[1]&=\max_{a\in\cA^s}Q^*(s,a), \text{~and}\\
   q^s[i]&= \max \{Q^*(s,a)~|~ Q^*(s,a)< q^s_{\cM}[i-1], a\in\cA^s\}, ~ 1<i\leq |\cA^s|.
\end{align}
By definition, we have $\max_{a\in \cA^s} Q^*(s,a) = q^s[1]>q^s[2]>\cdots>q^s[k^s]=\min_{a\in\cA^s} Q^*(s,a),$ where $k^s$ equals the number of different values in $\{Q^*(s,a), a\in\cA^s\}$. Next, for every state $s$ we define the value-gap vector $\Delta^s$ such that
$$\Delta^s[i] = q^s[1]-q^s[i], ~1\leq i \leq k^s.$$
It holds that $0 = \Delta^s[1]<\Delta^s[2]<\cdots<\Delta^s[k^s]=\max_{a\in\cA^s} Q^*(s,a)-\min_{a\in\cA^s} Q^*(s,a).
$
Lastly, }
Given an MDP $\cM$ and a constant $c$, for each state $s$, we define the following set \begin{align}\label{eq:As}
    \cA^s_{\cM}(c):=\begin{cases}
    \{a~|~V_{\cM}^*(s) - Q_{\cM}^*(s,a) < c\}, &c>0;\\
    \argmax_a Q^*_{\cM}(s,a), &c\leq 0,
    \end{cases}
\end{align}
where $V^*_{\cM}$ and $Q^*_{\cM}$ denote the optimal value function and $Q$-function of $\cM$, respectively. For any value of $c$, $\cA^s_{\cM}(c)$ is well-defined and always contains all optimal actions for state $s$ in $\cM$. If $C>c$, then $\cA^s_{\cM}(C)\supseteq\cA^s_{\cM}(c)$. If $c>1/(1-\gamma)$, $\cA_{\cM}^s(c)=\cA^s_{\cM}$, i.e., all available actions for state $s$ in $\cM$.

\paragraph{TV-distance for MDPs}
Given $\cM_0:=(\cS,\{\cA^s\}_{s\in\cS},p_0,r_0,\gamma)$ and $\cM:=(\cS,\{\cA^s\}_{s\in\cS},p,r,\gamma)$, we define
\begin{align}\label{eq:TV}
   d_{\mathrm{TV}}(\cM_0, \cM) = \max\big\{
   \max_{s\in\cS,a\in\cA^s}\|p_0(\cdot|s,a)-p(\cdot|s,a)\|_1, ~ \|r_0-r\|_{\infty}~\big\}.
\end{align}
$d_{\text{TV}}(\cdot,\cdot)$ is a metric between MDPs with the same state/action space and the discount factor. The name TV comes from that $\|p_0(\cdot|s,a)-p(\cdot|s,a)\|_1$ is equal to the \emph{total variation} distance between distributions $p_0(\cdot|s,a)$ and $p(\cdot|s,a)$.
We denote by $\cM\in B_{\mathrm{TV}}(\cM_0,\beta)$ if $d_{\mathrm{TV}}(\cM_0,\cM)\leq \beta$.

\paragraph{Generative Model}
Generative model is a special sample oracle. It allows any state-action pair $(s,a)$ as input, where $s\in\cS$ and $a\in\cA^s$, and outputs $(s',r(s,a,s'))$ with probability $p(s'|s,a)$.
\paragraph{Near-Optimal RL Algorithm} An RL algorithm is near-optimal if without any prior knowledge, it returns an $\varepsilon$-optimal policy for an MDP $\cM:=(\cS, \{\cA^s\}_{s\in\cS}, p,r, \gamma)$ with probability at least $1-\delta$ using $\widetilde{O}\left(\frac{\sum_{s\in\cS}|\cA^s|}{(1-\gamma)^3\varepsilon^2}\log(1/\delta)\right)$ samples. Near-optimal RL algorithms often require a generative model. Examples can be found in \citealt{azar2013minimax} and \citealt{NIPS2018_7765}.

\cut{

Given an integer $K$, we use $[K]$ to represent the set $\{1,2,\dots,K\}$. We use $|\cdot|$ to denote the cardinality of a set.
We
denote by $\cO, \Omega$ and $\Theta$ to denote leading orders in upper, lower, and minimax lower bounds, respectively; and we use $\widetilde{\cO}, \widetilde{\Omega}$ and $\widetilde{\Theta}$ to hide the polylog factors.

\paragraph{Markov Decision Process}
We consider the setting of an infinite-horizon discounted Markov Decision Process (MDP), $\cM:=(\cS, \cA, p,r, \gamma)$, 
where $\mathcal{S}$ is a finite state space, $\mathcal{A}$ is a finite action space, $p(s'|s,a)$ is the transition function, $r: \cS\times \cA\times\cS\rightarrow [0,1]$ is the reward function, and $\gamma\in(0,1)$ is a discount factor. We denote by $\cA^s$ the set of available actions for state $s$ and $\cS'$ the set of states with more than one available action.
We use $N$ for the total number of state-action pairs.

At step $t$, the controller observes a state $s_t$ and selects an action $a_t \in \mathcal{A}^{s_t}$ according to a \emph{policy} $\pi$, which maps a state to an available action. The environment then transitions to a new state $s_{t+1}$ with probability $p(s_{t+1}|s_t,a_t)$ and the controller receives an instant reward $r(s_{t},a_t,s_{t+1})$. Given a policy $\pi: \mathcal{S} \rightarrow \mathcal{A}$, we define its value function $V^{\pi}:\cS\rightarrow \RR$ as:
\begin{equation}\label{eq:Vdef}
V^{\pi}(s) := \mathbb{E}^\pi\bigg[\sum_{t=0}^\infty \gamma^{t} r(s_t ,a_t, s_{t+1})|s_0=s\bigg],
\end{equation}
where the expectation is taken over the trajectory following $\pi$. The objective of RL is to learn an optimal policy $\pi^*$ such that its value (on any state $s\in \cS$) is maximized over all policies, i.e.,
\[
\forall~ \pi, s\in\cS:\quad V^*:=V^{\pi^*}(s)\geq V^{\pi}(s).
\]
A policy $\pi$ is said to be \emph{$\varepsilon$-optimal} if it achieves near-optimal value for every state, i.e. $\forall~ s\in \cS:\quad V^{\pi}(s) \geq V^*(s) - \varepsilon.$
The action-value function (or $Q$-function) of a policy $\pi$ is defined as
\begin{equation}\label{eq:Qdef}
Q^{\pi}(s,a) := \sum_{s'\in\cS}p(s'|s,a) \cdot \big( r(s,a,s') +\gamma V^{\pi}(s')\big).
\end{equation}
The optimal $Q$-function is denoted by $Q^* := Q^{\pi^*}$. By Bellman Optimality Equation, we have
\begin{align}
    V^*(s) = \max_{a\in\cA^s} Q^*(s,a) = \max_{a\in\cA^s} \sum_{s'\in\cS}p(s'|s,a) \cdot \big( r(s,a,s') +\gamma V^{*}(s')\big).
\end{align}
\paragraph{Value Gap}
We extend the value gap notion in multi-armed bandits to MDPs. Given an MDP,
for every state we define a vector $q^{s}$ such that
\begin{align}
   q^s[1] &=\max_{a\in\cA^s}Q^*(s,a),\\
   q^s[i] &= \max \{Q^*(s,a)~|~ Q^*(s,a)< q^s_{\cM}[i-1], a\in\cA^s\}, \quad 1<i\leq |\cA^s|.
\end{align}
By definition, we have $\max_{a\in \cA^s} Q_{\cM}^*(s,a) = q^s[1]>q^s[2]>\cdots>q^s[k^s]=\min_{a\in\cA^s} Q^*(s,a),$ where $k^s$ equals the number of different values in $\{Q^*(s,a), a\in\cA^s\}$. Then for every state we define the value-gap vector $\Delta^s$ as
$$\Delta^s[i] = q^s[1]-q^s[i], \quad 1\leq i \leq k^s.$$
It holds that $0 = \Delta^s[1]<\Delta^s[2]<\cdots<\Delta^s[k^s]=\max_{a\in\cA^s} Q^*(s,a)-\min_{a\in\cA^s} Q^*(s,a).
$
Lastly, given any $c>0$, we denote by \begin{align}\label{eq:As}
    \cA^s(c):=\{a~|~\max_{a'} Q^*(s,a') - Q^*(s,a) < c\}.
\end{align}
Note that $\cA^s(c)$ is always nonempty. Given $C>c>0$, we have $\cA^s(C)\supseteq\cA^s(c)$ and if $c>1/(1-\gamma)$, $\cA^s(c)=\cA^s$.

\paragraph{TV-distance for MDPs}
To measure the closeness of two MDPs, we introduce a metric $d_{\mathrm{TV}}(\cdot, \cdot)$. Given $\cM_0:=(\cS,\cA,p_0,r_0,\gamma)$ and $\cM:=(\cS,\cA,p,r,\gamma)$, we define 
\begin{align}\label{eq:TV}
   d_{\mathrm{TV}}(\cM_0, \cM) = \max\Big\{
   \max_{(s,a)\in\cS\times\cA^s}\|p_0(\cdot|s,a)-p(\cdot|s,a)\|_1, ~ \|r_0-r\|_{\infty}~\Big\}.
\end{align}
The name TV comes from that $\|p_0(\cdot|s,a)-p(\cdot|s,a)\|_1$ is equal to the \emph{total variation} distance between distributions $p_0(\cdot|s,a)$ and $p(\cdot|s,a)$. Note that the distance is only valid between MDPs with the same state/action space and the discount factor.   We denote by $\cM\in B_{\mathrm{TV}}(\cM_0,\beta)$ if $d_{\mathrm{TV}}(\cM_0,\cM)\leq \beta$ for some $\beta>0$.

\paragraph{Generative Model}
Generative model is a special sample oracle which allows any $(s,a)\in\cS\times\cA^s$ as input and outputs $(s',r(s,a,s'))$ with probability $p(s'|s,a)$.}

\section{Main Results}\label{sec:main}
We formalize our problem of knowledge transferring as below.
\begin{problem}\label{problem}
Suppose the target unknown model is $\cM:=(\cS, \{\cA^s\}_{s\in\cS}, p, r, \gamma)$ and an agent is provided with the full knowledge of an approximate model $\cM_0:=(\cS, \{\cA^s\}_{s\in\cS}, p_0, r_0, \gamma)$ satisfying $\cM\in B_{\mathrm{TV}}(\cM_0, \beta)$, where $\beta >0$ is a known constant.
How many samples does it take to learn an $\varepsilon$-optimal policy for $\cM$ with probability at least $1-\delta$?
\end{problem}

Without additional structure assumption, we answer Problem \ref{problem} by exploiting value functions. Our idea is based on the fact that if $d_{\text{TV}}(\cM, \cM_0)\leq \beta$, then their optimal $Q$-functions are apart for at most
$\beta/(1-\gamma)^2$ (see Lemma \ref{lemma:Qbound}). Given full knowledge of $\cM_0$, we can then obtain upper and lower bounds on every entry of the optimal $Q$-function of $\cM$. The upper and lower bounds depict how optimal and sub-optimal an action could be at some state in $\cM$. Based on this information, for every state we can split its available actions into two sets: one contains all \emph{potential-optimal} actions and the other are all \emph{sub-optimal} actions. An action is potential-optimal if it has a chance to be optimal in $\cM$; otherwise, it is sub-optimal. For the learning stage, we remove all sub-optimal actions and apply an RL algorithm only on potential-optimal ones. A rigorous definition will be given in Section \ref{sec:analysis}.
This learning process is described in Algorithm \ref{alg:main}.


In Algorithm \ref{alg:main}, we first use a planning algorithm to get the optimal $Q$-function of $\cM_0$, $Q^*_0$. 
Then, for each state $s\in\cS$, we construct a set $\cA^s_{\cM_0}(\overline{C})$ (recall the definition in Equation \eqref{eq:As}). Later we will show that $\cA^s_{\cM_0}(\overline{C})$ consists of all potential-optimal actions.
By replacing the original action space $\{\cA^s\}_{s\in\cS}$ by $\{\cA^s_{\cM_0}(\overline{C})\}_{s\in\cS}$, we form a new MDP $\cM^c$ (Line \ref{line:Mc}), where $p^c(s'|s,a)=p(s'|s,a)$ and $r^c(s,a,s')=r(s,a,s)$ for every $s\in\cS, a\in\cA^s_{\cM_0}(\overline{C})$ and $s'\in\cS$. Note that $\cM^c$ is simply $\cM$ with a contracted action space. Finally, we apply a near-optimal RL algorithm to learn a policy of $\cM^c$. In Theorem \ref{thm:main_upper}, we show that the output policy is $\varepsilon$-optimal for $\cM$ with probability at least $1-\delta$ and the sample complexity therein is our upper bound result to Problem \ref{problem}. 

\begin{algorithm}[t]
	\caption{Transfer RL from an Approximate Model}
	\label{alg:main}
	\begin{algorithmic}[1]
		\State \textbf{Input:} full knowledge of $\cM_0$; a generative model of $\cM$; a near-optimal RL algorithm $\sA_{\text{opt}}$;
		\State \textbf{Parameters:} $\overline{C}>0$, $\varepsilon>0$, $\delta\in(0,1)$.
		\State Apply any planning algorithm on $\cM_0$ to get $Q^*_0$;
		\State Construct $\cA^s_{\cM_0}(\overline{C})$ with $Q_0^*$ (see Equation \eqref{eq:As});
		\State Form a new MDP $\cM^c:=(\cS,\{\cA^s_{\cM_0}(\overline{C})\}_{s\in\cS}, p^c, r^c, \gamma);$\label{line:Mc}
		\State Apply $\sA_{\text{opt}}$ to learn an $\varepsilon/2$-optimal policy $\pi$ of $\cM^c$ with probability at least $1-\delta$.
		\State \textbf{Output:} $\pi$.
	\end{algorithmic}
\end{algorithm}

\begin{theorem}[Main Result -- Upper Bound]\label{thm:main_upper}
Given $\varepsilon>0$ and $\delta\in(0,1)$, with probability at least $1-\delta$, Algorithm \ref{alg:main} returns an $\varepsilon$-optimal policy for $\cM$ using samples
\begin{align}
    \widetilde{\cO}\bigg(\frac{\sum_{s\in\cS'}|\cA^s_{\cM_0}(\overline{C})|}{(1-\gamma)^3\varepsilon^2}\log\Big(\frac{1}{\delta}\Big)\bigg),
\end{align}
where $\overline{C}:=\min \{2/(1-\gamma), 2\beta/(1-\gamma)^2\} - \varepsilon(1-\gamma)/2$.
\end{theorem}
In Theorem \ref{thm:main_upper}, the complexity is summed over $\cS'$ since no action exploration is needed for single-action states. Further, we have the lower bound result to Problem \ref{problem} in Theorem \ref{thm:main_lower}.

\begin{theorem}[Main Result -- Lower Bound]\label{thm:main_lower}
Let $\varepsilon\in(0, \varepsilon_0)$ and $\delta\in(0,\delta_0)$. The sample complexity for Problem \ref{problem} is
\begin{align}
    \Omega\bigg(\frac{\sum_{s\in\cS'}|\cA^s_{\cM_0}(\underline{C}_s)|}{(1-\gamma)^3\varepsilon^2}\log\Big(\frac{1}{\delta}\Big)\bigg),
\end{align}
where $\varepsilon_0= \frac{\beta\gamma\min_{s\in\cS'}V^*(s)^2}{16}\min (\beta/2, \frac{1-\gamma}{3\gamma})$, $\delta_0=1/40$ and
\begin{align}
    \underline{C}_s=\begin{cases}V^*(s) - \frac{9}{12(1-\gamma)-64(1-\gamma)^2\varepsilon+4.5\beta\gamma}, \quad &\text{if~} \beta/2+\frac{4\gamma-1}{3\gamma} \geq 1; \\
V^*(s) - \frac{V^*(s)^2}{V^*(s) +\beta\gamma V^*(s)^2+4\varepsilon(1+\gamma\beta V^*(s) /2)^2}, \quad &\text{otherwise}.
\end{cases}
\end{align}
\end{theorem}

The lower bound shares a similar structure as the upper bound except that $\underline{C}_s \leq \overline{C}$. In Section \ref{sec:meet}, we have more discussion about $\underline{C}_s$ and $\overline{C}$.
For each state $s\in\cS'$, exploring the set of actions $\cA^s_{\cM_0}(\overline{C})$ is \emph{sufficient} to obtain a near-optimal action in $\cM$, likewise,
exploring the actions in $\cA^s_
{\cM_0}(\underline{C}_s)$ is \emph{necessary} to find a near-optimal action in $\cM$.
Note that the sample complexity of learning-from-scratch is $ \widetilde{\Theta}\Big(\frac{\sum_{s\in\cS'}|\cA^s|}{\varepsilon^2(1-\gamma)^3}\log\big(\frac{1}{\delta}\big)\Big)
$ (see \citealt{azar2013minimax}). For the upper bound, if $|\cA^s_{\cM_0}(\overline{C})|\ll |\cA^s|$, knowledge transferring achieves a significantly smaller sample complexity than learning-from-scratch. The best case, as in Corollary \ref{coro:best}, is when $|\cA_{\cM_0}^s(\overline{C})|=1, ~\forall s\in\cS$ and the policy is transferable. For the lower bound, if $|\cA^s_{\cM_0}(\underline{C}_s)|$ is close to $|\cA^s|$, the prior knowledge does not offer much help. The worst case, as in Corollary \ref{coro:worst}, is when $|\cA^s_{\cM_0}(\underline{C}_s)|=|\cA^s|$, then with or without the prior knowledge, it has the same order of complexity. In particular, if $|\cA^s_{\cM_0}(\overline{C})| \approx |\cA^s_{\cM_0}(\underline{C}_s)|$, the upper bound is tight up to a log factor.


\begin{corollary}[The Best Scenario]\label{coro:best}
Suppose for every state $s\in\cS'$,
$|\argmax_a Q^*_{\cM_0}(s,a)|=1$ and $$\big|V^*_{\cM}(s) - \max\big\{Q^*_{\cM}(s,a)~|~Q^*_{\cM}(s,a)<V_{\cM}^*(s), a\in\cA^s\big\}\big|\geq 2\beta/(1-\gamma)^2-\varepsilon/(1-\gamma).$$ Then $|\cA_{\cM_0}^s(\overline{C})|=1, ~\forall s\in\cS$ and the optimal policy for $\cM_0$ is also optimal for $\cM$.
\end{corollary}

\begin{corollary}[The Worst Scenario]\label{coro:worst}
Let $\varepsilon\in\big(0, \varepsilon_0)$ and $\delta\in(0, \delta_0)$, where $\varepsilon_0$ and $\delta_0$ are as defined in Theorem \ref{thm:main_lower}. Suppose for every state $s\in\cS'$, $\argmax_a Q^*_{\cM_0}(s,a)=\cA^s$. Then the sample complexity for Problem \ref{problem} has the same order as that of learning-from-scratch.
\end{corollary}

\cut{\begin{figure}[t]
    \centering
    \vspace{-15pt}
    \includegraphics[width=.7\textwidth]{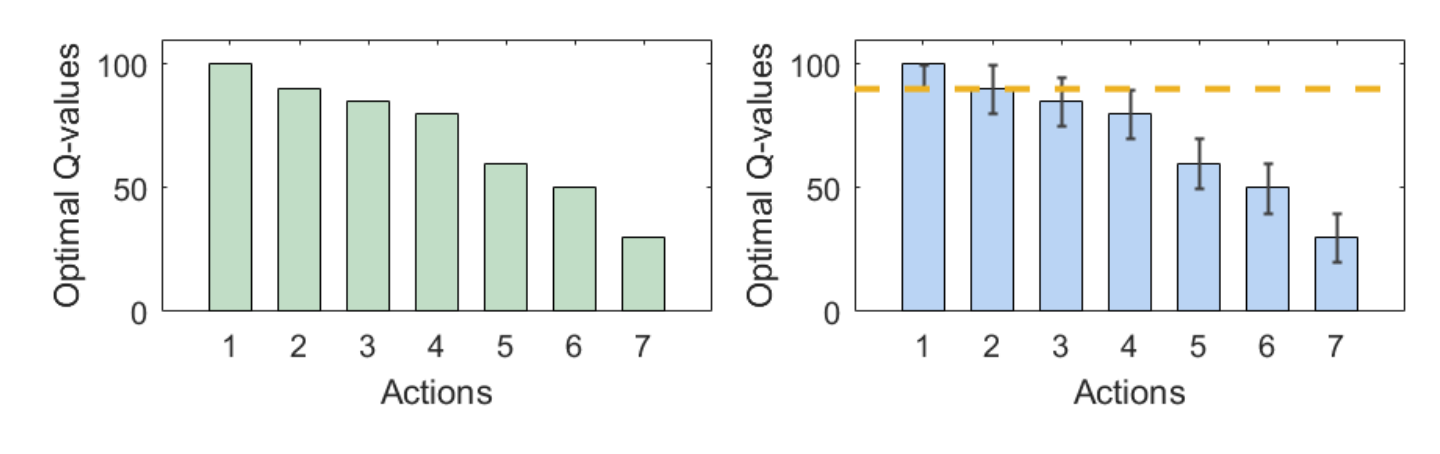}
    \vspace{-15pt}
    \caption{An illustration of how an approximate model helps RL. Given $Q^*_{\cM_0}(s,\cdot)$ (the left graph), we have interval estimation of $Q^*_{\cM}(s,\cdot)$ (the error bars in the right graph). Actions that are potentially optimal in $\cM$ should have the upper bound higher than the lower bound of the best action in $\cM_0$. Therefore, only actions $1, 2$ and $3$ are considered as valid candidates.}
    \label{fig:gap}
\end{figure}}

\section{Analysis}\label{sec:analysis}
In this section, we prove the main results. We defer all proofs to Appendix \ref{app:proof}.

\subsection{Proofs of the Upper Bound}\label{sec:upperbdd}
We follow the notations used in Problem \ref{problem} and always use $Q^*_0$ and $Q^*$ for the optimal $Q$-functions of $\cM_0$ and $\cM$, respectively. Our first result is
\begin{lemma}\label{lemma:Qbound}
$\|Q^*_0-Q^*\|_{\infty}\leq \min\{1/(1-\gamma), \beta/(1-\gamma)^2\}$.
\end{lemma}
As explained before, we use Lemma \ref{lemma:Qbound} to estimate $Q^*$ and then identify potential-optimal actions. Next, we give a rigorous definition of these actions. Before that, we introduce the following lemma.
\begin{lemma}\label{lemma:epsQ}
Given $\epsilon>0$, for every state $s\in\cS$, we denote by $a^s$ an action in $\cA^s$ such that $Q^*(s,a^s)\geq \max_{a} Q^*(s,a) -\epsilon(1-\gamma)$. Then $\pi(s)=a^s$ is an $\epsilon$-optimal policy for $\cM$.
\end{lemma}
Therefore, to construct an $\epsilon$-optimal policy, we only need to find actions satisfying the condition in Lemma \ref{lemma:epsQ}. Instead of searching over the whole action space, we use Lemma \ref{lemma:Qbound} to shrink the searching range. Specifically, we have the next lemma.
\begin{lemma}\label{lemma:aselect}
Given $\beta>0$ and $\epsilon>0$, let $C(\beta,\epsilon) := 2\cdot\min\{1/(1-\gamma), \beta/(1-\gamma)^2\}-\epsilon(1-\gamma)$. For every state $s\in\cS$, we can definitely find an action satisfying the condition in Lemma \ref{lemma:epsQ} from the set $\cA^s_{\cM_0}(C(\beta,\epsilon))$.
\end{lemma}

Actions in the set $\cA^s_{\cM_0}(C(\beta,\epsilon))$ are the potential-optimal actions. Next, we form a new MDP $\cM^c:=(\cS,\{\cA^s_{\cM_0}(C(\beta,\epsilon))\}_{s\in\cS}, p^c, r^c,\gamma)$, where $p^c(s'|s,a)=p(s'|s,a)$ and $r^c(s,a,s')=r(s,a,s)$ for every $s\in\cS, a\in\cA^s_{\cM_0}(\overline{C})$ and $s'\in\cS$. By Lemma \ref{lemma:aselect}, an optimal policy of $\cM^c$ is at least $\epsilon$-optimal for $\cM$. Then we use a near-optimal RL algorithm to learn an $\varepsilon/2$-optimal policy of $\cM^c$. Thus, the output policy of Algorithm \ref{alg:main} is $\epsilon+\varepsilon/2$-optimal for $\cM$. By the definition of an near-optimal RL algorithm and letting $\epsilon=\varepsilon/2$, we obtain Theorem \ref{thm:main_upper}.


%

\cut{Lemma \ref{lemma:Qbound} does not require additional structure assumptions. Based on it, given the full information of $\cM_0$, we can estimate $Q_0^*$ using any planning algorithm and then get interval estimates of every component of $Q^*$. For example, if we compute a function $\widetilde{Q}_0^*$ such that $\big\|\widetilde{Q}_0^*-Q_0^*\big\|_{\infty}\leq \epsilon$, then for every $(s,a)\in\cS\times\cA$,
$$Q^*(s,a)\in \left[\max\big\{0,\widetilde{Q}_0^*(s,a)-C_\beta-\epsilon\big\}, \min\big\{1/(1-\gamma),\widetilde{Q}_0^*(s,a)+C_\beta+\epsilon\big\}\right],$$
where $C_\beta=\beta/(1-\gamma)^2$. With such interval estimates, we can eliminate sub-optimal actions. The idea is straightforward. }

\subsection{Proofs of the Lower Bound}\label{sec:lowerbdd}
Next, we prove Theorem \ref{thm:main_lower}.
We first give a definition about the correctness of RL algorithms.

\begin{definition}(($\cM_0, \beta, \varepsilon, \delta$)-correctness)\label{def:epsilonalgorithm}
Given $\beta>0$ and a prior model $\cM_0$, we say that an RL algorithm $\sA$ is $(\cM_0, \beta, \varepsilon, \delta)$-correct if for every $\cM\in B_{\mathrm{TV}}(\cM_0, \beta)$, $\sA$ can output an $\varepsilon$-optimal policy with probability at least $1-\delta$.
\end{definition}

Next, we construct a class of hard MDPs. We will then select one model $\cM_0$ from the class as prior and  show that if an RL algorithm $\sA$ learns with samples significantly fewer than the lower bound, there would always exist an MDP $\cM\in B_{\mathrm{TV}}(\cM_0, \beta)$ such that $\sA$ cannot be $(\cM_0, \beta, \varepsilon, \delta)$-correct.
\cut{we construct a class of hard MDPs which is a generalization of the simple example in Sec. \ref{sec:toyexample1}. We select one model $\cM_0$ from the class and use its full information as the prior knowledge. Then, we consider any algorithm that is $(\cM_0, \beta, \varepsilon, \delta)$-correct.}

\cut{learning $\varepsilon$-optimal policies for

another two models $\cM^1$ and $\cM^2$ that are both $\beta-$close to $\cM_0$ under TV-distance. For the later two models, we assume access to their generative models respectively. Note that these two learning tasks (for $\cM^1$ or $\cM^2$) should have the same sample complexity since they are provided with the same prior knowledge and sample accesses. Then, we show that when $\varepsilon<\varepsilon_0(\beta)$, if the sample number is less than $C\frac{N}{(1-\gamma)^3\varepsilon^2}\log(\frac{N}{\delta})$, at least one of the learning tasks cannot be accomplished. Therefore, a lower bound is achieved.}

\cut{If there is an $(\cM_0, \beta,\varepsilon, \delta, Q)$-correct algorithm $\sA$, then it can return $\varepsilon$-optimal $Q$-values for both $\cM^1$ and $\cM^2$ with probability at least $1-\delta$. We then restrict $\varepsilon$ to be very small and select $\cM^1$ and $\cM^2$ smartly such that these two models are super close to each other but still the $Q$-values differ by more than 2$\varepsilon$. Then the difficulty of the problem is equivalent to distinguishing two very close values where the $\beta$ information is useless. By then, we are able to show that with number of samples significantly less than $C\frac{|\cS||\cA|}{(1-\gamma)^3\varepsilon^2}\log(\frac{|\cS|\cA|}{\delta})$, at least for one of $\cM^1$ and $\cM^2$, $\sA$ cannot produce an $\varepsilon$-optimal $Q$-values. Therefore, the lower bound result is obtained. }

\paragraph{Construction of the Hard Case}
We define a family of MDPs $\mathbb{M}$ with a structure as in Figure \ref{fig:mdp}. The state space $\cS$ consists of three disjoint subsets $\cX$ (gray nodes), $\cY_1$ (green nodes), and $\cY_2$ (blue nodes). The set $\cX$ includes $K$ states $\{x_1, x_2, \dots, x_K\}$ and each of them has $L$ available actions $\{a_1,a_2,\dots,a_L\}=:\cA$. States in $\cY_1$ and $\cY_2$ are all of single-action. For state $x\in\cX$, by taking action $a\in\cA$, it transitions to a state $y_1(x,a)\in\cY_1$ with probability 1. Note that such a mapping is one-to-one from $\cX\times\cA$ to $\cY_1$. For state $y_1(x,a)\in\cY_1$, it transitions to itself with probability $p_{\cM}(x,a)\in(0,1)$ and to a corresponding state $y_2(y_1) \in \cY_2$ with probability $1-p_{\cM}(x,a)$. $p_{\cM}(x,a)$ can be different for different models. All states in $\cY_2$ are absorbing. The reward function $R(s,a,s')=1$ if $s'\in\cY_1$; otherwise 0.

$\mathbb{M}$ is a generalization of a multi-armed bandit problem used in \citealt{mannor2004sample} to prove a lower bound on bandit learning. A similar example is also shown in \citealt{azar2013minimax}.
For an MDP $\cM\in\mathbb{M}$, it is fully determined by the parameter set $\{p_{\cM}(x_k,a_l), k\in[K], l\in[L]\}$. And its $Q$-function is
$Q_{\cM}(x,a) = 1/(1-\gamma p_{\cM}(x,a)), \forall~ (x,a)\in\cX\times\cA.$

\begin{figure}[t]
  \centering
  \vspace{-10pt}
  \includegraphics[width=.9\textwidth]{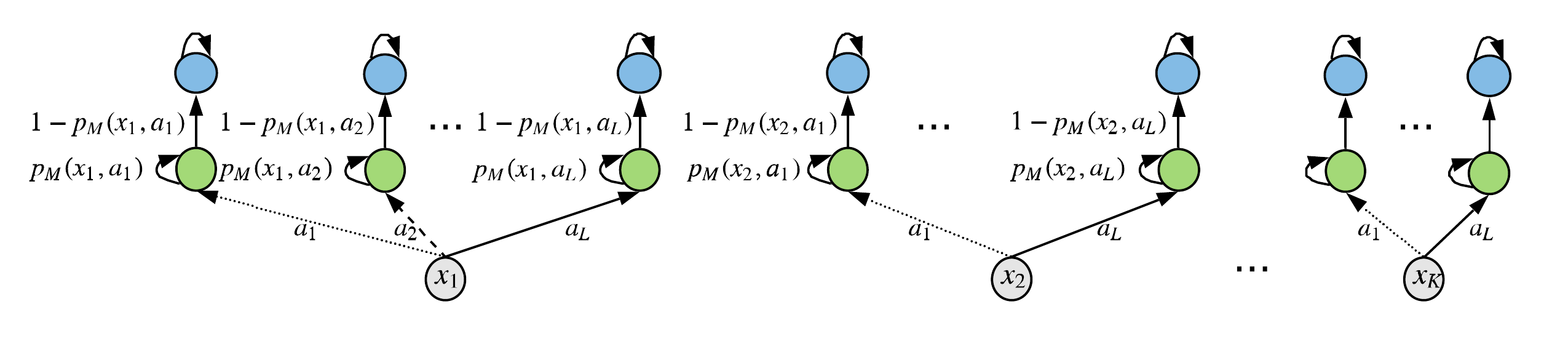}\\
  \vspace{-15pt}
  \caption{The class of MDPs considered to prove the lower bound in Theorem \ref{thm:main_lower}.} 
  \label{fig:mdp}
\end{figure}

In the sequel,
we restrict $\beta\in(0,2)$ (by definition in Equation \ref{eq:TV}, $d_{\text{TV}}(\cdot,\cdot)\leq 2$). Fixing $\beta$, we further restrict the discount factor $\gamma \in \big(\max\{0.4, 1-10\beta\}, 1\big)$.

\paragraph{Prior Model $\cM_0$ and Hypotheses of $\cM$}
For $\cM_0$, we simplify the notation $p_{\cM_0}(x_k,a_l)$ as $p_0(x_k,a_l)$. Without loss of generality, we assume
$
    1>p_0(x_k,a_1)\geq p_0(x_k,a_2)\geq \cdots \geq p_0(x_k,a_L)\geq 0, \text{~for every } x_k\in\cX.
$
Thus, in $\cM_0$, the $Q$-values from $a_1$ to $a_L$ are monotonically non-increasing. Given $\beta$ and $\gamma$, we require $p_0(x_k,a_1)\in\big(\frac{4\gamma-1}{3\gamma},1\big)$ for every $x_k\in\cX$.

After selecting a prior model $\cM_0$ satisfying the above conditions, we define $p_0^k := \max\{p_0(x_k, a_1)-\beta/2, \frac{4\gamma-1}{3\gamma}\}$ for every $k\in[K]$. Then let $\varepsilon_0 := \min_{k\in[K]} \big\{\beta\gamma(1-p_0^k)/(16(1-\gamma p_0^k)^2)\big\}$. Fixing $\varepsilon\in(0, \varepsilon_0)$, we define two numbers $\alpha_1^k$ and $\alpha_2^k := 4(1-\gamma p_0^k)^2\varepsilon/\gamma$ such that
\begin{align}\label{eq:a1}
    &\frac{1}{1-\gamma(p_0^k+\alpha_1^k)}-\frac{1}{1-\gamma p_0^k} = 2\varepsilon\quad\text{and}\quad\frac{1}{1-\gamma(p_0^k+\alpha_2^k)}-\frac{1}{1-\gamma (p_0^k+\alpha_1^k)} \geq 2\varepsilon.
\end{align}
\cut{Furthermore, for every $k\in[K]$, we define an integer
\begin{align}\label{eq:Lk}
    L_k:=\bigg\vert~\Big\{l\in[L]~\Big\vert ~\big\vert~p_0^k+\alpha_2^k-p_0(x_k,a_l)~\big\vert\leq\beta/2\Big\}~\bigg\vert.
\end{align}
\begin{lemma}\label{lemma:2}
Given $\varepsilon\in (0, \varepsilon_0)$, the set in Equation \eqref{eq:Lk} is $\{ 1\leq l\leq L_k\}$.
\end{lemma}}
Now, we are ready to define possibilities of $\cM$. Given $K$ integers $\{L_k\in[L]\}_{k\in[K]}$, let
\begin{align}\label{eq:M}
\cM_{1}: \text{for every } k\in[K], ~ & \begin{cases}p_{\cM_1}(x_k,a_1) = p_0^k+\alpha_1^k,\\ p_{\cM_1}(x_k,a_l)=p_0^k, \quad &2\leq l\leq L_k,\\
p_{\cM_1}(x_k,a_l)=p_0(x_k,a_l), \quad &l >L_k;\end{cases}\\
\text {for every } k\in[K], 1<l\leq L_k,~  \cM_{k,l}:~& \begin{cases}
p_{\cM_{k,l}}(x_{k},a_l) = p_0^k+\alpha_2^k\\
p_{\cM_{k,l}}(x_{k'},a_{l'}) = p_{\cM_1}(x_{k'},a_{l'}), \quad \forall~ (k',l')\neq(k,l).\end{cases}
\end{align}
Due to our specific selection of $p_0^k, \alpha_1^k$ and $\alpha_2^k$, we have the following lemma.
\begin{lemma}\label{lemma:checkM}
Let
$L_k:=\Big\vert\big\{l\in[L]~\big\vert ~\vert~p_0^k+\alpha_2^k-p_0(x_k,a_l)~\vert\leq\beta/2\big\}\Big\vert$ for every $k\in[K]$.
All possibilities defined in \eqref{eq:M} lie in $B_{\text{TV}}(\cM_0, \beta)$.
\end{lemma}
We refer to the models in Equation \eqref{eq:M} as \emph{hypotheses} of $\cM$. There are in total $1+\sum_{k\in[K]}(L_k-1)$ of them. In $\cM_1$, for every $x_k\in\cX$, $a_1$ is the optimal action and $a_2$ to $a_{L_k}$ are the second-best actions with a value only $2\varepsilon$ less than the optimal (Equation \eqref{eq:a1}).
$\{\cM_{k,l}\}$ are built on top of $\cM_1$ by raising up the value of the action $a_l$ at state $x_k$ by at least $4\varepsilon$ (Equation \eqref{eq:a1}). Thus, in $\cM_{k,l}$, the best action for $x_k$ is $a_l$ and the best action for $x_{k'}~(k'\neq k)$ is still $a_1$.
See Figure \ref{fig:modelQ} for illustration. Every hypothesis gives a probability measure over the same sample space. We denote by $\mathbb{E}_1$, $\mathbb{P}_1$ and $\mathbb{E}_{k,l}$, $\mathbb{P}_{k,l}$ the expectation and probability under hypothesis $\cM_1$ and $\cM_{k,l}$, respectively. These probability measures capture both randomness in the MDP and the randomization in an RL algorithm like its sampling strategy.
\begin{figure}
    \centering
    \includegraphics[width=.9\textwidth]{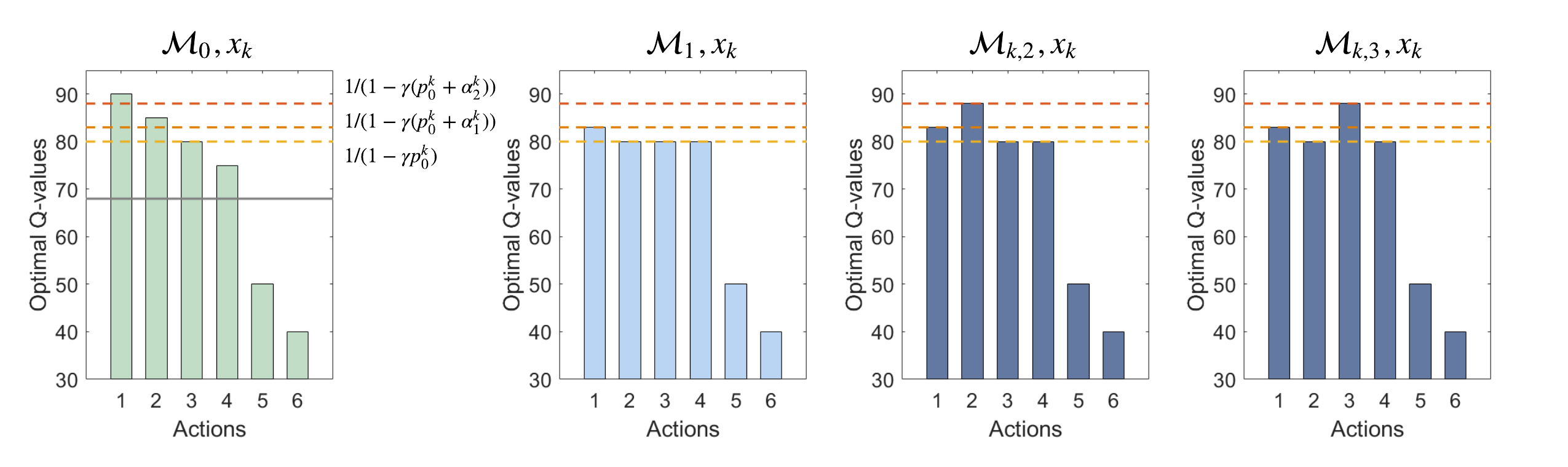}
    \vspace{-10pt}
    \caption{The optimal $Q$-values of $\cM_0$, $\cM_1$, $\cM_{k,2}$, and $\cM_{k,3}$ at state $x_k$ with 6 actions. The dashed lines indicate the values: $\frac{1}{1-\gamma p_0^k}$, $\frac{1}{1-\gamma (p_0^k+\alpha_1^k)}$, and $\frac{1}{1-\gamma (p_0^k+\alpha_2^k)}$, respectively. Actions above the grey line in $\cM_0$ are in the set for the definition of $L_k$ in Lemma \ref{lemma:checkM}. Note that for states $x_{k'}~ (k'\neq k)$, $\cM_{k,2}$ and $\cM_{k,3}$ have the same shape as $\cM_1$.}
    \label{fig:modelQ}
\end{figure}


We fix $\varepsilon\in(0, \varepsilon_0)$ and $\delta\in(0, 1/40)$. Let $\sA$ be an $(\cM_0, \beta,\varepsilon, \delta)$-correct RL algorithm. We denote by $T_{k,l}$ the number of samples that algorithm $\sA$ calls from the generative model with input state $y_1(x_{k},a_{l})$ till $\sA$ stops (these sample calls are not necessarily consecutive). For every $k\in[K], 1<l\leq L_k$, we define an event
$E_{k,l} = \{ \sA \text{ outputs a policy } \pi \text{ with } \pi(x_k) = a_1\}.$ Then we have the following key lemma.

\cut{\begin{lemma}\label{lemma:4}
For any $k\in[K], 1<l\leq L_k$, if ~$\mathbb{E}_1[T_{k,l}]\leq t^*$, $\mathbb{P}_1(A_{k,l})> 3/4$.
\end{lemma}

\begin{lemma}\label{lemma:5}
For any $k\in[K], 1<l\leq L_k$, $\mathbb{P}_1(C_{k,l})> 3/4$.
\end{lemma}

Since $\sA$ is $(\cM_0, \beta, \varepsilon, \delta)$-correct, it should return a policy $\pi$ such that when $\cM=\cM_1$, $\pi(x_{k})=a_1$ for every $k\in[K]$ with probability at least $1-\delta$, i.e. $\mathbb{P}_1\big(~E_{k,l}, ~\text{for all } k\in[K], 1<l\leq L_k~ \big)\geq 1-\delta>3/4$.
Combining the results above, it holds that
$
\mathbb{P}_1(A_{k,l}\cap E_{k,l} \cap C_{k,l})>1/4, \forall~ k\in[K], 1<l\leq L_k.
$
Next, we show that if the expectation of number of samples taken through $\sA$ on any $y_1(x_k,a_l)$ is less than $t^*$, with probability $>\delta$, $\sA$ can not return an $\varepsilon$-policy for the corresponding hypothesis $\cM_{k,l}$.
}

\begin{lemma}\label{lemma:newversion}
For any $k\in[K], 1<l\leq L_k$, if ~$\mathbb{E}_1[T_{k,l}]<\frac{c_1}{(1-\gamma)^3\varepsilon^2}\log\Big(\frac{1}{4\delta}\Big)$, $\mathbb{P}_{k,l}(E_{k,l})>\delta$, where $c_1>0$ is some large constant.
\end{lemma}
Based on Lemma \ref{lemma:newversion}, we can prove Theorem \ref{thm:main_lower}. The main idea is that if $\sA$ is $(\cM_0,\beta,\varepsilon, \delta)$-correct, it should return a policy $\pi$ such that $\pi(x_k)=a_l$ with probability $\geq 1-\delta$ under hypothesis $\cM_{k,l}$, i.e., $\mathbb{P}_{k,l}(E_{k,l})<\delta$. By Lemma~\ref{lemma:newversion}, this requires $\mathbb{E}_1[T_{k,l}]>t^*$ for all $k\in[K], 1<l\leq L_k$. Thus, in total we need $\Omega\left(\frac{\sum_{k\in[K]}L_k}{(1-\gamma)^3\varepsilon^2}\log(1/\delta)\right)$ samples. Following the definition of $L_k$, we can get Theorem 2. More technical details can be found in Appendix \ref{app:proof}. Note that as any online algorithm can be
realized in the generative model setting, the lower bound automatically adapts to the online setting.

\cut{
By definition of $L_k$, for every $k\in[K]$, we have
$
\{a_l, l\leq L_k\} = \cA^{x_k}_{\cM_0}\big(1/(1-\gamma p_0(x_k,a_1))-1/(1-\gamma (p_0^k+\alpha_2^k-\beta/2))\big).
$

If $p_0^k=\frac{4\gamma-1}{3\gamma}$, i.e. $p_0(x_k,a_1)-\beta/2\leq\frac{4\gamma-1}{3\gamma}$, then
\begin{align}\label{eq:Lk1}
    L_k &= \Big\vert\cA^{x_k}_{\cM_0}\Big(V^*(x_k) - \frac{9}{12(1-\gamma)-64(1-\gamma)^2\varepsilon+4.5\beta\gamma}\Big) \Big\vert.
    \end{align}
If $p_0^k=p_0(x_k,a_1)-\beta/2$, i.e. $p_0(x_k,a_1)-\beta/2 > \frac{4\gamma-1}{3\gamma}$, then
 \begin{align}\label{eq:Lk2}
    L_k 
    &=\Big\vert\cA^{x_k}_{\cM_0}\Big(V^*(x_k) - \frac{V^*(x_k)^2}{V^*(x_k)+\gamma\beta (V^*(x_k))^2 - 4\varepsilon(1+\gamma\beta V^*(x_k)/2)^2}\Big)\Big\vert
    \end{align}
  Combining Equation \eqref{eq:Lk1} and \eqref{eq:Lk2}, we conclude the proof Theorem \ref{thm:main_lower}.
}

\cut{The proof of Theorem~\ref{thm:pilowerbound} is established in  three steps:
\begin{enumerate}
\item We will show that for a fixed $(x,a)\in\cX\times\cA$,
$Q^*_{\cM^i}(x,a)$ for $i\in\{1,2\}$ cannot
be learnt up to $\epsilon$ error with high probability using only a small number of samples;
    \item We will generalize the above lower bound  to a uniform bound on all state-action pairs using the fact that their transitions are independent of each other;
    \item Lastly, we will extend the lower bound to learning an $\varepsilon$-optimal policy.
\end{enumerate}
The proof will follow Lemma~\ref{lemma:xabound}, Lemma~\ref{lemma:Qbound}, Proposition~\ref{prop:lowerboundQ}, and Lemma~\ref{lemma:policyeva}. For simplicity, we use $p_1(x,a)$ and $p_2(x,a)$ to denote $p_{\cM^1}(x,a)$ and $p_{\cM^2}(x,a)$, respectively.}
\cut{
\subsection*{Step 1: a fixed $(x,a) \in \cX\times\cA$}
\fei{Is the current definition fine?}

Given a fixed $(x,a)$ and a generative model, we can generate $t$ samples with the same input state $y_1(x,a)$. The rewards of these samples are equivalent to a length-$t$ sequence of i.i.d Bernoulli random variables $W_t:=\{X_1, X_2, \dots X_t\}$. We define two probability spaces $(\Omega, \cF, P^1)$ and $(\Omega, \cF, P^2)$ for $W_t$, where in the former, $\mathbb{P}(X_1=1)=p_1(x,a)$; in the latter, $\mathbb{P}(X_1=1)=p_2(x,a)$. Note that these two spaces have the same sample space and $\sigma-$algebra. The only difference is the probability measure. In the following context, we use the notation $\mathbb{P}^1(\cdot)$ if the event is measured by $P^1$ and vice versa for $\mathbb{P}^2(\cdot)$.

Suppose $\sA$ is an $(\cM_0, \beta, \varepsilon, \delta, Q)$-correct algorithm Denote by $w$ a realization of $W_t$ and $Q^{\sA}_{t}(x,a)$ the output Q-value of $(x,a)$ from $\sA$ with input $w$. We define events:
\begin{align}
\cE_1(x,a)&:=\{|Q^*_{\cM^1}(x,a)-Q^{\sA}_{t}(x,a)|\leq\varepsilon,\quad \cE_2(x,a):=\{|Q^*_{\cM^2}(x,a)-Q^{\sA}_{t}(x,a)|\leq\varepsilon\}
\end{align}
Denoting by $k$ the number of 1s in $w$, we define another event:
$$
\cE_3(x,a):= \bigg\{p_{1}(x,a)\cdot t-k\leq\sqrt{2p_{1}(x,a)\cdot (1-p_{1}(x,a))\cdot t\cdot \log\Big(\frac{c}{2\theta(x,a)}\Big)}\bigg\},$$
where $\theta(x,a):=\exp(-c\alpha(x,a)^2t/(p_{1}(x,a)(1-p_{1}(x,a))))$. Furthermore, let $\cE(x,a):=\cE_1(x,a)\cap\cE_3(x,a)$.

Similar to the proof in \citealt[Lemma 16]{azar2013minimax}, by Chernoff-Hoeffding bound and $p_1(x,a)>1/2$, we first have that
\begin{align}\label{eq:event3}
\mathbb{P}^1(\cE_3)>1-\frac{2\theta(x,a)}{c}.
\end{align}

\cut{
\begin{proof}
Denote by $\epsilon:=\sqrt{2p_{1}(x,a)(1-p_{1}(x,a))t\log(\frac{c}{2\theta(x,a)})}$. By Chernoff-Hoeffding bound, since $p_1(x,a)>1/2$, we have that
\begin{align}
\mathbb{P}^1(\cE_2)&\geq 1-\exp[-\text{KL}(p_{1}(x,a)+\frac{\epsilon}{t}||p_{1}(x,a))\cdot t]\\
&\geq 1-\exp[-\frac{\epsilon^2}{2p_{1}(x,a)(1-p_{1}(x,a))t}]\\
&\geq 1-\frac{2\theta(x,a)}{c}.
\end{align}
\end{proof}}

Next, we show that for the fixed $(x,a)$, there is certain probability such that the Q value approximated by $\sA$ cannot be very accurate on both models.
\begin{lemma}\label{lemma:xabound}
When $\varepsilon\in\bigg(0,\frac{(1-p_{0}-\beta'/4)\gamma^2}{8(1-\gamma p_{0})^2}\bigg)$, there exists a constant $c>0$ such that for every $(\cM_0, \beta, \varepsilon, \delta, Q)$-correct RL algorithm $\sA$, it holds that
$$\big(\max_{m\in\{1,2\}} \mathbb{P}^m (\neg~\cE_m(x,a))\big)>\frac{\theta(x,a)}{c},$$
by the choice of $\alpha(x,a)=\frac{2(1-\gamma p_{0})^2\varepsilon}{\gamma^2(p_1(x,a))}$.
\end{lemma}
\begin{proof}
We prove by contradiction. Suppose there is an algorithm $\sA$ such that
\begin{align}
\mathbb{P}^m(|Q^*_{\cM^m}(x,a)-Q^{\sA}_{t}(x,a)|>\varepsilon)\leq\frac{\theta(x,a)}{c}, \quad \text{for all }m\in\{1,2\}.
\end{align}
Then we have that $\mathbb{P}^1(\cE_1)\geq 1-\theta(x,a)/c$. Since $\theta(x,a)<1$, combining with Equation \eqref{eq:event3} and choosing $c>6$, we have $\mathbb{P}(\cE)>1/2$. Denote by $w$ a realization of $W_t$ and $L^m(w):=\mathbb{P}^m(W_t=w), m\in\{1,2\}$. Recall that $k$ is the number of 1s in $w$. Then, the likelihood ratio between model $\cM^1$ and model $\cM^2$ is:
\begin{align}\label{eq:ratio}
\frac{L^2(w)}{L^1(w)}&=
\frac{p_{2}(x,a)^{k}(1-p_{2}(x,a))^{t-k}}{p_{1}(x,a)^{k}(1-p_{1}(x,a))^{t-k}}=\big(1+\frac{\alpha(x,a)}{p_{1}(x,a)}\big)^{k}\big(1-\frac{\alpha(x,a)}{1-p_{1}(x,a)}\big)^{t-k}\\
&=\big(1-\frac{\alpha(x,a)}{1-p_{1}(x,a)}\big)^{\frac{1-p_{1}(x,a)}{p_{1}(x,a)}k}\big(1-\frac{\alpha(x,a)}{1-p_{1}(x,a)}\big)^{t-\frac{k}{p_{1}(x,a)}}\big(1+\frac{\alpha(x,a)}{p_{1}(x,a)}\big)^{k}
\end{align}
By our choice of $\varepsilon$ and $\alpha(x,a)$, we have that $\alpha(x,a) = \frac{1-p_0-\beta'/4}{4(p_0+b(x,a))}\leq \frac{1-p_0-\beta'/4}{2}\leq (1-p_1(x,a))/2$. With the fact that $\log (1-u) \geq-u-u^{2}$ for $u\in[0,1/2]$ and $\exp (-u) \geq 1-u$ for $u\in[0,1]$, we further have that
\begin{align}
\bigg(1-\frac{\alpha(x,a)}{1-p_{1}(x,a)}\bigg)^{\frac{1-p_{1}(x,a)}{p_{1}(x,a)}}& \geq \exp \left(\frac{1-p_{1}(x,a)}{p_{1}(x,a)}\left(-\frac{\alpha(x,a)}{1-p_{1}(x,a)}-\big(\frac{\alpha(x,a)}{1-p_{1}(x,a)}\big)^{2}\right)\right) \\ & \geq\left(1-\frac{\alpha(x,a)}{p_{1}(x,a)}\right)\left(1-\frac{\alpha(x,a)^{2}}{p_{1}(x,a)(1-p_{1}(x,a))}\right).
\end{align}
Thus
\begin{align} \frac{L_{2}(w)}{L_{1}(w)} & \geq\left(1-\frac{\alpha(x,a)^2}{p_{1}(x,a)^2}\right)^{k}\left(1-\frac{\alpha(x,a)^{2}}{p_{1}(x,a)(1-p_{1}(x,a))}\right)^{k}\bigg(1-\frac{\alpha(x,a)}{1-p_{1}(x,a)}\bigg)^{t-\frac{k}{p_{1}(x,a)}} \\ & \geq\left(1-\frac{\alpha(x,a)^2}{p_{1}(x,a)^2}\right)^{t}\left(1-\frac{\alpha(x,a)^{2}}{p_{1}(x,a)(1-p_{1}(x,a))}\right)^{t}\bigg(1-\frac{\alpha(x,a)}{1-p_{1}(x,a)}\bigg)^{t-\frac{k}{p_{1}(x,a)}}\end{align}
since $k\leq t$. Using $\log(1-u)\geq -2u$ for $u\in[0,1/2]$, we have that
$$\left(1-\frac{\alpha(x,a)^{2}}{p_{1}(x,a)(1-p_{1}(x,a))}\right)^{t} \geq \exp \left(-2 t \frac{\alpha(x,a)^{2}}{p_{1}(x,a)(1-p_{1}(x,a))}\right) \geq\left(2 \theta(x,a) / c\right)^{2 /c_1},$$
where $c_1\leq \frac{c}{\log\theta(x,a)}\log\frac{2\theta(x,a)}{c}$. Further, we have that
$$\left(1-\frac{\alpha(x,a)^2}{p_{1}(x,a)^2}\right)^{t} \geq \exp \left(-2t \frac{\alpha(x,a)^2}{p_{1}(x,a)^2}\right) \geq\left(2 \theta(x,a) / c\right)^{\frac{2(1-p_{1}(x,a))}{c_1p_{1}(x,a)}}.$$
On $\cE_3$, it holds that $(t-\frac{k}{p_1(x,a)}) \leq \sqrt{2\frac{1-p_{1}(x,a)}{p_{1}(x,a)}t\log\frac{c}{2\theta(x,a)}}$, thus
\begin{align}
\left(1-\frac{\alpha(x,a)}{1-p_{1}(x,a)}\right)^{t-\frac{k}{p_{1}(x,a)}} & \geq\left(1-\frac{\alpha(x,a)}{1-p_{1}(x,a)}\right)^{\sqrt{2\frac{1-p_{1}(x,a)}{p_{1}(x,a)}t\log\frac{c}{2\theta(x,a)}}} \\
& \geq \exp \left(-\sqrt{8\frac{\alpha(x,a)^2}{p_{1}(x,a)(1-p_{1}(x,a))}t\log\frac{c}{2\theta(x,a)}}\right) \\
& \geq \exp \left(-\sqrt{\frac{-8\log\theta(x,a)}{c}\log\frac{c}{2\theta(x,a)}}\right)\\
& \geq\left(2 \theta(x,a) / c\right)^{\sqrt{8/c_1}}
\end{align}
By taking $c_1=16$ and $c=17$,
$$\frac{L_{2}(W)}{L_{1}(W)} \geq\left(2 \theta(x,a) / c\right)^{2 / c_{1}+2(1-p_1(x,a)) /(c_{1}p_1(x,a))+\sqrt{8/c_1}}\geq 2 \theta(x,a) / c.$$
By a change of measure, we deduce that
\begin{equation}\label{eq:p2}
\mathbb{P}^2(\cE_1)\geq \mathbb{P}^2(\cE)=\mathbb{E}_2[\mathbf{1}_{\cE}]=\mathbb{E}_1(\frac{L_2(W)}{L_1(W)}\mathbf{1}_{\cE})\geq\frac{\theta(x,a)}{c}.
\end{equation}
For this choice of $\alpha(x,a)$,
\begin{align}
Q^*_{\cM^2}-Q^*_{\cM^1} &=\frac{\gamma}{1-\gamma p_{2}(x,a)}-\frac{\gamma}{1-\gamma p_{1}(x,a)}\\
&= \frac{\alpha(x,a)\gamma^2}{(1-\gamma p_2(x,a))(1-\gamma p_1(x,a))}\\
&>\frac{\alpha(x,a)\gamma^2}{(1-\gamma p_0)^2} = \frac{2\varepsilon}{p_0+b(x,a)}\geq 2\varepsilon.
\end{align}

Thus, event $\{|Q^*_{\cM^2}-Q^{\sA}_t|\leq\varepsilon\}$ does not overlap with the event $\{|Q^*_{\cM^1}-Q^{\sA}_t|\leq\varepsilon\}$. By \eqref{eq:p2}, $\mathbb{P}^2(\{|Q^*_{\cM^1}-Q^{\sA}_t|\leq\varepsilon\})\leq 1-\frac{\theta(x,a)}{c}$. Thus, $\mathbb{P}^2(\{|Q^*_{\cM^2}-Q^{\sA}_t|\geq\varepsilon\})\geq \frac{\theta(x,a)}{c}$, which contradicts with our assumption.
\end{proof}

Based on Lemma \ref{lemma:xabound}, for $\beta>0$ and $0.4\leq\gamma<1$, selecting $p_{0}=\frac{4\gamma-1}{3\gamma}$ and $\varepsilon\in\Big(0,\frac{1-p_{0}-\beta'/4}{4(1-\gamma p_{0})^2}\gamma^2\Big)$, there exists an $\cM\in B_{\mathrm{TV}}(\cM_0,\beta)$ such that for any $(x,a)\in\cX\times\cA$ and any $(\cM_0, \beta, \varepsilon,\delta,Q)$-correct algorithm $\sA$,
\begin{align}
\mathbb{P}^{\cM}(|Q^*_{\cM}(x,a)-Q^{\sA}_t(x,a)|>\varepsilon)&>\frac{1}{c}\exp(-c\alpha^2t/(p_1(x,a)(1-p_1(x,a)))).\\
&=\frac{1}{c}\exp(\frac{-4c(1-\gamma p_0)^4\varepsilon^2t}{\gamma^4p_1(x,a)(1-p_1(x,a))})>\exp(-34000(1-\gamma)^3\varepsilon^2t).
\end{align}
This  implies that for any state-action pair $(x,a)\in\cX\times\cA$ and $\delta>0$,
\begin{equation}\label{eq:delta}
\mathbb{P}^{\cM}(|Q^*_{\cM}(x,a)-Q^{\cA}_t(x,a)|>\varepsilon)>\delta
\end{equation}
whenever the number of transition samples $t$ is less than $\xi(\varepsilon, \delta):=\frac{1}{4\varepsilon^2(1-\gamma)^3}\log(\frac{1}{\delta})$, where $4=34000$.

\subsection*{Step 2: generalize to all state-action pairs}
Next, we prove a sample complexity lower bound on $\|Q^*_{\cM}-Q^{\sA}_T\|_{\infty}\leq\varepsilon$, where $T$ is the total number of samples for all $N:=3KL$ state-action pairs.

\begin{lemma}\label{lemma:Qbound}
For any $\delta'\in(0,1/2)$, $\varepsilon\in\Big(0,\frac{1-p_{0}-\beta'/4}{4(1-\gamma p_{0})^2}\gamma^2\Big)$, and any $(\cM_0, \beta, \varepsilon, \delta, Q)$-correct algorithm $\sA$, it holds that when $T \leq\frac{N}{6}\xi(\varepsilon, \frac{12\delta'}{N})$,
$$\big(\max_{m\in\{1,2\}}\mathbb{P}^{m}(\|Q^*_{\cM^m}-Q^{\sA}_T\|_{\infty}>\varepsilon)\big)>\delta'.$$
\begin{proof}
Denote by $\delta:=\frac{12\delta'}{N}$. When $T\leq\frac{N}{6}\xi(\varepsilon, \delta)$, at most $\lceil N/6 \rceil$ state-action pairs have samples more than $\xi(\varepsilon, \delta)$. Since $|\cY_1|=N/3$, at least $\lfloor N/6 \rfloor$ states in $\cY_1$ are sampled less than $\xi(\varepsilon, \delta)$ times. We label these states as $y_1(x_i,a_i), i\in\{1,2,\dots,\lfloor N/6 \rfloor\}$. Define $\cQ^m(x,a):=\{|Q^*_{\cM^m}(x,a)-Q^{\sA}_{T(x,a)}(x,a)|>\varepsilon\}$, where $T(x,a)$ is the number of samples for $y_1(x,a)$ within all $T$ samples. Since transitions of different $y_1$s are independent, the event $\cQ^m(x,a)$ and $\cQ^m(x',a')$ are conditionally independent given $T(x,a)$ and $T(x',a')$. Following Lemma~\ref{lemma:xabound}, we have that
\begin{align}
    &\mathbb{P}^m(\{\cQ^m((x_i,a_i))^c)\}_{1\leq i\leq\lfloor N/6 \rfloor} \cap \{T(x_i,a_i)\leq\xi(\varepsilon,\delta)\}_{1\leq i\leq \lfloor N/6 \rfloor})\\
    =& \sum_{t_1=0}^{\xi(\varepsilon, \delta)}\cdots\sum_{t_{\lfloor N/6 \rfloor}=0}^{\xi(\varepsilon,\delta)}\mathbb{P}^m(\{T(x_i,a_i)=t_i\}_{1\leq i\leq \lfloor N/6 \rfloor})\mathbb{P}^m(\{\cQ^m(x_i,a_i)^c\}_{1\leq i\leq \lfloor N/6 \rfloor}\vert \{T(x_i,a_i)=t_i\}_{1\leq i\leq \lfloor N/6 \rfloor})\\
    =&\sum_{t_1=0}^{\xi(\varepsilon, \delta)}\cdots\sum_{t_{\lfloor N/6 \rfloor}=0}^{\xi(\varepsilon,\delta)}\mathbb{P}^m(\{T(x_i,a_i)=t_i\}_{1\leq i\leq \lfloor N/6 \rfloor})\prod_{1\leq i\leq N/6}\mathbb{P}^m(\{\cQ^m(x_i,a_i)^c\}\vert \{T(x_i,a_i)=t_i\})\\
    \leq &\sum_{t_1=0}^{\xi(\varepsilon, \delta)}\cdots\sum_{t_{\lfloor N/6 \rfloor}=0}^{\xi(\varepsilon,\delta)}\mathbb{P}^m(\{T(x_i,a_i)=t_i\}_{1\leq i\leq \lfloor N/6 \rfloor})(1-\delta)^{\lfloor N/6 \rfloor},
\end{align}
where the last inequality follows Eq. \eqref{eq:delta}. Thus,
\begin{equation}
    \mathbb{P}^m(\{\cQ^m((x_i,a_i))^c)\}_{1\leq i\leq N/6} \vert \{T(x_i,a_i)\leq\xi(\epsilon,\delta)\}_{1\leq i\leq N/6})\leq (1-\delta)^{N/6}.
\end{equation}
When $T\leq(N/6)\xi(\varepsilon, \delta)$,
\begin{align}
    \mathbb{P}^m(\|Q^*_{\cM^m}-Q^{\cA}_T\|>\varepsilon)&\geq \mathbb{P}^m\big(\bigcup_{(x,a)\in\cX\times\cA} \cQ^m(x,a)\big)\\
    &\geq 1- \mathbb{P}^m(\{\cQ^m((x_i,a_i))^c)\}_{1\leq i\leq N/6} \vert \{T_{x_i,a_i}\leq\xi(\epsilon,\delta)\}_{1\leq i\leq N/6})\\
    &\geq 1-(1-\delta)^{N/6}\geq \frac{\delta N}{12}=
    \delta',
\end{align}
whenever $\delta N/6\leq 1$.
\end{proof}
\end{lemma}

Lemma~\ref{lemma:Qbound} implies that in order to have $\mathbb{P}^{\cM}(\|Q^*_{\cM}-Q^{\sA}_T\|_{\infty}>\varepsilon)<\delta'$ for every $\cM\in B_{\mathrm{TV}}(\cM_0, \beta)$, one need samples more than $C\frac{N}{\varepsilon^2(1-\gamma)^3}\log(N/{\delta'})$. Thus, a lower bound of learning a $Q$-function with a TV-distance-close prior model is achieved. We state the result formally in Proposition. \ref{prop:lowerboundQ}.
\begin{proposition}\label{prop:lowerboundQ}
For any $\beta>0$, there exists an MDP $\cM_0$, constants $\varepsilon_0$ and $\delta_0$, such that for all $\varepsilon\in(0,\varepsilon_0), \delta\in(0,\delta_0),$ and every $(\cM_0, \beta, \varepsilon, \delta, Q)$-correct RL algorithm $\sA$, the total number of samples needed to learn an \textbf{$\varepsilon$-optimal $Q$-function} for every $\cM\in B_{\mathrm{TV}}(\cM_0, \beta)$ with probability at least $1-\delta$ is
$$\Omega\bigg(\frac{N}{\varepsilon^2(1-\gamma)^3}\log(\frac{N}{\delta})\bigg),$$
where $N$ is the total number of state-action pairs.
\end{proposition}

\subsection*{Step 3: extend to learning policy}
In the last section,  we have  established the sample complexity lower bound for learning a sub-optimal $Q$-function with a TV-distance-close prior model. Next, we extend the result to learning an $\varepsilon$-optimal policy. To this end, we still consider the models $\cM_0, \cM^1$ and $\cM^2$ as before. We first show that given a policy $\pi$ for $\cM\in\mathbb{M}$, it takes much fewer samples to evaluate its $Q$-function compared with the result in Proposition~\ref{prop:lowerboundQ}. Therefore, given $\cM_0$, if an algorithm can produce $\varepsilon$-optimal policies for both $\cM^1$ and $\cM^2$ with significantly fewer samples, then it can produce $\varepsilon$-optimal $Q$-functions with fewer samples as well, which contradicts with Lemma~\ref{lemma:Qbound}. Thus, the lower bound is extended to learning policies.

\begin{lemma}\label{lemma:policyeva}
Given a policy $\pi$ for an MDP $\cM\in\mathbb{M}$ and a generative model of $\cM$, for $\varepsilon\in(0,1]$ and $\delta\in(0,1)$, there exists an algorithm that uses
$$\cO\bigg( \frac{N}{\varepsilon^2(1-\gamma)^3}\log(\frac{1}{\varepsilon(1-\gamma)})\log(\frac{N}{\delta})+ \frac{N}{\varepsilon^2(1-\gamma)^2}\log(\frac{N}{\delta})\bigg)$$
samples to obtain a function $Q:\cS\times\cA\rightarrow \RR$,
such that, with probability at least $1-\delta$,
$\|Q-Q^{\pi}\|_{\infty}\leq \varepsilon$.
\begin{proof}
Since all states in $\cY_2$ have values 0 and $V^{\pi}(y_1(x,a))=Q^{\pi}(x,a)/\gamma$, we only need to evaluate $Q^{\pi}(x,a)$ for all $(x,a)\in\cX\times\cA$. 
Given $\varepsilon>0$, when $T= \lceil \frac{1}{1-\gamma}\log\frac{2}{(1-\gamma)\varepsilon}\rceil$, the tail sum $\mathbb{E}^{\pi}[\sum_{t=T}^{\infty}\gamma^t R(s_t,a_t)|s_0=x, a_0=a]\leq \varepsilon/2$. Thus, we only need to evaluate $Q^{\pi}_T(x):=\mathbb{E}^{\pi}[\sum_{t=0}^{T}\gamma^t R(s_t,a_t)|s_0=x, a_0=a]$ to $\varepsilon/2-$accuracy.  We define a random variable $G^{\pi}_T(x,a):=R(x,a)+\sum_{t=1}^{T}\gamma^t R(s_t,a_t)$ following $\pi$. Then $G^{\pi}_T(x,a)\in[0, 1/(1-\gamma)]$ for all $(x,a)\in\cX\times\cA$ and $\mathbb{E}[G^{\pi}_T(x,a)]=Q^{\pi}_T(x,a)$. By Hoeffding Inequality, we only need to take $\frac{2}{(1-\gamma)^2\varepsilon^2}\log(\frac{|\cS|}{\delta})$ samples of $G^{\pi}_T(x)$ such that the empirical average $\overline{V}^{\pi}_T(x)$ satisfies
$$\mathbb{P}(|\overline{V}^{\pi}_T(x)-V^{\pi}_T(x)|\leq \varepsilon/2)\geq 1-\frac{\delta}{K}.$$ In total, there are $\lceil \frac{2}{(1-\gamma)^3\varepsilon^2}\log(\frac{K}{\delta})\log\frac{2}{(1-\gamma)\varepsilon}\rceil$ transition samples. Repeating the step to all $x\in\cX$ and taking the union bound, we can get a function $\overline{V}^{\pi}_T:\cX\rightarrow \RR$ such that, with probability at least $1-\delta$,
$$\|\overline{V}^{\pi}_T-V^{\pi}\|_{\infty}\leq \varepsilon.$$

Next, we use $\overline{V}_T^{\pi}$ to approximate $Q^{\pi}(x,a)$. Recall that
$$Q^{\pi}(x,a) = R(x,a) + \gamma P(\cdot|x,a)^TV^{\pi}.$$
Define a random variable $v(s,a)$ where $v(s,a)=V^{\pi}(s')$ with probability $p(s'|s,a)$. Then $P(\cdot|s,a)^TV^{\pi}$ is equal to the expectation of $v(s,a)$. By Hoeffding Inequality, with $n:=\frac{1}{2(1-\gamma)^2\varepsilon^2}\log(\frac{|\cS||\cA|}{\delta})$ transition samples from $(s,a)$, the empirical average $\overline{V}(s,a)$ satisfies
$$\mathbb{P}(|\overline{V}(s,a)-P(\cdot|s,a)^TV^{\pi}|\leq \varepsilon)\geq 1-\frac{\delta}{|\cS||\cA|}.$$
Repeating the same step for all $(s,a)\in\cS\times\cA$ and taking the union bound, we can get an $\varepsilon$-correct approximation of $Q^{\pi}$ from $V^{\pi}$ with probability at least $1-\delta$ with the sample complexity $\cO\big( \frac{|\cS||\cA|}{(1-\gamma)^2\varepsilon^2}\log(\frac{|\cS||\cA|}{\delta})\big)$. Adding these two complexities together, we can get the desired result.
\end{proof}
\end{lemma}

Combining Proposition~ \ref{prop:lowerboundQ} and Lemma \ref{lemma:policyeva}, we obtain the final result as stated in the following Theorem. 
\begin{theorem}\label{thm:lowerboundpi}
For any $\beta>0$, there exists some constants $\varepsilon_0, \delta_0, c_1, 4$, a class of MDPs $\mathbb{M}$, and an MDP $\cM_0\in\mathbb{M}$ such that for all $\varepsilon\in(0,\varepsilon_0), \delta\in(0,\delta_0),$ and every $(\beta, \varepsilon, \delta, \pi)$-correct RL algorithm $\sA$ which has full knowledge of $\cM_0$, the total number of samples needed to learn an \textbf{$\varepsilon$-optimal policy} of any MDP $\cM\in \mathbb{M}\cap B_{\mathrm{TV}}(\cM_0, \beta)$ with probability at least $1-\delta$ is
$$\Omega\big(\frac{|\cS||\cA|}{\varepsilon^2(1-\gamma)^3}\log\frac{|\cS||\cA|}{\delta}\big).$$

\begin{proof}
Given $\varepsilon$, one can take $\cM_0$ with $|\cA| \approx \frac{1}{(1-\gamma)\varepsilon}$. Then the complexity of policy evaluation is significantly smaller than learning $Q$-values. Thus, if an algorithm can find an $\varepsilon$-optimal policy of $\cM$ with significantly fewer samples than the lower bound, it can then output an $2\varepsilon$-optimal $Q$ function with samples significantly fewer than the lower bound, which contradicts with Prop. \ref{prop:lowerboundQ}.
\end{proof}
\end{theorem}
}

\section{Further Discussion}\label{sec:discussion}
In this section, we explain our results in a graphical way. We depict parameter relations in \ref{sec:para}, compare $\overline{C}$ and $\underline{C}_s$ in \ref{sec:meet}, illustrate when an approximate model does not help in \ref{sec:num}.
\subsection{Parameter Relations}\label{sec:para}
In our hard case in Section \ref{sec:lowerbdd}, given $\beta\in(0,2)$, we restrict $\gamma\in(\max\{0.4,1-10\beta\},1)$ and then select $p_0^k$ based on whether $p_0(x_k,a_1)\leq \beta/2+\frac{4\gamma-1}{3\gamma}$. We depict the relation between these parameters in the left graph in Figure \ref{fig:para}. One can see that when $\beta<1$, with a small $\gamma$ (below the blue line), we can find a $p_0(x_k,a_1)$ such that $p_0^k=p_0(x_k,a_1)-\beta/2$; otherwise, $p_0^k$ will always be $\frac{4\gamma-1}{3\gamma}$. For the lower bound, we require $\varepsilon<\varepsilon_0$. In the right graph of Figure \ref{fig:para}, we plot $\varepsilon_0$ in terms of $\beta$ and $\gamma$ with $p^k_0=\frac{4\gamma-1}{3\gamma}$. We can see that for a fixed $\gamma$, $\varepsilon_0\approx c\beta$ and the larger $\gamma$ is, the larger $c$ is.
\begin{figure}[htbp!]
        \centering
        \vspace{-10pt}
        \includegraphics[width=.75\textwidth]{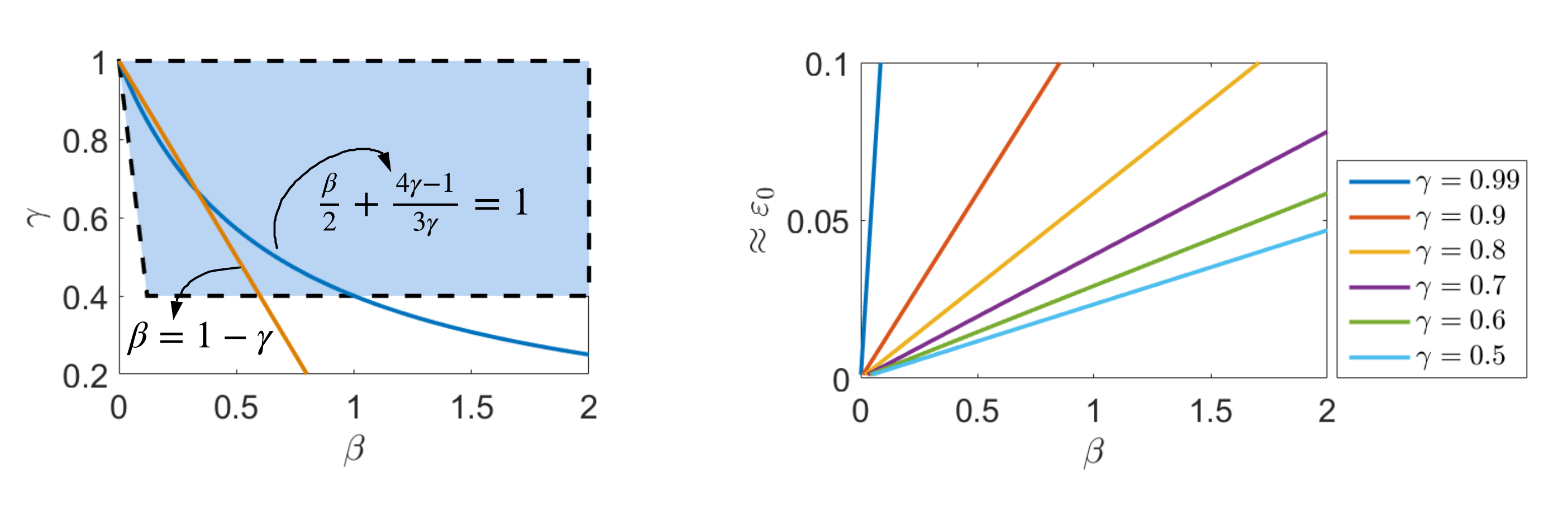}
        \vspace{-15pt}
        \caption{Relations between parameters. The blue area represents $\gamma\in(\max\{0.4, 1-10\beta\}, 1)$.}
        \label{fig:para}
        \vspace{-10pt}
    \end{figure}

\subsection{$\overline{C}$ and $\underline{C}_s$}\label{sec:meet}
We plot the values of $\overline{C}$ and $\underline{C}_s$ in our hard case in Figure \ref{fig:C}. We can see that $\underline{C}_s$ matches $\overline{C}$ up to a constant factor when $\beta$ is away from 0. When $\beta$ is small, the upper bound provides useful information on how much an approximate model can help. As $\beta$ increases, the upper bound becomes trivial and the lower bound is informative on the limitation of an approximate model. This change is not surprising: the larger the distance between two models is, the less helpful an approximate model would be. The value of $\beta$ when the upper bound reduces trivial also depends on $\gamma$: the greater $\gamma$ is, the smaller this value is since the more sensitive the value function becomes in terms of the variation in transition. To have the upper bound meets the lower bound, we need $|\cA^{s}_{\cM_0}(\overline{C})|= |\cA^{s}_{\cM_0}(\underline{C}_s)|$, which is achievable if the value gap in $\cM_0$ is big enough. We illustrate this situation in Figure \ref{fig:gap}.
\cut{\begin{align}
    q^*_s\big[|\cA^s_{\cM_0}(\overline{C})|\big]-q^*_s\big[|\cA^s_{\cM_0}(\overline{C})|+1\big] \geq \overline{C}-\underline{C}_s,\text{~for every~} s\in\cS',
\end{align}
i.e., }

\begin{figure}[htbp!]
\vspace{-10pt}
    \centering
    \includegraphics[width=.9\textwidth]{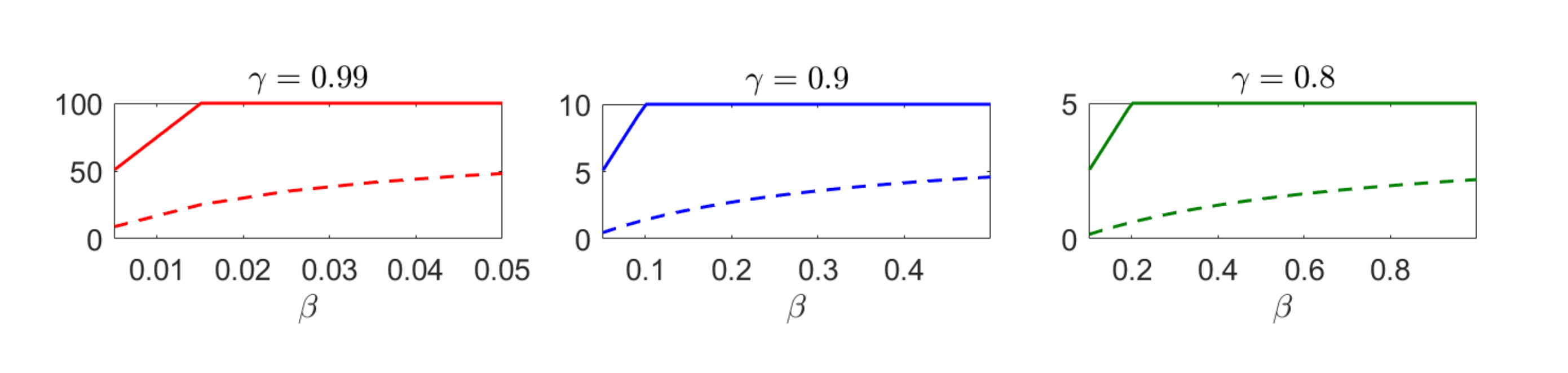}
    \vspace{-15pt}
    \caption{$\overline{C}$ (solid lines) and $\underline{C}_s$ (dashed lines) in the hard case with $p_0^k=(4\gamma-1)/(3\gamma)$. }
    \label{fig:C}
\end{figure}
\begin{figure}[htbp!]
    \centering
    \vspace{-13pt}
    \includegraphics[width=\textwidth]{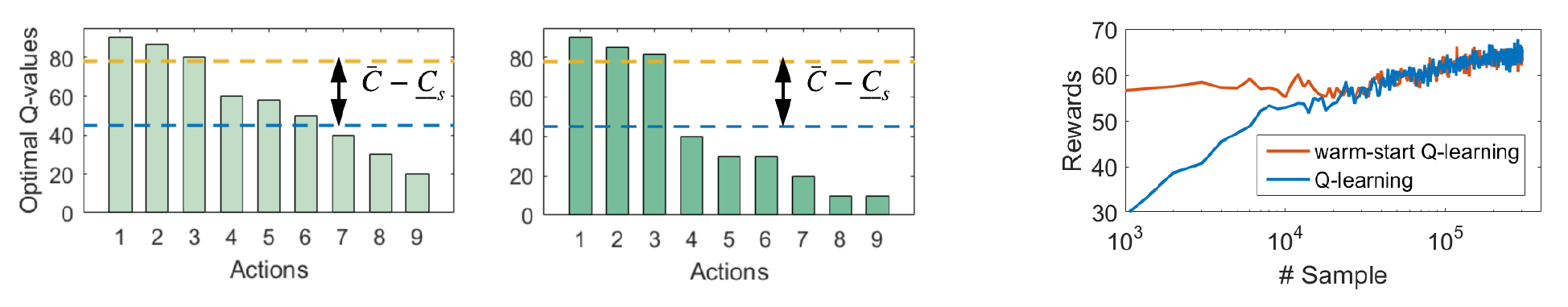}
    \vspace{-20pt}
    \caption{The left two figures depict two types of $\cM_0$. Actions with values above the yellow lines form $\cA^s_{\cM_0}(\underline{C}_s)$; actions with values above the blue lines form $\cA^s_{\cM_0}(\overline{C})$. In the left model, $\cA^s_{\cM_0}(\underline{C}_s)\subset\cA^s_{\cM_0}(\overline{C})$; In the right model, due to a large value gap, $\cA^s_{\cM_0}(\underline{C}_s)=\cA^s_{\cM_0}(\overline{C})$. The right figure is an empirical verification of the worst case.}
    \label{fig:gap}
\end{figure}

\subsection{Empirical Verification of the Worst Case}\label{sec:num}
We further do a numerical demonstration to show when an approximate model does not help. We test on a sailing problem \citep{sailing}. In Figure \ref{fig:gap}, we generate two MDPs $\cM_0$ and $\cM$ with $\cM\in B_{\mathrm{TV}}(\cM_0,0.3)$. We compare the performances of two algorithms: 1. direct Q-learning \citep{watkins1992q} on $\cM$ (blue line); 2. use the full knowledge of $\cM_0$ to compute $Q^*_{\cM_0}$, then use $Q^*_{\cM_0}$ as an initialization for Q-learning on $\cM$ (red line). One can observe a jump-start improvement. However, to reach higher rewards, it takes the same number of samples with or without the prior knowledge. This verifies our worst case result in Corollary \ref{coro:worst}: when the accuracy requirement is high, i.e., $\varepsilon$ is small, the help of an approximate model can be very limited.
\section{Conclusion and Future Work}\label{sec:conclusion}
In this paper, we provide sample complexity results of RL when a prior model that is close under TV-distance is provided. We show that the approximate model can help eliminate sub-optimal actions and reduce the sample complexity of learning a near-optimal policy for the true unknown model. We also show that the help can be rather limited if the value gap in the prior model is small and the precision requirement of the policy is high. For future work, we intend to exploit more structural information in transfer RL and consider other measures of closeness for MDPs.

\cut{
We formalize our problem of knowledge transferring as below.
\begin{problem}\label{problem}
Suppose the target unknown model is $\cM$ and an agent is provided with the full information of an approximate model $\cM_0$ satisfying $\cM\in B_{\mathrm{TV}}(\cM_0, \beta)$, where $\beta >0$ is a known constant.
How many samples does it take to learn an $\varepsilon$-optimal policy for $\cM$ with probability at least $1-\delta$?
\end{problem}

Without additional structure assumption, we answer Problem \ref{problem} through exploiting the value gap property in $\cM_0$ and provide both upper and lower sample complexity bounds:
\begin{theorem}[Main Result]\label{thm:main}
Let $\cM$ be the target model and $\cM_0$ be a source model satisfying $\cM\in\cB_{\text{TV}}(\cM_0,\beta)$. The sample complexity of learning an $\varepsilon$-optimal policy for $\cM$ with probability at least $1-\delta$ is
\begin{align}
    \cO\bigg(\frac{\sum_{s\in\cS'}|\cA^s_{\cM_0}(\bar{C})|}{\varepsilon^2(1-\gamma)^3}\log\Big(\frac{1}{\delta}\Big)\bigg)
\end{align}
for $\varepsilon>0$ and $\delta\in(0,1)$, where $\bar{C}:=\min \{2/(1-\gamma), 2\beta/(1-\gamma)^2\}$; and
\begin{align}
    \Omega\bigg(\frac{\sum_{s\in\cS'}|\cA^s_{\cM_0}(\underline{C}_s)|}{\varepsilon^2(1-\gamma)^3}\log\Big(\frac{1}{\delta}\Big)\bigg)
\end{align}
for $\varepsilon\in\big(0, \frac{\beta\gamma\min_{s\in\cS'}V^*(s)^2}{16}\min (\beta/2, \frac{1-\gamma}{3\gamma})\big)$ and $\delta\in(0, 1/40)$, where
\begin{align}
    \underline{C}_s=\begin{cases}V^*(s) - \frac{9}{12(1-\gamma)-64(1-\gamma)^2\varepsilon+4.5\beta\gamma}, \quad &\text{if~} \beta/2+\frac{4\gamma-1}{3\gamma} \geq 1 \\
V^*(s) - \frac{V^*(s)^2}{V^*(s) +\beta\gamma V^*(s)^2+4\varepsilon(1+\gamma\beta V^*(s) /2)^2}, \quad &\text{otherwise}.
\end{cases}
\end{align}
\end{theorem}

The upper bound and the lower bound share a unifying structure except for different constants $\bar{C}$ and $\underline{C}_s$. Recall the definition of $\cA^s_{\cM_0}(c)$ in Equation \eqref{eq:As}. In Section \ref{sec:analysis}, we will show that $\bar{C}\geq \underline{C}_s$ for every $s\in\cS'$. Thus, $\cA^s_{\cM_0}(\bar{C})\supseteq \cA^s_{\cM_0}(\underline{C}_s)$ and the bounds are well-defined.

Our main results are indeed not hard to grasp. To learn a near-optimal policy, one needs to sufficiently explore all actions with high rewards. Given an approximate model, we can abstract some prior knowledge about which actions are potentially good and then focus on exploring these actions. For each state $s\in\cS'$\footnote{We restrict to the set $\cS'$ since no action exploration is needed for single-action states.}, the set $\cA^s_{\cM_0}(\bar{C})$ characterizes actions we need to \emph{at most} explore and $\cA^s_
{\cM_0}(\underline{C}_s)$ consists of actions we must \emph{at least} explore. That said, we can definitely find a sufficiently good action in $\cA^s_{\cM_0}(\bar{C})$ to constitute an $\varepsilon$-optimal policy for any $\cM\in B_{\text{TV}}(\cM_{0},\beta)$. Conversely, if we ignore any action in $\cA^s_{\cM_0}(\underline{C}_s)$, we can fail the task for some $\cM\in B_{\text{TV}}(\cM_0, \beta)$ no matter which algorithm or learning setting (generative model or online learning) is applied. On the bright side, if $|\cA^s_{\cM_0}(\bar{C})|\ll |\cA^s|$, the sample complexity is reduced compared with learning-from-scratch\footnote{The sample complexity of learning-from-scratch is $=\cO\bigg(\frac{\sum_{s\in\cS'}|\cA^s|}{\varepsilon^2(1-\gamma)^3}\log\Big(\frac{1}{\delta}\Big)\bigg)$.} and knowledge transferring helps. The best case, as written in Corollary \ref{coro:best}, is when an optimal policy of $\cM_0$ is also optimal for $\cM$, then no learning is needed. On the dark side, if $|\cA^s_{\cM_0}(\underline{C}_s)|$ is close to $|\cA^s|$, the prior knowledge does not offer much help. The worst case, as described in Corollary \ref{coro:worst}, is when $|\cA^s_{\cM_0}(\underline{C}_s)|=|\cA^s|$, then with or without the prior knowledge, it has the same order of complexity. In a word, since $|\cA^s_{\cM_0}(\bar{C})|$ and $|\cA^s_{\cM_0}(\underline{C}_s)|$ fully depend on the value gap in the prior model, we use this intrinsic property of the prior model to provide a systematic answer to Problem \ref{problem}.

\begin{corollary}[The Best Scenario]\label{coro:best}
Suppose $\beta<(1-\gamma)/2$ and for every state $s\in\cS'$,
$$|\argmax_a Q^*_{\cM_0}(s,a)|=1 \quad \text{and}\quad \Delta_{\cM_0}^s[2]\geq 2\beta/(1-\gamma)^2.$$ Then the optimal policy for $\cM_0$ is optimal for $\cM\in \cB_{\text{TV}}(\cM_0, \beta)$.
\end{corollary}

\begin{corollary}[The Worst Scenario]\label{coro:worst}
If $|\argmax_a Q^*_{\cM_0}(s,a)|=|\cA^s|$ for every state $s\in\cS'$, then the sample complexity is
\begin{align}
    \Omega\bigg(\frac{\sum_{s\in\cS}|\cA^s|}{\varepsilon^2(1-\gamma)^3}\log\Big(\frac{1}{\delta}\Big)\bigg)
\end{align}
for $\varepsilon\in\big(0, \frac{\beta\gamma\min_{s\in\cS'}V^*(s)^2}{16}\min (\beta/2, \frac{1-\gamma}{3\gamma}\big)$ and $\delta\in(0, 1/40)$.
\end{corollary}

\section{Analysis}\label{sec:analysis}
In this section, we provide the full proof of the main result. We first prove the upper bound in Theorem \ref{thm:main} and Corollary \ref{coro:best} in Sec. \ref{sec:upperbdd}. We then prove the lower bound in Theorem \ref{thm:main} and Corollary \ref{coro:worst} in Sec. \ref{sec:lowerbdd}.

\subsection{Proofs of the Upper Bound}\label{sec:upperbdd}
We begin by Lemma \ref{lemma:Qbound}, which shows that if two MDPs are close under TV-distance, then their optimal $Q$-functions are close component-wisely.
\begin{lemma}\label{lemma:Qbound}
Let $\cM_0:=(\cS, \cA, \vP_0, R_0, \gamma)$ and $\cM:=(\cS, \cA, \vP, R, \gamma)$ be two MDPs with optimal $Q$-functions $Q^*_0$ and $Q^*$, respectively. If $d_{\text{TV}}(\cM_0, \cM)\leq \beta$, then $\|Q^*_0-Q^*\|_{\infty}\leq \min\{1/(1-\gamma), \beta/(1-\gamma)^2\}$.
\end{lemma}

\begin{proof}
Since $R(s,a,s')\in[0,1]$, the bound $1/(1-\gamma)$ is a direct result. For any deterministic policy $\pi$, we denote by $Q^{\pi}_0$ and $Q^{\pi}$ the action-value functions and $V^{\pi}_0$ and $V^{\pi}$ the state-value functions of running $\pi$ in $\cM_0$ and $\cM$, respectively. Then, by definition, for every $(s,a)\in\cS\times\cA$,
\begin{align}
    |Q_0^{\pi}(s,a)-Q^{\pi}(s,a)| &= |R_0(s,a) + \gamma P_0(\cdot|s,a)^\top V_0^{\pi} - R(s,a)-\gamma P(\cdot|s,a)^\top V^{\pi}|\\
    &\leq |R_0(s,a)-R(s,a)| + \gamma\cdot |P_0(\cdot|s,a)^{\top} V_0^{\pi}-P(\cdot|s,a)^\top V_0^{\pi}| \\
    &\quad + \gamma\cdot |P(\cdot|s,a)^{\top} V_0^{\pi}-P(\cdot|s,a)^\top V^{\pi}|\\
    &\leq \beta + \gamma\cdot \|P_0(\cdot|s,a)-P(\cdot|s,a)\|_1\cdot \|V_0^{\pi}\|_{\infty} + \gamma \cdot \|V_0^{\pi}-V^{\pi}\|_{\infty}\\
    &\leq \beta + \gamma\beta/(1-\gamma)+\gamma\|Q_0^{\pi}-Q^{\pi}\|_{\infty}.
\end{align}
Thus, $\|Q^{\pi}_0-Q^{\pi}\|_\infty\leq \beta/(1-\gamma)+\gamma\cdot \|Q^{\pi}_0-Q^{\pi}\|_\infty$ and $\|Q^{\pi}_0-Q^{\pi}\|_\infty\leq \beta/(1-\gamma)^2$. Now, we denote by $\pi^*_0$ a deterministic optimal policy of $\cM_0$ and $\pi^*$ a deterministic optimal policy of $\cM$. Then
\begin{align}
    Q^{\pi^*}_0-Q^{\pi^*}\leq Q^*_0-Q^*\leq Q^{\pi^*_0}_0-Q^{\pi^*_0}.
\end{align}
Thus, $\|Q^*_0-Q^*\|_{\infty}\leq \max\{~\|Q^{\pi^*}_0-Q^{\pi^*}\|_{\infty},~ \|Q^{\pi^*_0}_0-Q^{\pi^*_0}\|_{\infty}~\}\leq \min\{1/(1-\gamma), \beta/(1-\gamma)^2\}$.
\end{proof}

Lemma \ref{lemma:Qbound} is not a surprising but useful result. Given the full information of the prior model $\cM_0$, we can approach $Q^*_0$ to any accuracy with any planning algorithm. By Lemma \ref{lemma:Qbound}, we then obtain an interval estimation for each entry of $Q^*$. Based on the estimation, we get to identify \emph{good} actions for each state, i.e., actions that have a chance to be $\argmax_a Q^*_{\cM}(s,a)$. In later learning stage, we only need to explore these good actions. If the number of good actions is significantly smaller than $|\cA|$ for every state, then dimension reduction is achieved and the sample complexity is reduced.

\begin{figure}[t]
    \centering
    \includegraphics[width=.9\textwidth]{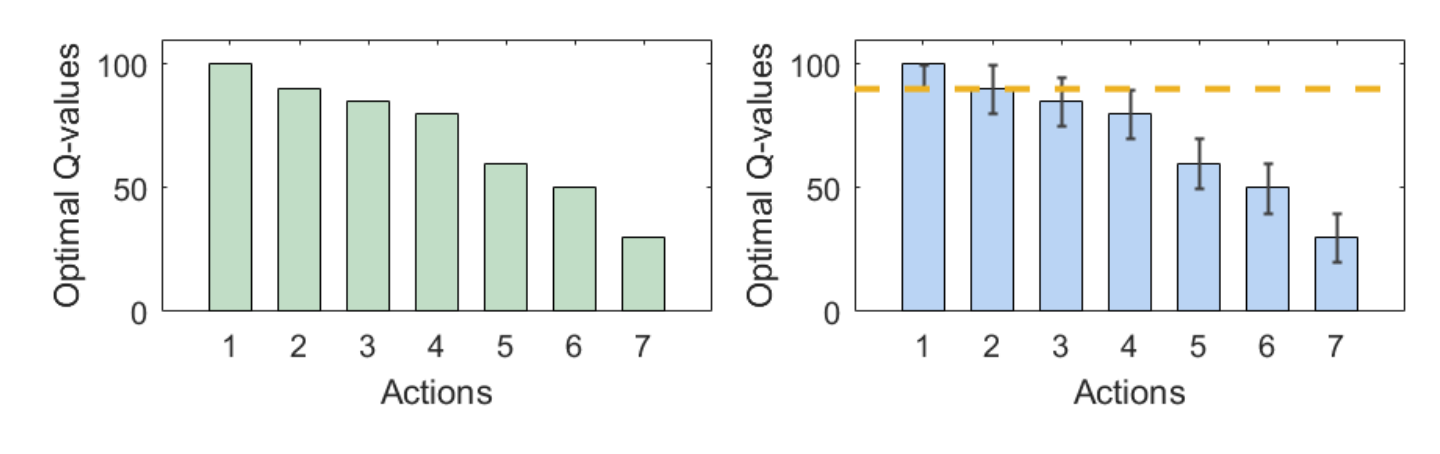}
    \caption{An illustration of how optimal $Q$-values of the prior model $\cM_0$ help contract the action space in the target model $\cM$. Suppose $\gamma=0.99$ and $\beta/(1-\gamma)^2=10$. Given $Q^*_{\cM_0}(s,\cdot)$ (the left graph), we have interval estimation of $Q^*_{\cM}(s,\cdot)$ (the error bars in the right graph). Actions that are potentially optimal in $\cM$ at $s$ should satisfy Equation \eqref{eq:goodaction}, i.e., the top end of the error bars lie above the yellow dashed line ($=90$). Therefore, at state s, only actions $1, 2$ and $3$ are considered as valid candidates and we will only explore these actions in the later learning.}
    \label{fig:gap}
\end{figure}

Next, we give a rigorous definition of good actions. We ignore the error in computing $Q^*_0$ since it can be arbitrarily small. Let $\bar{C} = \min\{2/(1-\gamma), 2\beta/(1-\gamma)^2\}$. For each state $s$ in $\cM$, a good action should satisfy
\begin{align}\label{eq:goodaction}
    &Q^*_0(s,a)+\bar{C}/2 > \max_a Q^*_0(s,a)-\bar{C}/2, \text{~or equivalently,~ }a\in\cA^s_{\cM_0}(\bar{C}).
\end{align}
The definition is intuitive. Suppose $a^*_s\in\argmax_a Q^*_0(s,a)$. By Lemma \ref{lemma:Qbound}, $Q^*_{\cM}(s,a_s^*)\geq \max_a Q^*_0(s,a)-\bar{C}/2$. Thus, a valid candidate of $\argmax_a Q^*_{\cM}(s,a)$ should have a value at least $\max_a Q^*_0(s,a)-\bar{C}/2$. Since the best an action could achieve in $\cM$ is $Q^*_0(s,a)+\bar{C}/2$, we require $Q^*_0(s,a)+\bar{C}/2 > Q^*_0(s,a^*_s)-\bar{C}/2$ for an action to be considered potentially better than $a^*_s$. Otherwise, we can just use $a^*_s$. See Figure \ref{fig:gap} for a simple illustration.

\begin{proof}[Proof of the Upper Bound in Theorem \ref{thm:main}]
Since we only need to explore actions in $\cA^s_{\cM_0}(\bar{C})$ for each state in $\cM$, by the upper bound result in \cite{azar2013minimax}, the sample complexity of learning an $\varepsilon$-optimal policy with probability at least $1-\delta$ is $$\cO\bigg(\frac{\sum_{s\in\cS'}|\cA^s_{\cM_0}(\bar{C})|}{\varepsilon^2(1-\gamma)^3}\log\Big(\frac{1}{\delta}\Big)\bigg)$$.
\end{proof}
If $\sum_{s\in\cS}|\cA^s_{\cM_0}(\bar{C})|$ is significantly smaller than $N$, then the sample complexity with prior knowledge is significantly smaller than that of learning-from-scratch. In this case, transfer knowledge succeeds. In Corollary \ref{coro:best}, the conditions imply the best scenario, i.e., $|\cA^s_{\cM_0}(\bar{C})|=1$ for every state. Then in $\cM_0$, there is only a unique optimal policy $\pi^*_0(s):=\argmax_a Q^*_0(s,a)$ and $\pi_0^*$ is also optimal for $\cM$. No learning is needed.

\subsection{Proofs of the Lower Bound}\label{sec:lowerbdd}
In this section, we prove the lower bound in Theorem \ref{thm:main} as well as Corollary \ref{coro:worst}.
Before starting the proof, we give the following definition about the correctness of RL algorithms.

\begin{definition}(($\cM_0, \beta, \varepsilon, \delta$)-correctness)\label{def:epsilonalgorithm}
Given $\beta>0$ and a prior model $\cM_0$, we say that an RL algorithm $\sA$ is $(\cM_0, \beta, \varepsilon, \delta)$-correct if for every $\cM\in B_{\mathrm{TV}}(\cM_0, \beta)$, $\sA$ can output an $\varepsilon$-optimal policy with probability at least $1-\delta$.
\end{definition}

Next, we construct a class of MDPs. We are going to select one model $\cM_0$ from the class as prior knowledge. Then we show that if an RL algorithm $
\sA$ learns with samples significantly fewer than the lower bound, there would always exist an MDP $\cM\in B_{\mathrm{TV}}(\cM_0, \beta)$ such that $\sA$ cannot be $(\cM_0, \beta, \varepsilon, \delta)$-correct. Hence, the lower bound complexity is established.
\cut{we construct a class of hard MDPs which is a generalization of the simple example in Sec. \ref{sec:toyexample1}. We select one model $\cM_0$ from the class and use its full information as the prior knowledge. Then, we consider any algorithm that is $(\cM_0, \beta, \varepsilon, \delta)$-correct.}

\cut{learning $\varepsilon$-optimal policies for

another two models $\cM^1$ and $\cM^2$ that are both $\beta-$close to $\cM_0$ under TV-distance. For the later two models, we assume access to their generative models respectively. Note that these two learning tasks (for $\cM^1$ or $\cM^2$) should have the same sample complexity since they are provided with the same prior knowledge and sample accesses. Then, we show that when $\varepsilon<\varepsilon_0(\beta)$, if the sample number is less than $C\frac{N}{(1-\gamma)^3\varepsilon^2}\log(\frac{N}{\delta})$, at least one of the learning tasks cannot be accomplished. Therefore, a lower bound is achieved.}


\cut{If there is an $(\cM_0, \beta,\varepsilon, \delta, Q)$-correct algorithm $\sA$, then it can return $\varepsilon$-optimal $Q$-values for both $\cM^1$ and $\cM^2$ with probability at least $1-\delta$. We then restrict $\varepsilon$ to be very small and select $\cM^1$ and $\cM^2$ smartly such that these two models are super close to each other but still the $Q$-values differ by more than 2$\varepsilon$. Then the difficulty of the problem is equivalent to distinguishing two very close values where the $\beta$ information is useless. By then, we are able to show that with number of samples significantly less than $C\frac{|\cS||\cA|}{(1-\gamma)^3\varepsilon^2}\log(\frac{|\cS|\cA|}{\delta})$, at least for one of $\cM^1$ and $\cM^2$, $\sA$ cannot produce an $\varepsilon$-optimal $Q$-values. Therefore, the lower bound result is obtained. }

\paragraph{Construction of the Hard Case}
We define a family of MDPs $\mathbb{M}$. These MDPs have the structure as depicted in Figure \ref{fig:mdp}. The state space $\cS$ consists of three disjoint subsets $\cX$ (gray nodes), $\cY_1$ (green nodes), and $\cY_2$ (blue nodes). The set $\cX$ includes $K$ states $\{x_1, x_2, \dots, x_K\}$ and each of them has $L$ available actions $\{a_1,a_2,\dots,a_L\}=:\cA$. States in $\cY_1$ and $\cY_2$ are all of single-action. In total, there are $N:=3KL$ state-action pairs. For state $x\in\cX$, by taking action $a\in\cA$, it transitions to a state $y_1(x,a)\in\cY_1$ with probability 1. Note that such a mapping is one-to-one from $\cX\times\cA$ to $\cY_1$. For state $y_1(x,a)\in\cY_1$, it transitions to itself with probability $p_{\cM}(x,a)\in[0,1]$ and to a corresponding state $y_2(y_1) \in \cY_2$ with probability $1-p_{\cM}(x,a)$. $p_{\cM}(x,a)$ can be different for different models. All states in $\cY_2$ are absorbing. The reward function $R(s,a,s')=1$ if $s'\in\cY_1$; otherwise 0.

\begin{figure}[htbp!]
  \centering
  \includegraphics[width=.98\textwidth]{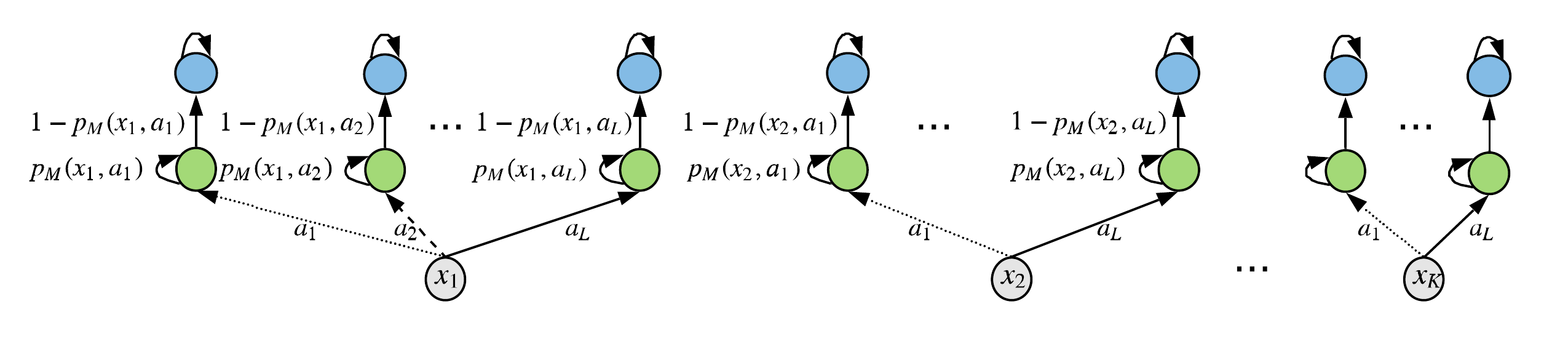}\\
  \caption{The class of MDPs considered to prove the lower bound in Theorem \ref{thm:main}. Nodes represent states and arrows show transitions. $\cX$ consists of all grey nodes. $\cY_1$ comprises of all green nodes. Blue nodes form $\cY_2$.}
  \label{fig:mdp}
\end{figure}

$\mathbb{M}$ is a generalization of a multi-armed bandit problem used in \citealt{mannor2004sample} to prove a lower bound on bandit learning. A similar example is also shown in \citealt{azar2013minimax} to prove a lower bound on reinforcement learning without any prior knowledge. For an MDP $\cM\in\mathbb{M}$, it is fully determined by the parameter set $\{p_{\cM}(x_k,a_l), k\in[K], l\in[L]\}$. And its $Q$-function has the values:
\begin{align}
Q_{\cM}(x,a) = \frac{1}{1-\gamma p_{\cM}(x,a)}, \quad \forall~ (x,a)\in\cX\times\cA.
\end{align}

Given $\beta>0$, we consider a prior model $\cM_0\in\mathbb{M}$ and a group of models in $\mathbb{M}$ that are $\beta$-close to $\cM_0$ under the TV-distance. Since the TV-distance between two models is at most 2 (see Definition in Equation \ref{eq:TV}), we restrict $\beta\in(0,2)$. Fixing $\beta$, we further restrict the discount factor $\gamma \in \big(\max\{0.4, 1-10\beta\}, 1\big)$.

\paragraph{Prior Model $\cM_0$}
We simplify the notation $p_{\cM_0}(x_k,a_l)$ as $p_0(x_k,a_l)$. Without loss of generality, we assume
\begin{align}\label{eq:p0}
    1>p_0(x_k,a_1)\geq p_0(x_k,a_2)\geq \cdots \geq p_0(x_k,a_L)\geq 0, \text{~for every } k\in[K].
\end{align}
Thus, in $\cM_0$, the $Q$-values from $a_1$ to $a_L$ are monotonically non-increasing. Specifically, given $\beta$ and $\gamma$ satisfying the aforementioned conditions, we require $p_0(x_k,a_1)\in(\frac{4\gamma-1}{3\gamma},1)$ for every $k\in[K]$.

\paragraph{Hypotheses of $\cM$} Given $\cM_0$, we first define $p_0^k := \max\{p_0(x_k, a_1)-\beta/2, \frac{4\gamma-1}{3\gamma}\}$ for every $k\in[K]$. Then let $\varepsilon_0 := \min_{k\in[K]} \Big\{\frac{\beta\gamma(1-p_0^k)}{16(1-\gamma p_0^k)^2}\Big\}$. Given $\varepsilon\in(0, \varepsilon_0)$, we denote by $\alpha_1^k$ be the solution to
\begin{align}\label{eq:a1}
    \frac{1}{1-\gamma\big(p_0^k+\alpha_1^k\big)}-\frac{1}{1-\gamma p_0^k} = 2\varepsilon
\end{align}
and $\alpha_2^k := 4(1-\gamma p_0^k)^2\varepsilon/\gamma$ such that
\begin{align}\label{eq:a2}
      \frac{1}{1-\gamma\big(p_0^k+\alpha_2^k\big)}-\frac{1}{1-\gamma \big(p_0^k+\alpha_1^k\big)} \geq 2\varepsilon.
\end{align}
Furthermore, for every $k\in[K]$, we define an integer index
\begin{align}\label{eq:Lk}
    L_k:=\bigg\vert~\Big\{l\in[L]~\Big\vert ~\big\vert~p_0^k+\alpha_2^k-p_0(x_k,a_l)~\big\vert\leq\beta/2\Big\}~\bigg\vert.
\end{align}

\begin{lemma}
Given $\varepsilon\in (0, \varepsilon_0)$, the set defined in Equation \eqref{eq:Lk} is equal to $\{ 1\leq l\leq L_k\}$.
\end{lemma}
\begin{proof}
Due to the monotonicity in Equation \eqref{eq:p0}, the set in Equation \eqref{eq:Lk} should be consecutive integers. So we only to show $l=1$ is an included. By definition of $p_0^k$, it holds that $ p_0(x_k,a_1)-\beta/2\leq p_0^k<p_0(x_k,a_1)$. When $\varepsilon \in (0, \varepsilon_0)$, we have $0<\alpha_2^k < \beta/2$. Thus, $p_0(x_k,a_1)<p_0^k+\alpha_2^k<p_0(x_k,a_1)+\beta/2$ and $l=1$ is in the set.
\end{proof}
Now, we are ready to define $1+\sum_{k\in[K]}L_k$ possibilities of $\cM$:
\begin{align}\label{eq:M}
\cM_{1}: \text{for every } k\in[K], ~ & \begin{cases}p_{\cM_1}(x_k,a_1) = p_0^k+\alpha_1^k,\\ p_{\cM_1}(x_k,a_l)=p_0^k, \quad 2\leq l\leq L_k,\quad \\
p_{\cM_1}(x_k,a_l)=p_0(x_k,a_l), \quad l >L_k;\end{cases}\\
\text {for every } k\in[K], 1<l\leq L_k, \quad \cM_{k,l}:~& \begin{cases}
p_{\cM_{k,l}}(x_{k},a_l) = p_0^k+\alpha_2^k\\
p_{\cM_{k,l}}(x_{k'},a_{l'}) = p_{\cM_1}(x_{k'},a_{l'}), \quad \forall~ (k',l')\neq(k,l).\end{cases}
\end{align}

\begin{lemma}
When $\varepsilon\in(0, \varepsilon_0)$, all possibilities defined in \eqref{eq:M} $\in B_{\text{TV}}(\cM_0, \beta)$.
\end{lemma}
\begin{proof}
We first verify $\cM_1\in B_{\text{TV}}(\cM_0, \beta)$. When $\varepsilon\in(0, \varepsilon_0)$, by definition we have $0<\alpha_1^k<\alpha_2^k<\beta/2$ and $p_0(x_k,a_1)-\beta/2\leq p_0^k<p_0(x_k,a_1)$. Then it holds that $$p_0(x_k,a_1)-\beta/2<p_0^k+\alpha_1^k < p_0(x_k,a_1)+\alpha_1^k < p_0(x_k,a_1)+\beta/2.$$ Thus, $|p_0^k+\alpha_1^k-p_0(x_k,a_1)|\leq \beta/2.$
For $2\leq l \leq L_k$, if $p_0(x_k,a_l)\geq p_0^k$, then $p_0(x_k,a_l)-p_0^k\leq p_0(x_k,a_1)-p_0^k\leq \beta/2$; otherwise, $p_0^k-p_0(x_k,a_l)\leq p_0^k+\alpha^k_2-p_0(x_k,a_l)\leq \beta/2$ (Equation \eqref{eq:Lk}). Hence, $\cM_1\in B_{\text{TV}}(\cM_0, \beta)$. For each $\cM_{k,l}$, the validity directly follows Equation \eqref{eq:Lk} and $\cM_1\in B_{\text{TV}}(\cM_0, \beta)$.
\end{proof}

We refer to these possibilities as \emph{hypotheses} of $\cM$. If we take a closer look, it is easy to observe that in $\cM_1$, for every $x_k\in\cX$, $a_1$ is the optimal action and $a_2$ to $a_{L_k}$ are the second-best actions with a value only $2\varepsilon$ less than the optimal (due to Equation \eqref{eq:a1}). The rest actions have even lower values. $\{\cM_{k,l}\}$ are built on top of $\cM_1$ by raising up the value of the action $a_l$ for state $x_k$ by at least $4\varepsilon$ correspondingly (combining Equation \eqref{eq:a1} and \eqref{eq:a2}). Thus, in $\cM_{k,l}$, the best action for $x_{k'}~(k'\neq k)$ is still $a_1$ but the best action for $x_k$ is $a_l$. See Figure \ref{fig:modelQ} for illustration.

\begin{figure}
    \centering
    \includegraphics[width=\textwidth]{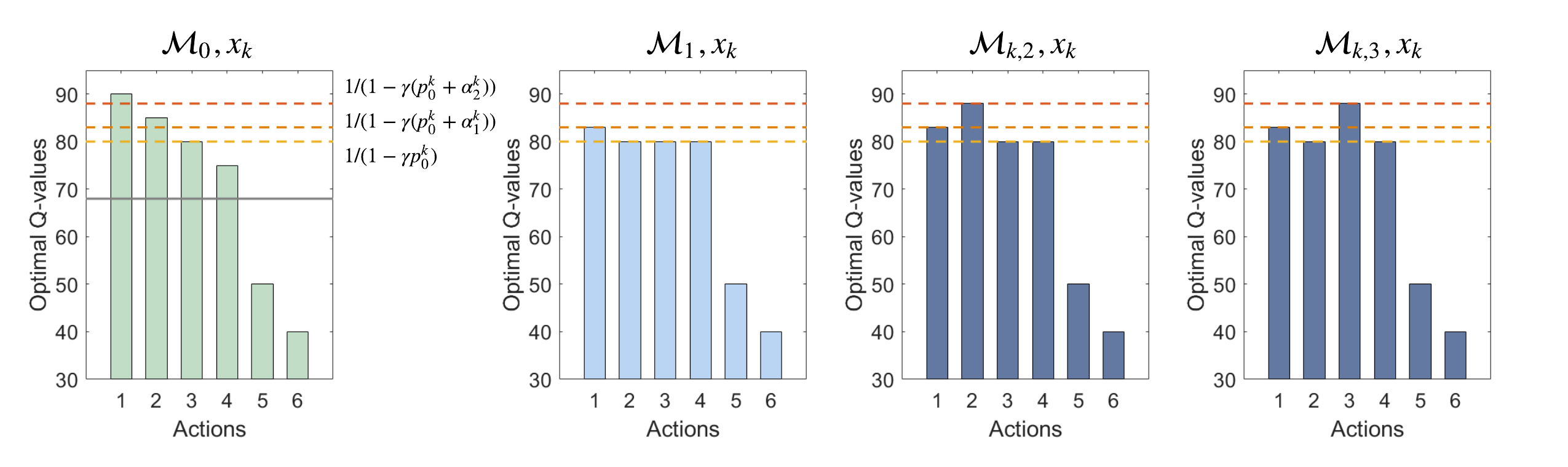}
    \caption{The $Q$-values of $\cM_0$, $\cM_1$, $\cM_{k,2}$, and $\cM_{k,3}$ at state $x_k$ with 6 actions. The three dashed lines indicates the values: $\frac{1}{1-\gamma p_0^k}$, $\frac{1}{1-\gamma (p_0^k+\alpha_1^k)}$, and $\frac{1}{1-\gamma (p_0^k+\alpha_2^k)}$. The grey line in $\cM_0$ ($=\frac{1}{1-\gamma(p_0^k+\alpha_2^k-\beta/2)}$) is the threshold. Only actions with values above it are in the set defined in Equation \eqref{eq:Lk}. In this example, $L_k=4$. Note that for states $x_{k'}~ (k'\neq k)$, $\cM_{k,2}$ and $\cM_{k,3}$ have the same shape as $\cM_1$.}
    \label{fig:modelQ}
\end{figure}

Every hypothesis gives a probability measure over the same sample space. We denote by $\mathbb{E}_1$, $\mathbb{P}_1$ and $\mathbb{E}_{k,l}$, $\mathbb{P}_{k,l}$ the expectation and probability under hypothesis $\cM_1$ and $\cM_{k,l}$, respectively. These probability measures capture both the randomness in the corresponding MDP and the randomization carried out by the algorithm $\sA$ like its sampling strategy. It is worth mentioning that in \citealt{azar2013minimax}, the authors implicitly assume that the sampling numbers to different states are determined before the start of the algorithm and do not change during learning (this is due to their \emph{conditionally independence} argument in Lemma 18). Such an assumption does not apply to adaptive sampling strategy. In our result, adaptive sampling is included.

In the sequel, we fix $\varepsilon\in(0, \varepsilon_0)$ and $\delta\in(0, 1/40)$.
Let
$$t^* = \frac{c_1}{(1-\gamma)^3\varepsilon^2}\log\Big(\frac{1}{4\delta}\Big),$$
where $c_1>0$ is to be determined later. We denote by $T_{k,l}$ the number of samples that algorithm $\sA$ calls from the generative model with input state $y_1(x_{k},a_{l})$ till $\sA$ stops (these sample calls are not necessarily consecutive). For every $k\in[K], 1<l\leq L_k$, we define the following events:
\begin{align}
A_{k,l} &= \{T_{k,l}\leq 4t^*\}, \quad B_{k,l} = \{ \sA \text{ outputs a policy } \pi \text{ with } \pi(x_k) = a_1\},\\
C_{k,l} &= \Big\{\max_{1\leq T_{k,l}\leq 4t^*} \big\vert p_0^k\cdot T_{k,l}- S_{k,l}(T_{k,l})\big\vert \leq\sqrt{16t^* \cdot p_0^k\cdot (1-p_0^k)\log(1/4\delta)}\Big\},
\end{align}
where $S_{k,l}(T_{k,l})$ is the sum of rewards (non-discounted) by calling the generative model $T_{k,l}$ times with input state $y_1(x_k,a_l)$. For these events, we have the following lemmas.

\begin{lemma}
For any $k\in[K], 1<l\leq L_k$, if ~$\mathbb{E}_1[T_{k,l}]\leq t^*$, $\mathbb{P}_1(A_{k,l})> 3/4$.
\end{lemma}
\begin{proof}
$$t^*\geq \mathbb{E}_1[T_{k,l}]> 4t^*\mathbb{P}_1(T_{k,l}>4t^*)=4t^*(1-\mathbb{P}_1(T_{k,l}\leq 4t^*)).$$
Thus, $\mathbb{P}_1(A_{k,l})> 3/4$.
\end{proof}

\begin{lemma}
For any $k\in[K], 1<l\leq L_k$, $\mathbb{P}_1(C_{k,l})> 3/4$.
\end{lemma}
\begin{proof}
Let $\epsilon:=\sqrt{16t^* \cdot p_0^k\cdot (1-p_0^k)\log(1/4\delta)}$. When $1<l\leq L_k$, under hypothesis $\cM_1$, $p_{\cM_1}(x_k,a_l)=p_0^k$. By definition, the instant rewards from state $y_1(x_k,a_l)$ are i.i.d. Bernoulli$(p_0^k)$ random variables and
$p_0^k\cdot T_{k,l}-S_{k,l}(T_{k,l})$ is a martingale. Using Doob's inequality (\cite[Theorem 4.4.2]{durrett2019probability}), we have the following bound:
\begin{align}
    \mathbb{P}_1\Big(\max_{1\leq T_{k,l}\leq 4t^*} \big\vert p_0^k\cdot T_{k,l}- S_{k,l}(T_{k,l})\big\vert \geq \sqrt{16t^* \cdot p_0^k\cdot (1-p_0^k)\log(1/4\delta)}\Big)\leq \frac{\mathbb{E}_1\Big[\big(4t^*\cdot p_0^k- S_{k,l}(4t^*)\big)^2\Big]}{16t^* \cdot p_0^k\cdot (1-p_0^k)\log(1/4\delta)}.
\end{align}
Since $\mathbb{E}_1[(4t^*\cdot p_0^k- S_{k,l}(4t^*))^2]=4t^*p_0^k(1-p_0^k)$ and $\delta<1/40$, we obtain that
\begin{align}
    \mathbb{P}_1(C_{k,l})\geq 1-1/(4\log(1/4\delta))>3/4.
\end{align}
\end{proof}

Since $\sA$ is $(\cM_0, \beta, \varepsilon, \delta)$-correct, it should return a policy $\pi$ such that when $\cM=\cM_1$, $\pi(x_{k})=a_1$ for every $k\in[K]$ with probability at least $1-\delta$, i.e. $\mathbb{P}_1\big(~B_{k,l}, ~\text{for all } k\in[K], 1<l\leq L_k~ \big)\geq 1-\delta>3/4$. We define the event $\cE_{k,l}:=A_{k,l}\cap B_{k,l} \cap C_{k,l}$. Combining the results above, it holds that
\begin{align}
\mathbb{P}_1(\cE_{k,l})>1-3/4=1/4, \quad \forall~ k\in[K],~ 1<l\leq L_k.
\end{align}
Next, we show that if the expectation of number of samples taken through $\sA$ on any $y_1(x_k,a_l)$ is less than $t^*$, with probability $>\delta$, $\sA$ can not return an $\varepsilon$-policy for the corresponding hypothesis $\cM_{k,l}$.

\begin{lemma}\label{lemma:newversion}
For any $k\in[K], 1<l\leq L_k$, if ~$\mathbb{E}_1[T_{k,l}]<t^*$, then $\mathbb{P}_{k,l}(B_{k,l})>\delta$.
\end{lemma}
\begin{proof}
Given $k\in[K]$ and $1<l\leq L_k$, we denote by $W$ the length-$T_{k,l}$ random sequence of the instant rewards by calling the generative model $T_{k,l}$ times with the input state $y_1(x_k,a_l)$. If $\cM=\cM_1$, this is an i.i.d. Bernoulli$(p_0^k)$ sequence; if $\cM=\cM_{k,l}$, this is an i.i.d Bernoulli$(p_0^k+\alpha_2^k)$ sequence. We define the likelihood function $L_{k,l}$ as
$$L_{k,l}(w) = \mathbb{P}_{k,l}(W=w)$$
for every possible realization $w$. We simplify the previous notation $S_{k,l}(T_{k,l})$ as $S_{k,l}$. Then we compute the following likelihood ratio
\begin{align}
\frac{L_{k,l}(W)}{L_1(W)}=& \frac{(p_0^k+\alpha_2^k)^{S_{k,l}}(1-p_0^k-\alpha_2^k)^{T_{k,l}-S_{k,l}}}{(p_0^k)^{S_{k,l}}(1-p_0^k)^{T_{k,l}-S_{k,l}}}= \left(1+\frac{\alpha_2^k}{p_0^k}\right)^{S_{k,l}}\left(1-\frac{\alpha_2^k}{1-p_0^k}\right)^{T_{k,l}-S_{k,l}}\\
=& \left(1+\frac{\alpha_2^k}{p_0^k}\right)^{S_{k,l}}\left(1-\frac{\alpha_2^k}{1-p_0^k}\right)^{S_{k,l}\frac{1-p_0^k}{p_0^k}}\left(1-\frac{\alpha_2^k}{1-p_0^k}\right)^{T_{k,l}-S_{k,l}/p_0^k}.
\end{align}
Since $\gamma>0.4$ and $p_0^k \geq \frac{4\gamma-1}{3\gamma}$, we have $p_0^k>1/2$. By our choice of $\alpha_2^k$ and $\varepsilon$, it holds that $\alpha_2^k/(1-p_0^k)\leq \beta/4\in(0, 1/2]$ and $\alpha_2^k/p_0^k\leq \beta(1-p_0^k)/(4p_0^k)\in(0,1/2)$. With the fact that $\log (1-u) \geq-u-u^{2}$ for $u\in[0,1/2]$ and $\exp (-u) \geq 1-u$ for $u\in[0,1]$, we have that
\begin{align}
\bigg(1-\frac{\alpha_2^k}{1-p_0^k}\bigg)^{\frac{1-p_0^k}{p_0^k}}& \geq \exp \left(\frac{1-p_0^k}{p_0^k}\left(-\frac{\alpha_2^k}{1-p_0^k}-\Big(\frac{\alpha_2^k}{1-p_0^k}\Big)^{2}\right)\right) \geq\left(1-\frac{\alpha_2^k}{p_0^k}\right)\left(1-\frac{(\alpha_2^k)^{2}}{p_0^k\cdot(1-p_0^k)}\right).
\end{align}
Thus,
\begin{align} \frac{L_{k,l}(W)}{L_{1}(W)} & \geq\left(1-\frac{(\alpha_2^k)^2}{(p_0^k)^2}\right)^{S_{k,l}}\left(1-\frac{(\alpha_2^k)^{2}}{p_0^k\cdot(1-p_0^k)}\right)^{S_{k,l}}\bigg(1-\frac{\alpha_2^k}{1-p_0^k}\bigg)^{T_{k,l}-S_{k,l}/p_0^k} \\ & \geq\left(1-\frac{(\alpha_2^k)^2}{(p_0^k)^2}\right)^{T_{k,l}}\left(1-\frac{(\alpha_2^k)^{2}}{p_0^k\cdot (1-p_0^k)}\right)^{T_{k,l}}\bigg(1-\frac{\alpha_2^k}{1-p_0^k}\bigg)^{T_{k,l}-S_{k,l}/p_0^k}
\end{align}
due to $S_{k,l}\leq T_{k,l}$. Next, we proceed on the event $\cE_{k,l}$. By definition, if $\cE_{k,l}$ occurs, $A_{k,l}$ also occurs. Using $\log(1-u)\geq -2u$ for $u\in[0,1/2]$, it follows that
\begin{align}
   \left(1-\frac{(\alpha_2^k)^2}{(p_0^k)^2}\right)^{T_{k,l}} \geq& \left(1-\frac{(\alpha_2^k)^2}{(p_0^k)^2}\right)^{4t^*} \geq \exp \left(-8t^* \frac{(\alpha_2^k)^2}{(p_0^k)^2}\right) = \exp \left(-8\frac{c_1}{(1-\gamma)^3\varepsilon^2}\log(1/4\delta)\frac{16(1-\gamma p_0^k)^4\varepsilon^2}{\gamma^2(p_0^k)^2}\right)\\
\geq &\exp \Big(-128 c_1 \log(1/4\delta) \frac{(1-\frac{4\gamma-1}{3})^4}{(1-\gamma)^3\gamma^2(\frac{4\gamma-1}{3\gamma})^2} \Big)=\exp \Big(-128 c_1 \log(1/4\delta)\frac{256(1-\gamma)}{9(4\gamma-1)^2} \Big)\\
\geq &\exp \Big(-128 c_1 \log(1/4\delta)\frac{256}{9*0.6} \Big)\geq\left(4\delta\right)^{6100c_1},
\end{align}
where the second line follows $p_0^k\geq \frac{4\gamma-1}{3\gamma}$. Using $\log(1-u)\geq -2u$ for $u\in[0,1/2]$, we also obtain
\begin{align}
    \left(1-\frac{(\alpha_2^k)^{2}}{p_0^k\cdot(1-p_0^k)}\right)^{T_{k,l}} &\geq \left(1-\frac{(\alpha_2^k)^{2}}{p_0^k\cdot(1-p_0^k)}\right)^{4t^*}\geq \exp \left(-8t^* \frac{(\alpha_2^k)^{2}}{p_0^k(1-p_0^k)}\right)\\
    &=\exp \left(-8\frac{c_1}{(1-\gamma)^3\varepsilon^2}\log(1/4\delta)\frac{16(1-\gamma p_0^k)^4\varepsilon^2}{\gamma^2p_0^k(1-p_0^k)}\right)\\
    &\geq\exp \Big(-128 c_1 \log(1/4\delta) \frac{(1-\frac{4\gamma-1}{3})^4}{(1-\gamma)^3\gamma^2(\frac{4\gamma-1}{3\gamma})\min\{\frac{1-\gamma}{3\gamma}, \beta/2\}} \Big)\\
    &=\exp \Big(-128 c_1 \log(1/4\delta)\frac{256}{27\gamma(4\gamma-1)}\cdot \frac{1-\gamma}{\min\{\frac{1-\gamma}{3\gamma}, \beta/2\}}\Big)\\
    &\geq \exp \Big(-128 c_1 \log(1/4\delta)\frac{256\cdot 20}{27\gamma(4\gamma-1)}\Big)\\
    &\geq \exp \Big(-128 c_1 \log(1/4\delta)\frac{5120}{27*0.4*0.6} \Big)\geq\left(4\delta\right)^{102000c_1},
\end{align}
where the third line follows $1-p_0^k\geq \min\{1-\frac{4\gamma-1}{3\gamma}, 1-p_0(x_k,a_1)+\beta/2\}\geq \min\{\frac{1-\gamma}{3\gamma}, \beta/2\}$ and the fifth line is due to $\frac{1-\gamma}{\min\{\frac{1-\gamma}{3\gamma}, \beta/2\}}\leq\max\{3\gamma, 20\}=20$ (since $\gamma> 1-2\beta$).
Further, when $\cE_{k,l}$ occurs, $A_{k,l}$ and $C_{k,l}$ both occur. Therefore, following similar steps, we have
\begin{align}
\left(1-\frac{\alpha_2^k}{1-p_0^k}\right)^{T_{k,l}-S_{k,l}/{p_0^k}} & \geq\left(1-\frac{\alpha_2^k}{1-p_0^k}\right)^{\max_{1\leq T_{k,l}\leq 4t^*} |T_{k,l}-S_{k,l}/p_0^k|}\geq\left(1-\frac{\alpha_2^k}{1-p_0^k}\right)^{\sqrt{16t^*\frac{1-p_0^k}{p_0^k}\log(1/4\delta)}} \\
& \geq \exp \left(-\sqrt{64\frac{(\alpha_2^k)^2}{p_0^k(1-p_0^k)}t^*\log(1/4\delta)}\right)
\geq\left(4\delta \right)^{\sqrt{810000c_1}}.
\end{align}
In total, we have $L_{k,l}(W)/L_{1}(W)\geq(4\delta)^{108100c_1+\sqrt{810000c_1}}$.
By taking $c_1$ small enough, e.g. $c_1=5e^{-7}$, we have $L_{k,l}(W)/L_{1}(W) > 4\delta.$
By a change of measure,
\begin{equation}\label{eq:p2}
\mathbb{P}_{k,l}(B_{k,l})\geq \mathbb{P}_{k,l}(\cE_{k,l})=\mathbb{E}_{k,l}[\mathbf{1}_{\cE_{k,l}}]=\mathbb{E}_1\left[\frac{L_{k,l}(W)}{L_1(W)}\mathbf{1}_{\cE_{k,l}}\right]> 4\delta*1/4=\delta.
\end{equation}
\end{proof}
\begin{proof}[Proof of the Lower Bound in Theorem \ref{thm:main}]
Since $\sA$ is $(\cM^0,\beta,\varepsilon, \delta)$-correct, under hypothesis $\cM_{k,l}$, $\sA$ should produce a policy $\pi$ such that $\pi(x_k)=a_l$ with probability $\geq 1-\delta$. Thus, we should have $\mathbb{P}_{k,l}(B_{k,l})<\delta$ for all $k\in[K], 1<l\leq L_k$. From Lemma~\ref{lemma:newversion}, it requires $\mathbb{E}_1[T_{k,l}]>t^*$ for all $k\in[K], 1<l\leq L_k$. In total, we need $\Omega\left(\frac{\sum_{k\in[K]}L_k}{(1-\gamma)^3\varepsilon^2}\log(1/\delta)\right)$ samples. By definition of $L_k$, for every $k\in[K]$, we have
\begin{align}\label{eq:L_k}
\{a_l, l\leq L_k\} &= \cA^{x_k}_{\cM_0}\Big(\frac{1}{1-\gamma p_0(x_k,a_1)}-\frac{1}{1-\gamma (p_0^k+\alpha_2^k-\beta/2)}\Big)\\
&=\cA^{x_k}_{\cM_0}\Big(V^*(x_k) - \frac{1}{1-\gamma p_0^k - 4(1-\gamma p_0^k)^2\varepsilon+\beta\gamma/2}\Big).
\end{align}
If $p_0^k=\frac{4\gamma-1}{3\gamma}$, i.e. $p_0(x_k,a_1)-\beta/2\leq\frac{4\gamma-1}{3\gamma}$, then
\begin{align}\label{eq:Lk1}
    L_k &= \Big\vert\cA^{x_k}_{\cM_0}\Big(V^*(x_k) - \frac{9}{12(1-\gamma)-64(1-\gamma)^2\varepsilon+4.5\beta\gamma}\Big) \Big\vert.
    \end{align}
If $p_0^k=p_0(x_k,a_1)-\beta/2$, i.e. $p_0(x_k,a_1)-\beta/2 > \frac{4\gamma-1}{3\gamma}$, then
 \begin{align}\label{eq:Lk2}
    L_k &= \Big\vert\cA^{x_k}_{\cM_0}\Big(V^*(x_k) - \frac{1}{1-\gamma p_0(x_k,a_1)+\gamma\beta - 4\varepsilon(1-\gamma p_0(x_k,a_1)+\gamma\beta/2)^2}\Big)\Big\vert\\
    &=\Big\vert\cA^{x_k}_{\cM_0}\Big(V^*(x_k) - \frac{V^*(x_k)^2}{V^*(x_k)+\gamma\beta (V^*(x_k))^2 - 4\varepsilon(1+\gamma\beta V^*(x_k)/2)^2}\Big)\Big\vert
    \end{align}
  Combining Equation \eqref{eq:Lk1} and \eqref{eq:Lk2}, we have
$$\Omega\left(\frac{\sum_{k\in[K]}L_k}{(1-\gamma)^3\varepsilon^2}\log(1/\delta)\right)=\Omega\left(\frac{\sum_{s\in\cS'}|\cA^{s}_{\cM_0}(\underline{C}_s)|}{(1-\gamma)^3\varepsilon^2}\log(1/\delta)\right),$$
which concludes our proof of the lower bound result in Theorem \ref{thm:main}.
\end{proof}
In Corollary \ref{coro:worst}, we have the worst scenario, where $L_k=|\cA|$. This is possible if there is no gap between the values in $\cM_0$. Then, for any $c>0$, $|\cA^{x_k}_{\cM_0}(c)|=|\cA|$.
\cut{The proof of Theorem~\ref{thm:pilowerbound} is established in  three steps:
\begin{enumerate}
\item We will show that for a fixed $(x,a)\in\cX\times\cA$,
$Q^*_{\cM^i}(x,a)$ for $i\in\{1,2\}$ cannot
be learnt up to $\epsilon$ error with high probability using only a small number of samples;
    \item We will generalize the above lower bound  to a uniform bound on all state-action pairs using the fact that their transitions are independent of each other;
    \item Lastly, we will extend the lower bound to learning an $\varepsilon$-optimal policy.
\end{enumerate}
The proof will follow Lemma~\ref{lemma:xabound}, Lemma~\ref{lemma:Qbound}, Proposition~\ref{prop:lowerboundQ}, and Lemma~\ref{lemma:policyeva}. For simplicity, we use $p_1(x,a)$ and $p_2(x,a)$ to denote $p_{\cM^1}(x,a)$ and $p_{\cM^2}(x,a)$, respectively.}
\cut{
\subsection*{Step 1: a fixed $(x,a) \in \cX\times\cA$}
\fei{Is the current definition fine?}

Given a fixed $(x,a)$ and a generative model, we can generate $t$ samples with the same input state $y_1(x,a)$. The rewards of these samples are equivalent to a length-$t$ sequence of i.i.d Bernoulli random variables $W_t:=\{X_1, X_2, \dots X_t\}$. We define two probability spaces $(\Omega, \cF, P^1)$ and $(\Omega, \cF, P^2)$ for $W_t$, where in the former, $\mathbb{P}(X_1=1)=p_1(x,a)$; in the latter, $\mathbb{P}(X_1=1)=p_2(x,a)$. Note that these two spaces have the same sample space and $\sigma-$algebra. The only difference is the probability measure. In the following context, we use the notation $\mathbb{P}^1(\cdot)$ if the event is measured by $P^1$ and vice versa for $\mathbb{P}^2(\cdot)$.

Suppose $\sA$ is an $(\cM_0, \beta, \varepsilon, \delta, Q)$-correct algorithm Denote by $w$ a realization of $W_t$ and $Q^{\sA}_{t}(x,a)$ the output Q-value of $(x,a)$ from $\sA$ with input $w$. We define events:
\begin{align}
\cE_1(x,a)&:=\{|Q^*_{\cM^1}(x,a)-Q^{\sA}_{t}(x,a)|\leq\varepsilon,\quad \cE_2(x,a):=\{|Q^*_{\cM^2}(x,a)-Q^{\sA}_{t}(x,a)|\leq\varepsilon\}
\end{align}
Denoting by $k$ the number of 1s in $w$, we define another event:
$$
\cE_3(x,a):= \bigg\{p_{1}(x,a)\cdot t-k\leq\sqrt{2p_{1}(x,a)\cdot (1-p_{1}(x,a))\cdot t\cdot \log\Big(\frac{c}{2\theta(x,a)}\Big)}\bigg\},$$
where $\theta(x,a):=\exp(-c\alpha(x,a)^2t/(p_{1}(x,a)(1-p_{1}(x,a))))$. Furthermore, let $\cE(x,a):=\cE_1(x,a)\cap\cE_3(x,a)$.

Similar to the proof in \citealt[Lemma 16]{azar2013minimax}, by Chernoff-Hoeffding bound and $p_1(x,a)>1/2$, we first have that
\begin{align}\label{eq:event3}
\mathbb{P}^1(\cE_3)>1-\frac{2\theta(x,a)}{c}.
\end{align}

\cut{
\begin{proof}
Denote by $\epsilon:=\sqrt{2p_{1}(x,a)(1-p_{1}(x,a))t\log(\frac{c}{2\theta(x,a)})}$. By Chernoff-Hoeffding bound, since $p_1(x,a)>1/2$, we have that
\begin{align}
\mathbb{P}^1(\cE_2)&\geq 1-\exp[-\text{KL}(p_{1}(x,a)+\frac{\epsilon}{t}||p_{1}(x,a))\cdot t]\\
&\geq 1-\exp[-\frac{\epsilon^2}{2p_{1}(x,a)(1-p_{1}(x,a))t}]\\
&\geq 1-\frac{2\theta(x,a)}{c}.
\end{align}
\end{proof}}

Next, we show that for the fixed $(x,a)$, there is certain probability such that the Q value approximated by $\sA$ cannot be very accurate on both models.
\begin{lemma}\label{lemma:xabound}
When $\varepsilon\in\bigg(0,\frac{(1-p_{0}-\beta'/4)\gamma^2}{8(1-\gamma p_{0})^2}\bigg)$, there exists a constant $c>0$ such that for every $(\cM_0, \beta, \varepsilon, \delta, Q)$-correct RL algorithm $\sA$, it holds that
$$\big(\max_{m\in\{1,2\}} \mathbb{P}^m (\neg~\cE_m(x,a))\big)>\frac{\theta(x,a)}{c},$$
by the choice of $\alpha(x,a)=\frac{2(1-\gamma p_{0})^2\varepsilon}{\gamma^2(p_1(x,a))}$.
\end{lemma}
\begin{proof}
We prove by contradiction. Suppose there is an algorithm $\sA$ such that
\begin{align}
\mathbb{P}^m(|Q^*_{\cM^m}(x,a)-Q^{\sA}_{t}(x,a)|>\varepsilon)\leq\frac{\theta(x,a)}{c}, \quad \text{for all }m\in\{1,2\}.
\end{align}
Then we have that $\mathbb{P}^1(\cE_1)\geq 1-\theta(x,a)/c$. Since $\theta(x,a)<1$, combining with Equation \eqref{eq:event3} and choosing $c>6$, we have $\mathbb{P}(\cE)>1/2$. Denote by $w$ a realization of $W_t$ and $L^m(w):=\mathbb{P}^m(W_t=w), m\in\{1,2\}$. Recall that $k$ is the number of 1s in $w$. Then, the likelihood ratio between model $\cM^1$ and model $\cM^2$ is:
\begin{align}\label{eq:ratio}
\frac{L^2(w)}{L^1(w)}&=
\frac{p_{2}(x,a)^{k}(1-p_{2}(x,a))^{t-k}}{p_{1}(x,a)^{k}(1-p_{1}(x,a))^{t-k}}=\big(1+\frac{\alpha(x,a)}{p_{1}(x,a)}\big)^{k}\big(1-\frac{\alpha(x,a)}{1-p_{1}(x,a)}\big)^{t-k}\\
&=\big(1-\frac{\alpha(x,a)}{1-p_{1}(x,a)}\big)^{\frac{1-p_{1}(x,a)}{p_{1}(x,a)}k}\big(1-\frac{\alpha(x,a)}{1-p_{1}(x,a)}\big)^{t-\frac{k}{p_{1}(x,a)}}\big(1+\frac{\alpha(x,a)}{p_{1}(x,a)}\big)^{k}
\end{align}
By our choice of $\varepsilon$ and $\alpha(x,a)$, we have that $\alpha(x,a) = \frac{1-p_0-\beta'/4}{4(p_0+b(x,a))}\leq \frac{1-p_0-\beta'/4}{2}\leq (1-p_1(x,a))/2$. With the fact that $\log (1-u) \geq-u-u^{2}$ for $u\in[0,1/2]$ and $\exp (-u) \geq 1-u$ for $u\in[0,1]$, we further have that
\begin{align}
\bigg(1-\frac{\alpha(x,a)}{1-p_{1}(x,a)}\bigg)^{\frac{1-p_{1}(x,a)}{p_{1}(x,a)}}& \geq \exp \left(\frac{1-p_{1}(x,a)}{p_{1}(x,a)}\left(-\frac{\alpha(x,a)}{1-p_{1}(x,a)}-\big(\frac{\alpha(x,a)}{1-p_{1}(x,a)}\big)^{2}\right)\right) \\ & \geq\left(1-\frac{\alpha(x,a)}{p_{1}(x,a)}\right)\left(1-\frac{\alpha(x,a)^{2}}{p_{1}(x,a)(1-p_{1}(x,a))}\right).
\end{align}
Thus
\begin{align} \frac{L_{2}(w)}{L_{1}(w)} & \geq\left(1-\frac{\alpha(x,a)^2}{p_{1}(x,a)^2}\right)^{k}\left(1-\frac{\alpha(x,a)^{2}}{p_{1}(x,a)(1-p_{1}(x,a))}\right)^{k}\bigg(1-\frac{\alpha(x,a)}{1-p_{1}(x,a)}\bigg)^{t-\frac{k}{p_{1}(x,a)}} \\ & \geq\left(1-\frac{\alpha(x,a)^2}{p_{1}(x,a)^2}\right)^{t}\left(1-\frac{\alpha(x,a)^{2}}{p_{1}(x,a)(1-p_{1}(x,a))}\right)^{t}\bigg(1-\frac{\alpha(x,a)}{1-p_{1}(x,a)}\bigg)^{t-\frac{k}{p_{1}(x,a)}}\end{align}
since $k\leq t$. Using $\log(1-u)\geq -2u$ for $u\in[0,1/2]$, we have that
$$\left(1-\frac{\alpha(x,a)^{2}}{p_{1}(x,a)(1-p_{1}(x,a))}\right)^{t} \geq \exp \left(-2 t \frac{\alpha(x,a)^{2}}{p_{1}(x,a)(1-p_{1}(x,a))}\right) \geq\left(2 \theta(x,a) / c\right)^{2 /c_1},$$
where $c_1\leq \frac{c}{\log\theta(x,a)}\log\frac{2\theta(x,a)}{c}$. Further, we have that
$$\left(1-\frac{\alpha(x,a)^2}{p_{1}(x,a)^2}\right)^{t} \geq \exp \left(-2t \frac{\alpha(x,a)^2}{p_{1}(x,a)^2}\right) \geq\left(2 \theta(x,a) / c\right)^{\frac{2(1-p_{1}(x,a))}{c_1p_{1}(x,a)}}.$$
On $\cE_3$, it holds that $(t-\frac{k}{p_1(x,a)}) \leq \sqrt{2\frac{1-p_{1}(x,a)}{p_{1}(x,a)}t\log\frac{c}{2\theta(x,a)}}$, thus
\begin{align}
\left(1-\frac{\alpha(x,a)}{1-p_{1}(x,a)}\right)^{t-\frac{k}{p_{1}(x,a)}} & \geq\left(1-\frac{\alpha(x,a)}{1-p_{1}(x,a)}\right)^{\sqrt{2\frac{1-p_{1}(x,a)}{p_{1}(x,a)}t\log\frac{c}{2\theta(x,a)}}} \\
& \geq \exp \left(-\sqrt{8\frac{\alpha(x,a)^2}{p_{1}(x,a)(1-p_{1}(x,a))}t\log\frac{c}{2\theta(x,a)}}\right) \\
& \geq \exp \left(-\sqrt{\frac{-8\log\theta(x,a)}{c}\log\frac{c}{2\theta(x,a)}}\right)\\
& \geq\left(2 \theta(x,a) / c\right)^{\sqrt{8/c_1}}
\end{align}
By taking $c_1=16$ and $c=17$,
$$\frac{L_{2}(W)}{L_{1}(W)} \geq\left(2 \theta(x,a) / c\right)^{2 / c_{1}+2(1-p_1(x,a)) /(c_{1}p_1(x,a))+\sqrt{8/c_1}}\geq 2 \theta(x,a) / c.$$
By a change of measure, we deduce that
\begin{equation}\label{eq:p2}
\mathbb{P}^2(\cE_1)\geq \mathbb{P}^2(\cE)=\mathbb{E}_2[\mathbf{1}_{\cE}]=\mathbb{E}_1(\frac{L_2(W)}{L_1(W)}\mathbf{1}_{\cE})\geq\frac{\theta(x,a)}{c}.
\end{equation}
For this choice of $\alpha(x,a)$,
\begin{align}
Q^*_{\cM^2}-Q^*_{\cM^1} &=\frac{\gamma}{1-\gamma p_{2}(x,a)}-\frac{\gamma}{1-\gamma p_{1}(x,a)}\\
&= \frac{\alpha(x,a)\gamma^2}{(1-\gamma p_2(x,a))(1-\gamma p_1(x,a))}\\
&>\frac{\alpha(x,a)\gamma^2}{(1-\gamma p_0)^2} = \frac{2\varepsilon}{p_0+b(x,a)}\geq 2\varepsilon.
\end{align}

Thus, event $\{|Q^*_{\cM^2}-Q^{\sA}_t|\leq\varepsilon\}$ does not overlap with the event $\{|Q^*_{\cM^1}-Q^{\sA}_t|\leq\varepsilon\}$. By \eqref{eq:p2}, $\mathbb{P}^2(\{|Q^*_{\cM^1}-Q^{\sA}_t|\leq\varepsilon\})\leq 1-\frac{\theta(x,a)}{c}$. Thus, $\mathbb{P}^2(\{|Q^*_{\cM^2}-Q^{\sA}_t|\geq\varepsilon\})\geq \frac{\theta(x,a)}{c}$, which contradicts with our assumption.
\end{proof}

Based on Lemma \ref{lemma:xabound}, for $\beta>0$ and $0.4\leq\gamma<1$, selecting $p_{0}=\frac{4\gamma-1}{3\gamma}$ and $\varepsilon\in\Big(0,\frac{1-p_{0}-\beta'/4}{4(1-\gamma p_{0})^2}\gamma^2\Big)$, there exists an $\cM\in B_{\mathrm{TV}}(\cM_0,\beta)$ such that for any $(x,a)\in\cX\times\cA$ and any $(\cM_0, \beta, \varepsilon,\delta,Q)$-correct algorithm $\sA$,
\begin{align}
\mathbb{P}^{\cM}(|Q^*_{\cM}(x,a)-Q^{\sA}_t(x,a)|>\varepsilon)&>\frac{1}{c}\exp(-c\alpha^2t/(p_1(x,a)(1-p_1(x,a)))).\\
&=\frac{1}{c}\exp(\frac{-4c(1-\gamma p_0)^4\varepsilon^2t}{\gamma^4p_1(x,a)(1-p_1(x,a))})>\exp(-34000(1-\gamma)^3\varepsilon^2t).
\end{align}
This  implies that for any state-action pair $(x,a)\in\cX\times\cA$ and $\delta>0$,
\begin{equation}\label{eq:delta}
\mathbb{P}^{\cM}(|Q^*_{\cM}(x,a)-Q^{\cA}_t(x,a)|>\varepsilon)>\delta
\end{equation}
whenever the number of transition samples $t$ is less than $\xi(\varepsilon, \delta):=\frac{1}{4\varepsilon^2(1-\gamma)^3}\log(\frac{1}{\delta})$, where $4=34000$.

\subsection*{Step 2: generalize to all state-action pairs}
Next, we prove a sample complexity lower bound on $\|Q^*_{\cM}-Q^{\sA}_T\|_{\infty}\leq\varepsilon$, where $T$ is the total number of samples for all $N:=3KL$ state-action pairs.

\begin{lemma}\label{lemma:Qbound}
For any $\delta'\in(0,1/2)$, $\varepsilon\in\Big(0,\frac{1-p_{0}-\beta'/4}{4(1-\gamma p_{0})^2}\gamma^2\Big)$, and any $(\cM_0, \beta, \varepsilon, \delta, Q)$-correct algorithm $\sA$, it holds that when $T \leq\frac{N}{6}\xi(\varepsilon, \frac{12\delta'}{N})$,
$$\big(\max_{m\in\{1,2\}}\mathbb{P}^{m}(\|Q^*_{\cM^m}-Q^{\sA}_T\|_{\infty}>\varepsilon)\big)>\delta'.$$
\begin{proof}
Denote by $\delta:=\frac{12\delta'}{N}$. When $T\leq\frac{N}{6}\xi(\varepsilon, \delta)$, at most $\lceil N/6 \rceil$ state-action pairs have samples more than $\xi(\varepsilon, \delta)$. Since $|\cY_1|=N/3$, at least $\lfloor N/6 \rfloor$ states in $\cY_1$ are sampled less than $\xi(\varepsilon, \delta)$ times. We label these states as $y_1(x_i,a_i), i\in\{1,2,\dots,\lfloor N/6 \rfloor\}$. Define $\cQ^m(x,a):=\{|Q^*_{\cM^m}(x,a)-Q^{\sA}_{T(x,a)}(x,a)|>\varepsilon\}$, where $T(x,a)$ is the number of samples for $y_1(x,a)$ within all $T$ samples. Since transitions of different $y_1$s are independent, the event $\cQ^m(x,a)$ and $\cQ^m(x',a')$ are conditionally independent given $T(x,a)$ and $T(x',a')$. Following Lemma~\ref{lemma:xabound}, we have that
\begin{align}
    &\mathbb{P}^m(\{\cQ^m((x_i,a_i))^c)\}_{1\leq i\leq\lfloor N/6 \rfloor} \cap \{T(x_i,a_i)\leq\xi(\varepsilon,\delta)\}_{1\leq i\leq \lfloor N/6 \rfloor})\\
    =& \sum_{t_1=0}^{\xi(\varepsilon, \delta)}\cdots\sum_{t_{\lfloor N/6 \rfloor}=0}^{\xi(\varepsilon,\delta)}\mathbb{P}^m(\{T(x_i,a_i)=t_i\}_{1\leq i\leq \lfloor N/6 \rfloor})\mathbb{P}^m(\{\cQ^m(x_i,a_i)^c\}_{1\leq i\leq \lfloor N/6 \rfloor}\vert \{T(x_i,a_i)=t_i\}_{1\leq i\leq \lfloor N/6 \rfloor})\\
    =&\sum_{t_1=0}^{\xi(\varepsilon, \delta)}\cdots\sum_{t_{\lfloor N/6 \rfloor}=0}^{\xi(\varepsilon,\delta)}\mathbb{P}^m(\{T(x_i,a_i)=t_i\}_{1\leq i\leq \lfloor N/6 \rfloor})\prod_{1\leq i\leq N/6}\mathbb{P}^m(\{\cQ^m(x_i,a_i)^c\}\vert \{T(x_i,a_i)=t_i\})\\
    \leq &\sum_{t_1=0}^{\xi(\varepsilon, \delta)}\cdots\sum_{t_{\lfloor N/6 \rfloor}=0}^{\xi(\varepsilon,\delta)}\mathbb{P}^m(\{T(x_i,a_i)=t_i\}_{1\leq i\leq \lfloor N/6 \rfloor})(1-\delta)^{\lfloor N/6 \rfloor},
\end{align}
where the last inequality follows Eq. \eqref{eq:delta}. Thus,
\begin{equation}
    \mathbb{P}^m(\{\cQ^m((x_i,a_i))^c)\}_{1\leq i\leq N/6} \vert \{T(x_i,a_i)\leq\xi(\epsilon,\delta)\}_{1\leq i\leq N/6})\leq (1-\delta)^{N/6}.
\end{equation}
When $T\leq(N/6)\xi(\varepsilon, \delta)$,
\begin{align}
    \mathbb{P}^m(\|Q^*_{\cM^m}-Q^{\cA}_T\|>\varepsilon)&\geq \mathbb{P}^m\big(\bigcup_{(x,a)\in\cX\times\cA} \cQ^m(x,a)\big)\\
    &\geq 1- \mathbb{P}^m(\{\cQ^m((x_i,a_i))^c)\}_{1\leq i\leq N/6} \vert \{T_{x_i,a_i}\leq\xi(\epsilon,\delta)\}_{1\leq i\leq N/6})\\
    &\geq 1-(1-\delta)^{N/6}\geq \frac{\delta N}{12}=
    \delta',
\end{align}
whenever $\delta N/6\leq 1$.
\end{proof}
\end{lemma}

Lemma~\ref{lemma:Qbound} implies that in order to have $\mathbb{P}^{\cM}(\|Q^*_{\cM}-Q^{\sA}_T\|_{\infty}>\varepsilon)<\delta'$ for every $\cM\in B_{\mathrm{TV}}(\cM_0, \beta)$, one need samples more than $C\frac{N}{\varepsilon^2(1-\gamma)^3}\log(N/{\delta'})$. Thus, a lower bound of learning a $Q$-function with a TV-distance-close prior model is achieved. We state the result formally in Proposition. \ref{prop:lowerboundQ}.
\begin{proposition}\label{prop:lowerboundQ}
For any $\beta>0$, there exists an MDP $\cM_0$, constants $\varepsilon_0$ and $\delta_0$, such that for all $\varepsilon\in(0,\varepsilon_0), \delta\in(0,\delta_0),$ and every $(\cM_0, \beta, \varepsilon, \delta, Q)$-correct RL algorithm $\sA$, the total number of samples needed to learn an \textbf{$\varepsilon$-optimal $Q$-function} for every $\cM\in B_{\mathrm{TV}}(\cM_0, \beta)$ with probability at least $1-\delta$ is
$$\Omega\bigg(\frac{N}{\varepsilon^2(1-\gamma)^3}\log(\frac{N}{\delta})\bigg),$$
where $N$ is the total number of state-action pairs.
\end{proposition}

\subsection*{Step 3: extend to learning policy}
In the last section,  we have  established the sample complexity lower bound for learning a sub-optimal $Q$-function with a TV-distance-close prior model. Next, we extend the result to learning an $\varepsilon$-optimal policy. To this end, we still consider the models $\cM_0, \cM^1$ and $\cM^2$ as before. We first show that given a policy $\pi$ for $\cM\in\mathbb{M}$, it takes much fewer samples to evaluate its $Q$-function compared with the result in Proposition~\ref{prop:lowerboundQ}. Therefore, given $\cM_0$, if an algorithm can produce $\varepsilon$-optimal policies for both $\cM^1$ and $\cM^2$ with significantly fewer samples, then it can produce $\varepsilon$-optimal $Q$-functions with fewer samples as well, which contradicts with Lemma~\ref{lemma:Qbound}. Thus, the lower bound is extended to learning policies.

\begin{lemma}\label{lemma:policyeva}
Given a policy $\pi$ for an MDP $\cM\in\mathbb{M}$ and a generative model of $\cM$, for $\varepsilon\in(0,1]$ and $\delta\in(0,1)$, there exists an algorithm that uses
$$\cO\bigg( \frac{N}{\varepsilon^2(1-\gamma)^3}\log(\frac{1}{\varepsilon(1-\gamma)})\log(\frac{N}{\delta})+ \frac{N}{\varepsilon^2(1-\gamma)^2}\log(\frac{N}{\delta})\bigg)$$
samples to obtain a function $Q:\cS\times\cA\rightarrow \RR$,
such that, with probability at least $1-\delta$,
$\|Q-Q^{\pi}\|_{\infty}\leq \varepsilon$.
\begin{proof}
Since all states in $\cY_2$ have values 0 and $V^{\pi}(y_1(x,a))=Q^{\pi}(x,a)/\gamma$, we only need to evaluate $Q^{\pi}(x,a)$ for all $(x,a)\in\cX\times\cA$. 
Given $\varepsilon>0$, when $T= \lceil \frac{1}{1-\gamma}\log\frac{2}{(1-\gamma)\varepsilon}\rceil$, the tail sum $\mathbb{E}^{\pi}[\sum_{t=T}^{\infty}\gamma^t R(s_t,a_t)|s_0=x, a_0=a]\leq \varepsilon/2$. Thus, we only need to evaluate $Q^{\pi}_T(x):=\mathbb{E}^{\pi}[\sum_{t=0}^{T}\gamma^t R(s_t,a_t)|s_0=x, a_0=a]$ to $\varepsilon/2-$accuracy.  We define a random variable $G^{\pi}_T(x,a):=R(x,a)+\sum_{t=1}^{T}\gamma^t R(s_t,a_t)$ following $\pi$. Then $G^{\pi}_T(x,a)\in[0, 1/(1-\gamma)]$ for all $(x,a)\in\cX\times\cA$ and $\mathbb{E}[G^{\pi}_T(x,a)]=Q^{\pi}_T(x,a)$. By Hoeffding Inequality, we only need to take $\frac{2}{(1-\gamma)^2\varepsilon^2}\log(\frac{|\cS|}{\delta})$ samples of $G^{\pi}_T(x)$ such that the empirical average $\bar{V}^{\pi}_T(x)$ satisfies
$$\mathbb{P}(|\bar{V}^{\pi}_T(x)-V^{\pi}_T(x)|\leq \varepsilon/2)\geq 1-\frac{\delta}{K}.$$ In total, there are $\lceil \frac{2}{(1-\gamma)^3\varepsilon^2}\log(\frac{K}{\delta})\log\frac{2}{(1-\gamma)\varepsilon}\rceil$ transition samples. Repeating the step to all $x\in\cX$ and taking the union bound, we can get a function $\bar{V}^{\pi}_T:\cX\rightarrow \RR$ such that, with probability at least $1-\delta$,
$$\|\bar{V}^{\pi}_T-V^{\pi}\|_{\infty}\leq \varepsilon.$$

Next, we use $\bar{V}_T^{\pi}$ to approximate $Q^{\pi}(x,a)$. Recall that
$$Q^{\pi}(x,a) = R(x,a) + \gamma P(\cdot|x,a)^TV^{\pi}.$$
Define a random variable $v(s,a)$ where $v(s,a)=V^{\pi}(s')$ with probability $p(s'|s,a)$. Then $P(\cdot|s,a)^TV^{\pi}$ is equal to the expectation of $v(s,a)$. By Hoeffding Inequality, with $n:=\frac{1}{2(1-\gamma)^2\varepsilon^2}\log(\frac{|\cS||\cA|}{\delta})$ transition samples from $(s,a)$, the empirical average $\bar{V}(s,a)$ satisfies
$$\mathbb{P}(|\bar{V}(s,a)-P(\cdot|s,a)^TV^{\pi}|\leq \varepsilon)\geq 1-\frac{\delta}{|\cS||\cA|}.$$
Repeating the same step for all $(s,a)\in\cS\times\cA$ and taking the union bound, we can get an $\varepsilon$-correct approximation of $Q^{\pi}$ from $V^{\pi}$ with probability at least $1-\delta$ with the sample complexity $\cO\big( \frac{|\cS||\cA|}{(1-\gamma)^2\varepsilon^2}\log(\frac{|\cS||\cA|}{\delta})\big)$. Adding these two complexities together, we can get the desired result.
\end{proof}
\end{lemma}

Combining Proposition~ \ref{prop:lowerboundQ} and Lemma \ref{lemma:policyeva}, we obtain the final result as stated in the following Theorem. 
\begin{theorem}\label{thm:lowerboundpi}
For any $\beta>0$, there exists some constants $\varepsilon_0, \delta_0, c_1, 4$, a class of MDPs $\mathbb{M}$, and an MDP $\cM_0\in\mathbb{M}$ such that for all $\varepsilon\in(0,\varepsilon_0), \delta\in(0,\delta_0),$ and every $(\beta, \varepsilon, \delta, \pi)$-correct RL algorithm $\sA$ which has full knowledge of $\cM_0$, the total number of samples needed to learn an \textbf{$\varepsilon$-optimal policy} of any MDP $\cM\in \mathbb{M}\cap B_{\mathrm{TV}}(\cM_0, \beta)$ with probability at least $1-\delta$ is
$$\Omega\big(\frac{|\cS||\cA|}{\varepsilon^2(1-\gamma)^3}\log\frac{|\cS||\cA|}{\delta}\big).$$

\begin{proof}
Given $\varepsilon$, one can take $\cM_0$ with $|\cA| \approx \frac{1}{(1-\gamma)\varepsilon}$. Then the complexity of policy evaluation is significantly smaller than learning $Q$-values. Thus, if an algorithm can find an $\varepsilon$-optimal policy of $\cM$ with significantly fewer samples than the lower bound, it can then output an $2\varepsilon$-optimal $Q$ function with samples significantly fewer than the lower bound, which contradicts with Prop. \ref{prop:lowerboundQ}.
\end{proof}
\end{theorem}
}

\section{Further Discussion}
\subsection{Parameter Relations}
In the proof of the lower bound result, given $\beta\in(0,2)$, we first restrict $\gamma\in(\max\{0.4,1-10\beta\},1)$, then we choose $p_0^k$ based on whether $p_0(x_k,a_1)\leq \beta/2+\frac{4\gamma-1}{3\gamma}$ or not. These restrictions are key for the validation of the lower bound. In Figure \ref{fig:para}, we depict the relation between these quantities. One can observe that, when $\beta<1$, with a small $\gamma$ (below the blue line), we are able to select a $p_0(x_k,a_1)$ such that $p_0^k=p_0(x_k,a_1)-\beta/2$; otherwise, $p_0^k$ will always be $\frac{4\gamma-1}{3\gamma}$.
\begin{figure}[htbp!]
        \centering
        \vspace{-15pt}
        \includegraphics[width=\textwidth]{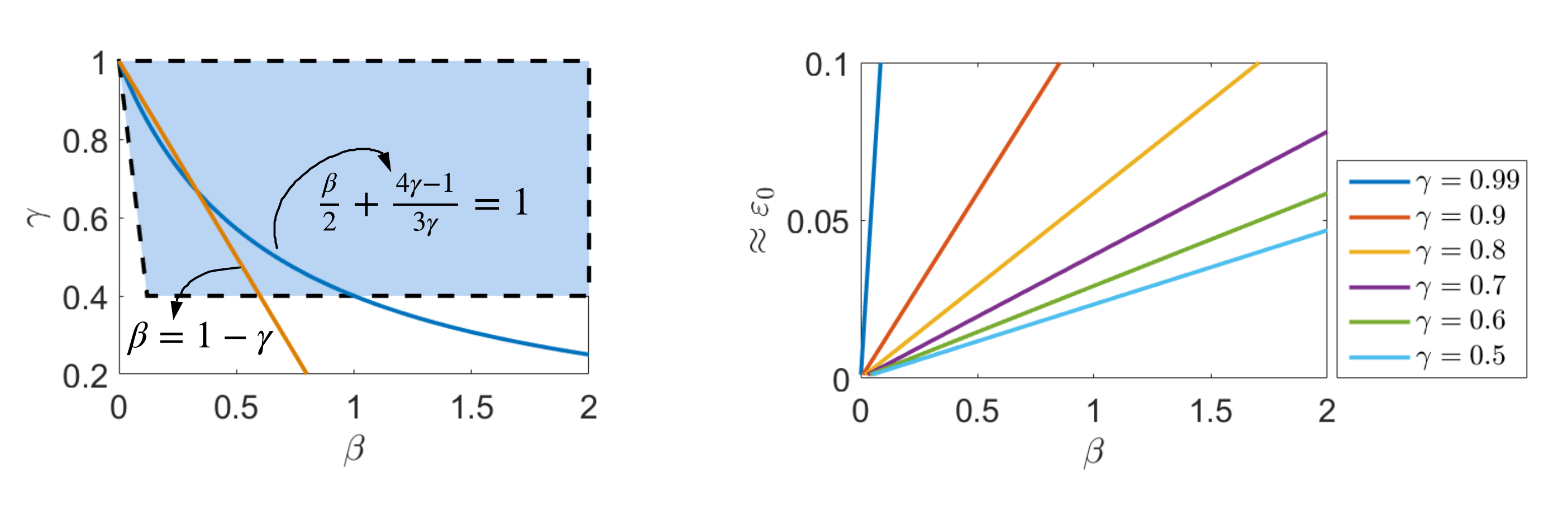}
        \vspace{-15pt}
        \caption{The blue area represents the feasible set: $\gamma\in(\max\{0.4, 1-10\beta\}, 1)$. Below the blue line, we can select $p_0(x_k,a_1)$ such that $p_0(x_k,a_1)>\beta_2+\frac{4\gamma-1}{3\gamma}$ and thus have $p_0^k=p_0(x_k,a_1-\beta/2)$; above the blue line, $p_0(x_k,a_1)$ will always be smaller than $\beta_2+\frac{4\gamma-1}{3\gamma}$ and $p_0^k$ will always be $\frac{4\gamma-1}{3\gamma}$. Below the brown line, we have $\bar{C} = 2\beta/(1-\gamma)^2$, which can produce a non-trivial upper bound.}
        \label{fig:para}
    \end{figure}

\subsection{When Upper Bound Meets Lower Bound}
In general, $\bar{C}>\underline{C}_s$. However, if the value gap in $\cM_0$ is big enough, we might still have
\begin{align}\label{eq:minimax}
    \cA^{s}_{\cM_0}(\bar{C})= \cA^{s}_{\cM_0}(\underline{C}_s), \text{~for every~} s\in\cS'.
\end{align}
If $\bar{C}=1/(1-\gamma)$, then $\cA^s_{\cM_0}=\cA$. If the worst case happens as in Corollary \ref{coro:worst} (value gap super close to 0), we reach minimax optimal. If $\bar{C}=\beta/(1-\gamma)^2$ (this can happen in the area below the brown line in Figure \ref{fig:para}), by the definitions about value gap in Section \ref{sec:pre}, Equation \eqref{eq:minimax} implies
\begin{align}
    q^*_s\big[|\cA^s_{\cM_0}(\bar{C})|\big]-q^*_s\big[|\cA^s_{\cM_0}(\bar{C})|+1\big] \geq \bar{C}-\underline{C}_s,\text{~for every~} s\in\cS',
\end{align}
i.e., the value gap between the $|\cA^s_{\cM_0}(\bar{C})|$-th optimal action and the $\big(|\cA^s_{\cM_0}(\bar{C})|+1\big)$-th optimal action should be at least $\bar{C}-\underline{C}_s$. We illustrate this situation in Figure \ref{fig:gap}. We depict two types of prior models, where the right one has a larger value gap than the left one. For the right model, the actions we should at least consider (above the yellow line) are equal to the actions we should at most consider (above the blue line). Thus, for the right model, the sample complexity to learn for a $\beta$-close model is minimax optimal.
\begin{figure}[htbp!]
    \centering
    \includegraphics[width=.9\textwidth]{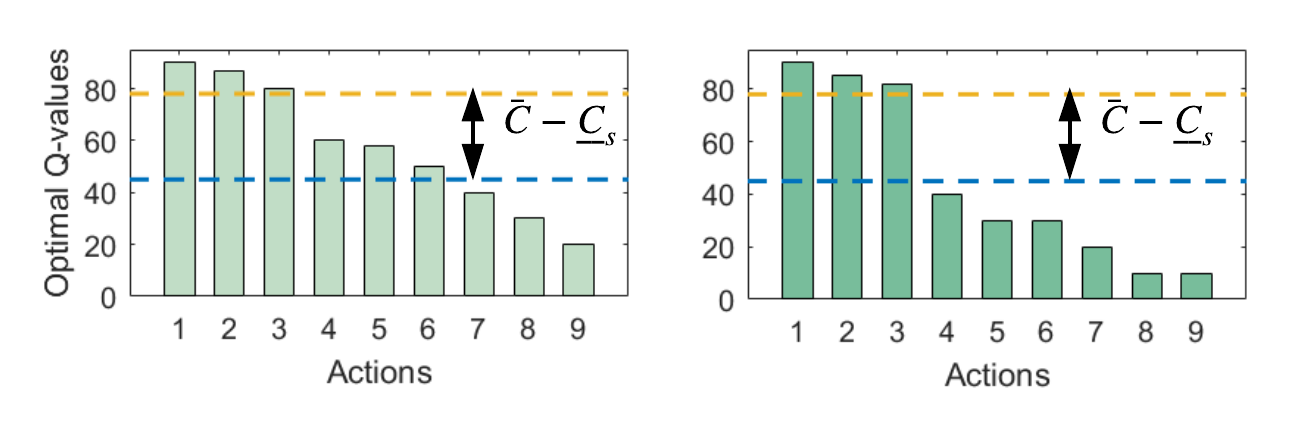}
    \caption{Comparison between two prior models. Actions above the dashed yellow lines are what we should at least consider, which correspond to the sample complexity lower bound; actions above the blue dashed lines are what we should at most consider, which correspond to the sample complexity upper bound. For the left model, the upper bound is strictly greater than the lower bound; for the right model, due to a large value gap, the upper bound has the same order as the lower bound.}
    \label{fig:gap}
\end{figure}

\subsection{Empirical Verification of the Worst Case}\label{sec:toyexample1}
Besides the previous simple case,
we also do a numerical demonstration on a sailing problem \citep{sailing}. 
In Figure \ref{fig:toytest}, we generate two MDPs $\cM_0$ and $\cM$ with $\cM\in B_{\mathrm{TV}}(\cM_0,0.3)$. We compare the performances of two algorithms: 1. direct Q-learning \citep{watkins1992q} with transition samples from $\cM$ (blue line); 2. use the full knowledge of $\cM_0$ to generate a nearly optimal $Q$-function for $\cM_0$, then use that $Q$-function to initialize the proceeding Q-learning algorithm with transition samples from $\cM$ (red line). Both algorithms use the same batch of transition samples from $\cM$. Since $\cM_0$ is close to $\cM$, the warm-start Q-learning is much better than the learning-from-scratch counterpart in the initial stage. However, these two curves overlap when they become closer to the optimal value, indicating similar sample complexities for both algorithms when pursuing a high-precision Q-value estimation.
\begin{figure}[htbp!]
  \centering
  \includegraphics[width=.5\textwidth]{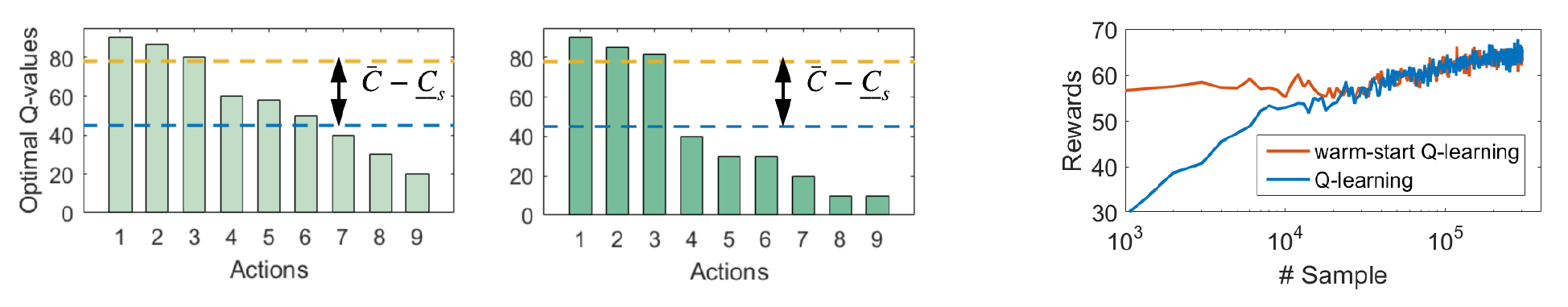}\\
  \caption{A toy test comparison between direct Q-learning and warm-start Q-learning with a nearly optimal Q-value of $\cM_0$ as initialization.}
  \label{fig:toytest}
\end{figure}

\section{Conclusion and Future Work}
In this paper, we show that 
transfer learning 
from a TV-distance neighbourhood
cannot reduce the sample complexity
of learning a high precision policy of an MDP,
unless extra structural information is provided. We study the case where the new unknown model can be represented as a convex combination of a finite set of base models. In this setting, transfer learning  achieves significantly lower sample complexity compared with learning from scratch.

\newpage
\appendix
\section{A Case Study for Knowledge Transfer in Reinforcement Learning}
In this section, we impose a new assumption on \emph{similarities} among models such that transferring knowledge achieves fast adaptation. We consider a sequence of MDPs, where they have the same state and action spaces, and the discount factor, but different transition dynamics and/or reward functions. At each time step $t$, we want to learn an $\varepsilon$-optimal policy for  $\cM_t:=(\cS,\cA,\vP_t, R_t, \gamma).$ The assumption we propose is a convex hull structure as stated below.

\begin{assumption} \label{ass:convex}Given a finite set of MDPs $\mathbb{M}:=\{\cM^1, \cM^2, \dots, \cM^K\}$ where $\cM^k=(\cS,\cA,\vP^k, R^k, \gamma)$, we have $\cM_t\in \text{conv}(\mathbb{M})$\footnote{$\cM\in \text{conv}(\mathbb{M})$ if there exists a vector $C:=[c_1,c_2,\dots,c_k]^\top\in\Delta^k$ such that for any $(s,a)\in\cS\times\cA$, $P(\cdot|s,a)= \sum_{k=1}^K c_k P^k(\cdot|s,a)$ and $R=\sum_{k=1}^K c_k R^k$.} for all $t>0$. We have full knowledge of all MDPs in $\mathbb{M}$ and access to a generative model of each $\cM_t$.
\end{assumption}

We define a set of matrices $\cV:=\{\vV_{s,a}\in\RR^{|\cS|\times K}, s\in \cS, a\in\cA\}$,
\cut{\begin{equation}
\vV_{s,a} = \begin{bmatrix} p^1(s_1|s,a) & p^2(s_1|s,a) & \cdots & p^K(s_1|s,a)\\
p^1(s_2|s,a) & p^2(s_2|s,a) & \cdots & p^K(s_2|s,a)\\
\vdots & \vdots & \vdots & \vdots\\
p^1(s_{|\cS|}|s,a) & p^2(s_{|\cS|}|s,a) & \cdots & p^K(s_{|\cS|}|s,a)
\end{bmatrix},
\end{equation}}
where the $k$th column of $\vV_{s,a}$ is $P^k(\cdot|s,a)$. Since $\cM_t\in\text{conv}(\mathbb{M})$, there exists a vector $C_t\in\Delta^K$ such that
\begin{align}\label{eq:convex}
P_t(\cdot|s,a)=\vV_{s,a}~ C_t, \quad \forall~(s,a)\in\cS\times\cA.
\end{align}
We define a matrix $\vU\in\RR^{|\cS|^2|\cA|\times K}$ by stacking all $\vV_{s,a}$ vertically, i.e. $$\vU:=\begin{bmatrix} \vV_{s_1,a_1}\\
\vV_{s_1,a_2}\\
 \vdots\\
\vV_{s_{|\cS|}, a_{|\cA|}}\end{bmatrix}.$$
We make the following assumption about $\vU$.
\begin{assumption}\label{ass:fullrank}
$\vU$ has full column rank.
\end{assumption}
Since $K$ is much smaller than $|\cS|^2|\cA|$, the assumption is easy to be satisfied in real applications. Then a direct result is:
\begin{lemma}\label{lemma:kpair}
There exists a set $\{(s'_k,a'_k)\}_{k=1}^K$ such that the matrix formed by stacking all $\vV_{s'_k,a'_k}$ vertically has column rank $K$.
\cut{\begin{proof}
We apply column-wise Gaussian elimination on $\vU$. Due to that $\vU$ is of full rank, we will get a matrix as:
$$\begin{bmatrix}
u'_1   & 0      & 0      & \cdots & 0     \\
\times & 0      & 0      & \cdots & 0     \\
\vdots & \vdots & \vdots & \cdots & \vdots\\
\times & u'_2   & 0      & \cdots & 0     \\
\times & \times & 0      & \cdots & 0\\
\vdots & \vdots & \vdots & \cdots & \vdots\\
\times & \times & \times & \cdots & u'_K\\
\vdots & \vdots & \vdots & \cdots & \vdots\\
\times & \times & \times & \cdots & \times
\end{bmatrix}.
$$
Then we select the corresponding $\vV_{s,a}$ where $\{u'_k\}_{k=1}^K$ lie in, there will be at most $K$ pairs of $(s,a)$. One can always add redundant pairs and collect $K$ pairs in total. Denote these selected pairs by $\{(s'_k,a'_k)\}_{k=1}^K$, then $[\vV(s'_1,a'_1);\vV(s'_2,a'_2);\cdots;\vV(s'_K,a'_K)]$ has column rank $K$.
\end{proof}}
\end{lemma}
The proof is easy with basic linear algebra. Let $\{(s'_k,a'_k)\}_{k=1}^K$ be the set in Lemma \ref{lemma:kpair}. We define
\begin{align}\label{eq:Utrun}
&\vU_{\text{trun}}\in\RR^{K|\cS|\times K}:=\frac{1}{K}\begin{bmatrix}\vV_{s'_1,a'_1}\\
\vV_{s'_2,a'_2}\\
\vdots\\
\vV_{s'_K,a'_K}\end{bmatrix}, \quad
P_t\in\RR^{K|\cS|} :=\frac{1}{K}\begin{bmatrix}P_t(\cdot|s'_1,a'_1)\\
P_t(\cdot|s'_2,a'_2)\\
\vdots\\
P_t(\cdot|s'_K,a'_K)\end{bmatrix}.
 \end{align}
Basically, we shrink the size of $\vU$ and normalize it to $\vU_{\text{trun}}$. Equation \eqref{eq:convex} reduces to
\begin{equation}\label{eq:realC}
P_t = \vU_{\text{trun}} C_t.
\end{equation}
Let $\lambda_{\text{max}}$ and $\lambda_{\text{min}}$ be the largest and smallest eigenvalues of $\vU_{\text{trun}}^\top\vU_{\text{trun}}$, respectively. Note that $\lambda_{\text{min}}>0$ due to the full column rank property of $\vU_{\text{trun}}$. \cut{In the next lemma, we show that $\lambda_{\text{max}} \leq 1/K$ and $\lambda_{\min}$ can be arbitrarily small.
\begin{lemma}\label{lemma:eigen}
 $0<\lambda_{\text{max}}\leq 1/K$ and $\lambda_{\text{min}}>0$ but $\lambda_{\text{min}}$ can be arbitrarily close to 0.

\begin{proof}
First of all, since $\vU_{\text{trun}}^T\vU_{\text{trun}}$ is positive definite (due to full rank of $\vU_{\text{trun}}$), $\lambda_{\text{max}}\geq\lambda_{\text{min}}>0$. We rewrite $\vU_{\text{trun}}$ as $[P^1, P^2, \cdots, P^K]$, where for each $i\in\{1,2,\dots, K\}$, $$P^i\in\RR^{K|\cS|}:=\begin{bmatrix} \frac{1}{K}P^i(\cdot|s'_1,a'_1)\\
\frac{1}{K}P^i(\cdot|s'_2,a'_2)\\
\vdots\\
\frac{1}{K}P^i(\cdot|s'_K,a'_K)\end{bmatrix}.$$
And we have that $\|P^i\|_2 = \frac{\sqrt{\sum_{j=1}^K \|P^i(\cdot|s'_j,a'_j)\|^2_2}}{K}\leq \sqrt{K}/{K} = 1/\sqrt{K}$. Then, $\vU_{\text{trun}}^T\vU_{\text{trun}}$ can be rewritten as
\begin{equation}
\begin{bmatrix}
\|P^1\|_2^2 & (P^1)^TP^2 & \cdots & (P^1)^TP^K\\
(P^2)^TP^1 & \|P^2\|_2^2 & \cdots & (P^2)^TP^K\\
\vdots & \vdots & \vdots & \vdots \\
(P^K)^TP^1 & (P^K)^TP^2 & \cdots & \|P^K\|_2^2
\end{bmatrix}
\end{equation}
If we apply Guassian elimination to $\vU_{\text{trun}}^T\vU_{\text{trun}}$ without changing the eigenvalues, we can get the following equivalent diagonal matrix:
\begin{equation}
\begin{bmatrix}
\|P^1\|_2^2 & 0  & 0 & \cdots \\
0 & \|P^2\|_2^2-\frac{((P^1)^TP^2)^2}{\|P^1\|_2^2}  & 0 & \cdots \\
0 & 0 & \|P^3\|_2^2-\frac{((P^1)^TP^3)^2}{\|P^1\|_2^2}-\frac{\Big(P_2^TP_3-\frac{((P^1)^TP^3)((P^1)^TP^2)}{\|P^1\|_2^2}\Big)^2}{\|P^2\|_2^2-\frac{((P^1)^TP^2)^2}{\|P^1\|_2^2}} & \cdots\\
\vdots & \vdots & 0 & \ddots  \\
\end{bmatrix}
\end{equation}
Since the matrix is positive definite, all diagonal entries should be positive. Then, the $i$th diagonal element in the above matrix is equal to $\|P^i\|_2^2$ minus some nonnegative quantity. Thus, the diagonal entries are all smaller than $1/K$. Therefore, $\lambda_{\text{max}}\leq 1/K$.
To show that $\lambda_{\text{min}}$ can be arbitrarily small, we consider the following example. Denote by $e_i$ the unit vector with the $i$th component equalling 1. Then, let $P^i(\cdot|s'_j,a'_j)\equiv (1-\delta_i)e_1 + \delta_i e_i$ for all $j$. In this case, $\vU_{\text{trun}}$ has the form as:
\begin{equation}
\begin{bmatrix}
\frac{1}{K} & \frac{1-\delta_2}{K} & \frac{1-\delta_3}{K} & \cdots & \frac{1-\delta_K}{K}\\
0 & \frac{\delta_2}{K} & 0 & \cdots & 0\\
0 & 0 & \frac{\delta_3}{K} & \cdots & 0 \\
\vdots & \vdots & \vdots & \cdots & \frac{\delta_K}{K}\\
\frac{1}{K} & \frac{1-\delta_2}{K} & \frac{1-\delta_3}{K} & \cdots & \frac{1-\delta_K}{K}\\
0 & \frac{\delta_2}{K} & 0 & \cdots & 0\\
0 & 0 & \frac{\delta_3}{K} & \cdots & 0 \\
\vdots & \vdots & \vdots & \cdots & \frac{\delta_K}{K}\\
\vdots & \vdots & \vdots & \cdots & \vdots\\
\end{bmatrix}.
\end{equation}
Note that as long all $\delta_i$ are positive, the above matrix is of rank $K$. In this case, $\vU_{\text{trun}}^T\vU_{\text{trun}}$ is
\begin{equation}
\begin{bmatrix}
\frac{1}{K} & \frac{1-\delta_2}{K} & \frac{1-\delta_3}{K} & \cdots & \frac{1-\delta_{K}}{K}\\
\frac{1-\delta_2}{K} & \frac{(1-\delta_2)^2+\delta_2^2}{K} & \frac{(1-\delta_2)(1-\delta_3)}{K}& \cdots & \frac{(1-\delta_2)(1-\delta_K)}{K}\\
\frac{1-\delta_3}{K} & \frac{(1-\delta_2)(1-\delta_3)}{K} & \frac{(1-\delta_3)^2+\delta_3^2}{K} & \cdots & \frac{(1-\delta_2)(1-\delta_K)}{K}\\
\vdots & \vdots & \vdots & \cdots & \vdots \\
\frac{1-\delta_K}{K} & \frac{(1-\delta_K)(1-\delta_2)}{K} & \frac{(1-\delta_K)(1-\delta_3)}{K}& \cdots & \frac{(1-\delta_K)^2+\delta_K^2}{K}\\
\end{bmatrix}
\end{equation}
We can apply row-wise Gaussian elimination without changing the eigenvalues of $\vU_{\text{trun}}^T\vU_{\text{trun}}$ and get the following equivalent form:
\begin{equation}
\begin{bmatrix}
\frac{1}{K} & \frac{1-\delta_2}{K} & \frac{1-\delta_3}{K} & \cdots & \frac{1-\delta_{K}}{K}\\
0 & \frac{\delta_2^2}{K} & 0& \cdots & 0\\
0 & 0 & \frac{\delta_3^2}{K} & \cdots & 0\\
\vdots & \vdots & \vdots & \cdots & \vdots \\
0 & 0 & 0& \cdots & \frac{\delta_K^2}{K}\\
\end{bmatrix}.
\end{equation}
Clearly, the smallest eigenvalue for the above matrix is $\min_i \delta_i^2/K$, which can be arbitrarily close to 0 with small enough $\delta_i$'s. Thus, $\lambda_{\text{min}}$ can be arbitrarily close to 0.
\end{proof}
\end{lemma}
Next,}
We give an algorithm to learn an $\varepsilon$-optimal policy for $\cM_t$. The algorithm is presented in Algorithm \ref{alg:case2}. For every $t>0$, we first take samples to construct $P_t$'s empirical estimation $\widehat{P}_t$. Then we find a vector $\widehat{C}_t\in\Delta^K$ such that $\vU_{\text{trun}}\widehat{C}_t\approx \widehat{P}_t$. Next, we use $\widehat{C}_t$ to form a model $\widetilde{\cM}_t\in\text{conv}(\mathbb{M})$ as an approximation to the true $\cM_t$. Finally, the algorithm returns a policy which is $\varepsilon/2$-optimal policy for $\widetilde{\cM_t}$. We will show in the proceeding that the output policy is $\varepsilon$-optimal for $\cM_t$.
\cut{This process is illustrated in Figure \ref{fig:alg} from a model perspective.
\begin{figure}[htbp!]
  \centering
  \includegraphics[width=.5\textwidth]{Algorithmpng}\\
  \caption{An illustration for Algorithm \ref{alg:case2}. The black dot is $\cM_t$, the blue dot is the empirical approximation $\widehat{\cM}_t$, the green dot is the projection of $\widehat{\cM}_t$ to the hyperplane spanned by $\mathbb{M}$, and the red dot is the projection of $\cM_{\text{proj}}$ to the convex hull. The solid black line represents empirical error, the green line represents optimization error due to inexact solving, and the red line represents the distance between the real model and the model we use to produce a policy.}
  \label{fig:alg}
\end{figure}}
\begin{algorithm}
  \caption{Transfer RL with a Convex Model Space}
  \label{alg:case2}
  \begin{algorithmic}[1]
  \State \textbf{Input:} $\mathbb{M}:=\{\cM^1,\cM^2,\dots,\cM^K\}, \{(s'_k,a'_k)\}_{k=1}^K, \text{ a generative model } $\texttt{GM}$ ,  \varepsilon>0, \delta>0,$ a zero vector $\widehat{P}_t\in\RR^{K|\cS|}, L=\lceil\frac{432K\lambda_{\text{max}}}{\varepsilon^2(1-\gamma)^4\lambda_{\text{min}}^2}\log(\frac{1+K|\cS|}{\delta})\rceil$.
  \While {$t>0$}
    \State $i\gets 1$, $\widehat{P}_t\gets 0$;
    \While {$i\leq L$}
     \State Sample $j\in[K]$ uniformly;
     \State Sample $(s,r)\gets \texttt{GM}(s'_j,a'_j)$;
     \State Increment the $n$th coordinate in $\widehat{P}_t$ by 1, where $n=(j-1)|\cS|+s$;
     \State $i\gets i+1;$
     \EndWhile
     \State Normalize $\widehat{P}_t$: $\widehat{P_t}\gets \frac{1}{L}\widehat{P_t}$;
     \State Calculate $\widehat{C}_t:=\Proj_{\Delta^K}(\vU_{\text{trun}}^\top\vU_{\text{trun}})^{-1}\vU_{\text{trun}}^\top\widehat{P}_t$ ($\vU_{\text{trun}}$ as defined in Equation \eqref{eq:Utrun});
     \State Formulate $\widetilde{\cM}_t:=(\cS,\cA,\sum_{i=1}^K c_t^K \vP^k, \sum_{k=1}^K c_t^KR^k, \gamma)$, where $[c_t^1, c_t^2, \dots, c_t^K]^\top=\widehat{C}_t;$
     \State Run any planning algorithm and get an $\varepsilon/2$-optimal policy $\pi_t$ for $\widetilde{\cM}_t$.
     \State \textbf{Output:} $\pi_t$.
\EndWhile
\end{algorithmic}
\end{algorithm}

Our proof consists of two steps: 1. we show in Lemma \ref{lemma:closemdp} that if the convex coefficients of two MDPs $\cM, \widehat{\cM} \in \text{conv}(\mathbb{M})$ are close under $\|\cdot\|_2$, then an $\varepsilon$-optimal policy $\pi$ for $\widehat{\cM}$ is also nearly optimal for $\cM$; 2. we show in Lemma \ref{lemma:cbound} that the convex coefficients of $\widetilde{\cM}_t$ and $\cM_t$ are close under $\|\cdot\|_2$.

\begin{lemma}\label{lemma:closemdp}
Suppose for MDPs $\cM:=(\cS,\cA,\vP,R,\gamma)$ and $\widehat{\cM}:=(\cS,\cA,\widehat{\vP},\widehat{R},\gamma)$, there exists two vectors $C:=[c_1,c_2,\dots,c_K]\in\Delta^K$ and $D:=[d_1,d_2,\dots,d_K]\in\Delta^K$ such that
\begin{align}
\vP&=\sum_{k=1}^K c_k \vP^k,\quad  R=\sum_{k=1}^K c_k R^k;\\
\widehat{\vP} &=\sum_{k=1}^K d_k \vP^k, \quad \widehat{R} = \sum_{k=1}^K d_k R^k.
 \end{align}
If $\|C-D\|_2\leq \alpha$, then an $\varepsilon'$-optimal policy for $\widehat{\cM}$ is $\bigg(\varepsilon'+6\alpha\sqrt{K}/(1-\gamma)^2\bigg)$-optimal for $\cM$.
\begin{proof}
Denote by $V^*$ and $\widehat{V}^*$ the optimal value vectors for $\cM$ and $\widehat{\cM}$, respectively. Given an $\varepsilon'$-optimal policy $\pi$ for $\widehat{\cM}$, we denote by $V^{\pi}$ and $\widehat{V}^{\pi}$ the value vectors following $\pi$ in $\cM$ and $\widehat{\cM}$, respectively. Then, by triangle inequality, we first have that
\cut{\begin{align}
Q^*-Q^{\pi} &\leq \|Q^{\pi}-\widehat{Q}^{\pi}\|+\|\widehat{Q}^{\pi}-\widehat{Q}^*\|+\|\widehat{Q}^*-Q^*\|\\
&\leq \|Q^{\pi}-\widehat{Q}^{\pi}\|+\|\widehat{Q}^*-Q^*\|+\varepsilon'
\end{align}}
\begin{align}
\|V^{\pi}-V^*\|_{\infty} &\leq \|V^{\pi}-\widehat{V}^{\pi}\|_{\infty}+\|\widehat{V}^{\pi}-\widehat{V}^*\|_{\infty}+\|\widehat{V}^*-V^*\|_{\infty}\\
&\leq \|V^{\pi}-\widehat{V}^{\pi}\|_{\infty}+\varepsilon'+\|\widehat{V}^*-V^*\|_{\infty}\label{eq:valuebound}
\end{align}
\cut{\begin{align}
Q^{\pi} &= R + \gamma P^{\pi} Q^{\pi}\\
\widehat{Q}^{\pi} &= \widehat{R} + \gamma \widehat{P}^{\pi} \widehat{Q}^{\pi}.
\end{align}}
To bound $\|V^{\pi}-\widehat{V}^{\pi}\|_{\infty}$, we notice that from Equation \eqref{eq:Vpi},
\cut{\begin{align}
V^{\pi}(s) &= r(s,\pi(s)) + \gamma P(\cdot|s,\pi(s)) V^{\pi}=r(s,\pi(s))+\gamma \big(\sum_{i=1}^K c_i P^i(\cdot|s,\pi(s))\big)V^{\pi}\\
\widehat{V}^{\pi}(s) &= \widehat{r}(s,\pi(s)) + \gamma \widehat{P}(\cdot|s,\pi(s)) \widehat{V}^{\pi}=\widehat{r}(s,\pi(s))+\gamma \big(\sum_{i=1}^K d_i P^i(\cdot|s,\pi(s))\big)\widehat{V}^{\pi}\\
\end{align}
Let $R^{\pi}:=[r(s_1,\pi(s_1));r(s_2,\pi(s_2));\dots;R(s_{|\cS|},\pi(s_{|\cS|}))]$, $\widehat{R}^{\pi}:=[\widehat{R}(s_1,\pi(s_1));\widehat{R}(s_2,\pi(s_2));\dots;\widehat{R}(s_{|\cS|},\pi(s_{|\cS|}))]$, $P^{\pi}:=[p(\cdot|s_1,\pi(s_1)); p(\cdot|s_2,\pi(s_2));\dots;p(\cdot|s_{|\cS|},\pi(s_{|\cS|}))]$, and $\widehat{P}^{\pi}:=[\widehat{p}(\cdot|s_1,\pi(s_1)); \widehat{p}(\cdot|s_2,\pi(s_2));\dots;\widehat{p}(\cdot|s_{|\cS|},\pi(s_{|\cS|}))]$. Then we have that}
\begin{align}
    V^{\pi}=(\vI-\gamma \vP^{\pi})^{-1} R^{\pi}, \quad \widehat{V}^{\pi}=(\vI-\gamma \widehat{\vP}^{\pi})^{-1} \widehat{R}^{\pi}.
\end{align}
\cut{\begin{align}
\|Q^{\pi}-\widehat{Q}^{\pi}\|&= \|(I-\gamma P^{\pi})^{-1} R - (I-\gamma \widehat{P}^{\pi})^{-1} \widehat{R}\|\\
&= \|\sum_{n=0}^{\infty}\gamma^n\big((P^{\pi})^n R-(\widehat{P}^{\pi})^n\widehat{R}\big)\|\\
&\leq \|(I-\gamma P^{\pi})^{-1} R - (I-\gamma \widehat{P}^{\pi})^{-1} R\| + \beta/(1-\gamma)\\
&= \|(\vI-\gamma \vP^{\pi})^{-1}(\gamma P^{\pi}-\gamma \widehat{P}^{\pi})(I-\gamma \widehat{P}^{\pi})^{-1} R\| + \beta/(1-\gamma)\\
&\leq \frac{\gamma\beta}{(1-\gamma)^2} + \frac{\beta}{1-\gamma}
\end{align}}

By Gershgorin Circle Theorem \citep{gershgorin1931uber}, the absolute values of all eigenvalues of $\gamma \vP^{\pi}$ are strictly smaller than 1. So $(\vI-\gamma \vP^{\pi})^{-1}=\vI+\sum_{n=1}^\infty \big(\gamma\vP^{\pi}\big)^n$. Based on this, it holds that
\begin{align}
\|V^{\pi}-\widehat{V}^{\pi}\|_{\infty}&= \left\|(\vI-\gamma \vP^{\pi})^{-1} R^{\pi}-(\vI-\gamma \widehat{\vP}^{\pi})^{-1} \widehat{R}^{\pi}\right\|_{\infty} \\
&= \left\|R^{\pi}-\widehat{R}^{\pi} + \sum_{n=1}^{\infty} \gamma^n \Big[(\vP^{\pi})^n R^{\pi}-(\widehat{\vP}^{\pi})^n \widehat{R}^{\pi}\Big]\right\|_{\infty}\\
&\leq \|R^{\pi}-\widehat{R}^{\pi}\|_{\infty}+\sum_{n=1}^{\infty} \gamma^n \Big[\left\|\big((\vP^{\pi})^n-(\widehat{\vP}^{\pi})^n\big) R^{\pi}\right\|_{\infty} + \left\|(\widehat{\vP}^{\pi})^n(R^{\pi}-\widehat{R}^{\pi})\right\|_{\infty}\Big]\\
&\leq \|R^{\pi}-\widehat{R}^{\pi}\|_{\infty}+\sum_{n=1}^{\infty} \gamma^n \left[\left\|(\vP^{\pi}-\widehat{\vP}^{\pi})\Big(\sum_{i=0}^{n-1} (\vP^{\pi})^i(\widehat{\vP}^{\pi})^{n-1-i}\Big) R^{\pi}\right\|_{\infty} + \|R^{\pi}-\widehat{R}^{\pi}\|_{\infty}\right]\\
&= \frac{\|R^{\pi}-\widehat{R}^{\pi}\|_{\infty}}{1-\gamma}+\sum_{n=1}^{\infty} \gamma^n \Big[\|(\vP^{\pi}-\widehat{\vP}^{\pi})R_n\|_{\infty}\Big]\label{eq:ineq}
\end{align}
where $R_n:=\sum_{i=0}^{n-1} (\vP^{\pi})^i(\widehat{\vP}^{\pi})^{n-1-i}R^{\pi}\in[0,n]^{|\cS|}$. The third line is by triangle inequality and the fourth line is due to that $(\widehat{\vP}^{\pi})^n$ is a transition matrix.
For the first term in Equation \eqref{eq:ineq}, we have that
\begin{align}
\|R^{\pi}-\widehat{R}^{\pi}\|_{\infty}&=\max_{s\in\cS}\left|P(\cdot|s,\pi(s))^\top R(s,\pi(s),\cdot)-\widehat{P}(\cdot|s,\pi(s))^\top \widehat{R}(s,\pi(s),\cdot)\right|\\
&\leq \max_{s\in\cS}\left| P(\cdot|s,\pi(s))^\top R(s,\pi(s),\cdot)-\widehat{P}(\cdot|s,\pi(s))^\top R(s,\pi(s),\cdot)\right|\\
&\quad+\max_{s\in\cS}\left|\widehat{P}(\cdot|s,\pi(s))^\top R(s,\pi(s),\cdot) - \widehat{P}(\cdot|s,\pi(s))^\top \widehat{R}(s,\pi(s),\cdot)\right|\\
&=\max_{s\in\cS}\left|\sum_{k=1}^K (c_k-d_k) P^k(\cdot|s,\pi(s))^\top R(s,\pi(s), \cdot)\right|+\max_{s\in\cS}\left|\sum_{k=1}^K(c_k-d_k) \widehat{P}(\cdot|s,\pi(s))^\top R^k(s,\pi(s), \cdot)\right|\\
&\leq 2\|C-D\|_1\leq 2\sqrt{K}\|C-D\|_2,\label{eq:Rpi}
\end{align}
where the second line is by triangle inequality and the last line is due to $R^k(s,\pi(s),\cdot)\in[0,1]^{|\cS|}$. For the second term in Equation \eqref{eq:ineq}, observe that
\begin{align}
\|(\vP^{\pi}-\widehat{\vP}^{\pi})R_n\|_{\infty}&=\max_{s\in\cS} \left|\sum_{k=1}^K (c_k-d_k)P^k(\cdot|s,\pi(s))^\top R_n\right|\\
&\leq \|C-D\|_1 \|R_n\|_{\infty}\leq n\sqrt{K}\|C-D\|_2.\label{eq:pn}
\end{align}
Combining \eqref{eq:Rpi} and \eqref{eq:pn}, we have the following result:
\begin{align}
\eqref{eq:ineq} &\leq 2\alpha\sqrt{K}/(1-\gamma) + \alpha \sqrt{K} \gamma/(1-\gamma)^2\leq 3\alpha\sqrt{K}/(1-\gamma)^2.
\end{align}
To bound $\|\widehat{V}^*-V^*\|_{\infty}$,
let $\pi^*$ and $\widehat{\pi}^*$ be optimal policies for $\cM$ and $\widehat{\cM}$, respectively. Then it holds that
\begin{align}
    \widehat{V}^*-V^* \leq \|\widehat{V}^{\widehat{\pi}^*}-V^{\widehat{\pi}^*}\|_{\infty}, \quad V^* -\widehat{V}^* \leq \|V^{\pi^*}-\widehat{V}^{\pi^*}\|_{\infty}.
\end{align}
Following the same steps as before, we have $\|\widehat{V}^{\widehat{\pi}^*}-V^{\widehat{\pi}^*}\|_{\infty}\leq 3\alpha\sqrt{K}/(1-\gamma)^2$ and $ \|V^{\pi^*}-\widehat{V}^{\pi^*}\|_{\infty}\leq 3\alpha\sqrt{K}/(1-\gamma)^2$. Therefore, $\|\widehat{V}^*-V^*\|_{\infty}\leq 3\alpha\sqrt{K}/(1-\gamma)^2$. The desired result is obtained.
\end{proof}
\end{lemma}

\begin{lemma}\label{lemma:cbound}
 Let $C_t$ be the solution of Equation \eqref{eq:realC}. Then in Algorithm \ref{alg:case2}, for each $t>0$, $\|\widehat{C}_t-C_t\|_2\leq (1-\gamma)^2\varepsilon/(12\sqrt{K})$ with probability at least $1-\delta$.
\begin{proof}
As defined in Equation \eqref{eq:Utrun}, $\|P_t\|_1=1$. Therefore, $P_t$ can be taken as a distribution. The samples we take in Algorithm \ref{alg:case2} are indeed all i.i.d. samples following $P_t$. Therefore, $\widehat{P}_t$ is an empirical approximation of $P_t$ with $L$ samples. By Bernstein inequality for matrices \citep[Theorem 6.1.1]{tropp2015introduction}, we have that   $$\mathbb{P}\Big(\|\widehat{P}_t-P_t\|_2\leq \epsilon \Big)\geq 1-(1+K|\cS|)\exp \Big(\frac{-L^2\epsilon^2}{2L + 2L/3}\Big)\geq 1-(1+K|\cS|)\exp\Big(\frac{-L\epsilon^2}{3}\Big).$$
Taking $\epsilon=\frac{\lambda_{\text{min}}\varepsilon(1-\gamma)^2}{12\sqrt{\lambda_{\text{max}}}\sqrt{K}}$, when $L=\lceil\frac{432K\lambda_{\text{max}}}{\varepsilon^2(1-\gamma)^4\lambda_{\text{min}}^2}\log(\frac{1+K|\cS|}{\delta})\rceil$, we have that
$$
\mathbb{P}\Big(\|\widehat{P}_t-P_t\|_2\leq \frac{\lambda_{\text{min}}\varepsilon(1-\gamma)^2}{12\sqrt{\lambda_{\text{max}}}\sqrt{K}}\Big)\geq 1- \delta.
$$
Based on this, it holds that with probability at least $1-\delta$,
\begin{align}
    \|\widehat{C}_t-C_t\|_2 &= \| \Proj_{\Delta^K}\big((U_{\text{trun}}^\top U_{\text{trun}})^{-1}U_{\text{trun}}^\top  \widehat{P}_t\big)- (U_{\text{trun}}^\top U_{\text{trun}})^{-1}U_{\text{trun}}^\top P_t\|_2\\
    &\leq \|(U_{\text{trun}}^\top U_{\text{trun}})^{-1}U_{\text{trun}}^\top \widehat{P}_t- (U_{\text{trun}}^\top U_{\text{trun}})^{-1}U_{\text{trun}}^\top P_t\|_2\\
& \leq \frac{\sqrt{\lambda_{\text{max}}}}{\lambda_{\text{min}}}\|\widehat{P}_t-P_t\|_2,
\end{align}
where the second line is due to that $\Delta^K$ is a convex set and $C_t\in\Delta^K$. Therefore, $\|\widehat{C}_t-C_t\|_2\leq \frac{\varepsilon(1-\gamma)^2}{12\sqrt{K}}$ with probability at least $1-\delta$.
\end{proof}
\end{lemma}

Combining Lemma \ref{lemma:closemdp} and \ref{lemma:cbound}, and the fact that $\pi_t$ is $\varepsilon/2$-optimal for $\widetilde{\cM}_t$, we have the following result.

\begin{proposition}\label{prop:scforconv}
Let $\varepsilon>0$ and $\delta>0$. Under Assumption \ref{ass:convex} and \ref{ass:fullrank}, for any $t>0$, with probability at least $1-\delta$, Algorithm \ref{alg:case2} returns an $\varepsilon$-optimal policy for $\cM_t$ with samples
\begin{align}
    \cO\bigg(\frac{K}{\varepsilon^2(1-\gamma)^4}\log(\frac{1+K|\cS|}{\delta})\bigg).
\end{align}
\end{proposition}

The sample complexity in Proposition \ref{prop:scforconv} is irrelevant with $|\cS|$ which makes it significantly smaller than the lower bound $\Omega\big(\frac{N}{\varepsilon^2(1-\gamma)^3}\log(1/\delta)\big).$ Therefore, fast adaptation is achieved.

\cut{Using Lemma \ref{lemma:eigen}, we have the following result.
\begin{corollary}
Given $\varepsilon>0$ and $\delta>0$, in Case 2 with Assump. \ref{ass:fullrank}, for any $t>0$, the sample complexity for returning an $\varepsilon$-optimal policy for $M_t$ with probability at least $1-\delta$ is
\begin{align}
    \cO(\frac{1}{\varepsilon^2(1-\gamma)^4\lambda_{\text{min}}^2}\log(\frac{1+K|\cS|}{\delta})),
\end{align}
\end{corollary}}

\cut{\section{Time-varying Reinforcement Learning under New Assumptions}
In the previous section, we showed a sample complexity lower bound for any $(\varepsilon, \delta)$-correct RL algorithm to learn $M^t$, which indeed meets an upper bound of independent learning (i.e. without any information of previous models) as shown in \citealt{azar2013minimax}. This implies that when the accuracy requirement is high, i.e. when $\varepsilon$ is small enough, the information from an approximate model does not help reduce the sample complexity. There are several possible reasons:
\begin{itemize}
\item mean value being close or the transition probability being close does not imply a smaller range of transition samples. Thus, we are unable to have a tighter control on sample values.
\item Such a one-time information is not enough if one wants to try variance reduction technique, since the baseline value cannot be improved.
\end{itemize}

Due to the above reasons, we need to propose further assumptions on the time-varying RL problem such that the previous information can be effectively utilized and the sample complexity upper bound can be smaller than general scenarios.

\subsection{Finite MDPs}
The first assumption we propose is that all $M_t$ belongs to a finite set of MDPs $\cM:=\{M^1,M_2,\dots,M^m\}$, where $M^i:=(\cS,\cA,P^i,R^i,\gamma)$. Besides that $M^i$ and $M^j$ are $\beta-$close, we also have that any two MDPs in $\cM$ are $\alpha-$distant, where $\alpha>0$. At each time step $t$, $M_t=M^i$ with probability $p^i_t$, where $0<p_{\text{min}}\leq p^i_t\leq p_{\text{max}}<1$. Compared with the general setting, this new assumption provides more usable information and new ways of learning.
Depending on how much information we know about $\cM$, we separate the analysis in three cases:
\begin{enumerate}
\item For each $M^i$, we know $P^i$ and $R^i$.
\item For each $M^i$, we know an $\varepsilon/8$-optimal deterministic policy $\pi^i$;
\item We know nothing about the models except that they are $\alpha-$distant.
\end{enumerate}
Given $M_t$, we denote by $\cM(M_t)$ the model $M_t$ is inside of $\cM$.
\subsubsection{Case 1}
In this case, we only need to identify $\cM(M_t)$, then return a corresponding $\varepsilon$-optimal policy. Since any two models are $\alpha-$distant, given $M^i$ and $M^j$, there exists a state-action-state tuple $(s_{(i,j)}, a_{(i,j)}, s'_{(i,j)})$ such that
\begin{equation}\label{eq:alpha}
|P^i(s'_{(i,j)}\vert s_{(i,j)}, a_{(i,j)})-P^j(s'_{(i,j)}\vert s_{(i,j)}, a_{(i,j)})|\geq\alpha.
\end{equation}
Based on this, we can call $\cS\cO$ with input $(s_{(i,j)}, a_{(i,j)})$ enough times to construct an empirical probability and use concentration inequality to rule out wrong candidates. Specifically, given any state-action-state pair $(s,a,s')$, if calling $\cS\cO$ for $K$ times with input $(s,a)$, the empirical probability is defined as:
\begin{equation}\label{eq:Pbar}
\bar{P}(s' \vert s, a):=\frac{\text{\# getting } s' \text { as next state when calling } \cS\cO(s, a)}{K}.
\end{equation}
The Algorithm is presented in Algorithm \ref{alg:case1}.

\vspace*{.5cm}
\begin{algorithm}
  \caption{Case 1 of Finite MDPs}
  \label{alg:case1}
  \begin{algorithmic}[1]
  \State \textbf{Input:} $\cM, \varepsilon>0, \delta>0, K=\lceil\frac{32}{\alpha^2}\log(m/\delta)\rceil, T>0$.
  \State For each $M^i$, compute an $\varepsilon$-optimal policy $\pi^i$ with any planning algorithm
  \State Select $m(m-1)/2$ tuples $(s_{(i,j)}, a_{(i,j)}, s'_{(i,j)})_{1\leq i<j\leq m}$ such that \eqref{eq:alpha} holds.
  \While {$t < T$}
 \State $i\gets 1;~j\gets 2$.
 \While {$j\leq m$}
 \State Take $K$ transition samples from $\cS\cO$ with input $(s_{(i,j)},a_{(i,j)})$;
 \State Compute $\bar{P}(s'_{(i,j)} \vert s_{(i,j)}, a_{(i,j)})$ as in \eqref{eq:Pbar}.
 \If {$|\bar{P}(s'_{(i,j)} \vert s_{(i,j)}, a_{(i,j)})-P^i(s'_{(i,j)} \vert s_{(i,j)}, a_{(i,j)})|\leq \alpha/8$}
 \State $j=j+1$;
 \ElsIf {$|\bar{P}(s'_{(i,j)} \vert s_{(i,j)}, a_{(i,j)})-P^j(s'_{(i,j)} \vert s_{(i,j)}, a_{(i,j)})|\leq \alpha/8$}
 \State $i=j; ~j=j+1$;
 \Else
 \State $i=j+1$;~ $j=j+2;$
 \EndIf
 \EndWhile
 \State \textbf{Output:} $\pi^i$ for $M_t$. \label{lin:case1output}
\EndWhile
\end{algorithmic}
\end{algorithm}

Next, we show the correctness of Algorithm \ref{alg:case1} and calculate the sample complexity.

\begin{lemma}\label{lemma:case1}
Given $\varepsilon>0$ and $\delta>0$, in Algorithm \ref{alg:case1}, for each time step $t$, the output policy is $\varepsilon$-optimal for $M_t$ with probability at least $1-\delta$.
\begin{proof}
We only need to show that at Line \ref{lin:case1output} of Algorithm \ref{alg:case1}, index $i$ labels the correct model for $M_t$.
Denote by $P_t$ the transition probability of $M_t$. By Hoeffding's inequality,
\begin{align}\label{eq:hoeffdingcase1}
\mathbb{P}(|\bar{P}(s'_{(i,j)} \vert s_{(i,j)}, a_{(i,j)})-P_t(s'_{(i,j)} \vert s_{(i,j)}, a_{(i,j)})|\leq \frac{\alpha}{8})\geq 1-\exp(-2K\frac{\alpha^2}{64})\geq 1-\frac{\delta}{n}
\end{align}
Due to Equations \eqref{eq:alpha} and \eqref{eq:hoeffdingcase1}, at least one model among $M^i$ and $M^j$ will be at least $\alpha/4-$distant from $\bar{P}(s'_{(i,j)} \vert s_{(i,j)}, a_{(i,j)})$, which makes it not a candidate with high probability. During the inner loop, at each iteration, we pass at least one model which is not $\cM(M_t)$ with probability at least $1-\delta/n$. In total, by union bound, the final left model, which is labelled by $i$, is the correct model for $M_t$ with probability at least $1-\delta$
\end{proof}
\end{lemma}

\begin{proposition}\label{prop:case1}
Given $\varepsilon>0$ and $\delta>0$, for case 1 of Finite MDPs, it takes $\cO(\frac{m}{\alpha^2}\log(m/\delta))$ samples each time to return an $\varepsilon$-optimal policy with probability at least $1-\delta$ for $M_t$.
\begin{proof}
By Lemma \ref{lemma:case1}, the sample complexity is $\cO((m-1)K)=\cO(\frac{m}{\alpha^2}\log(m/\delta))$.
\end{proof}
\end{proposition}

\subsubsection{Case 2}
For case 2, identifying $\cM(M_t)$ is impossible without model knowledge. Thus, we take a different approach by directly evaluating all policies under $M_t$, then select the one with the highest value. Given a deterministic policy $\pi$, the value vector $V^{\pi}\in \RR^{|\cS|}$ satisfies the following Bellman equation:
\begin{align}\label{eq:bellman}
V^{\pi}=R^{\pi}+\gamma P^{\pi}V^{\pi},
\end{align}
where $R^{\pi}_j=R(s_j,\pi(s_j))$ and $P^{\pi}_{jw}=P(s_w\vert s_j,\pi(s_j))$. Due to the Equation \eqref{eq:bellman}, we first approximate $P^{\pi}$ by sample average, then solve the corresponding linear system to obtain $V^{\pi}$. The algorithm is summarized in Algorithm \ref{alg:case2} and the complexity is shown in Prop. \ref{prop:case2}.

\begin{algorithm}[h]
  \caption{Case 2 of Finite MDPs}
  \label{alg:case2}
  \begin{algorithmic}[1]
  \State \textbf{Input:} $ \varepsilon>0, \delta>0, \Pi:=\{\pi^i\}_{i=1}^m, K=\lceil\frac{192m|\cS|^2\gamma^2}{(1-\gamma)^4\varepsilon^2}\log\big(\frac{(1+|\cS|)|\cS|^2}{\delta}\big)\rceil, T>0$.
  \While {$t < T$}
  \State $\Pi'\gets\Pi$;
  \For {$i=1;i\leq m ;i=i+1$}
  \For {$j=1;j\leq|\cS|;j=j+1$}
 \State Take $K$ transition samples from $\cS\cO$ with input $(s_j,\pi^i(s_j))$;
 \State Record reward $R(s_j,\pi^i(s_j))$ and compute $\bar{P}^{\pi^i}_{j,w}= \frac{\text{\# next state is }s_w}{K}$;
 \EndFor
 \State Compute $\bar{V}^{\pi^i} := (I-\gamma\bar{P}^{\pi^i})^{-1}R^{\pi^i};$
 \EndFor
 \For {$j=1;j\leq |\cS|;j=j+1$}\label{lin:forcase2}
 \State $\bar{V}^{\text{max}}_j\gets\max_{\pi\in\Pi'} \bar{V}^{\pi}_j$;
 \State Check all policies in $\Pi'$ and eliminate $\pi\in\Pi'$ if $\bar{V}^{\pi}_j<\bar{V}^{\text{max}}_j-\varepsilon/2$;
 \EndFor\label{lin:endforcase2}
 \State \textbf{Output:} Any policy from $\Pi'$.
\EndWhile
\end{algorithmic}
\end{algorithm}
\begin{lemma}\label{lemma:nbound}
Given two transition matrix $P^1$ and $P^2$, if $\|P^1-P^2\|_2\leq\epsilon$, then $\|P^1^n-P^2^n\|_2\leq n\epsilon$.
\begin{proof}
where the second inequality comes from
\begin{align}
\|P^1^n-P^2^n\|_2& \leq\|P^1-P^2\|_2\|P^1^{n-1}+P^1^{n-2}P^2+\dots+P^2^{n-1}\|_2\\
&\leq \|P^1-P^2\|_2 \Big(\sum_{j=0}^{n-1} \| (P^1^jP_2^{n-1-j}\|_2\Big)\\
&\leq \|P^1-P^2\|_2 \Big(\sum_{j=0}^{n-1} \| P^1\|^j_2\|P^2\|^{n-1-j}_2\Big)\leq n\epsilon.
\end{align}
\end{proof}
\end{lemma}

\begin{lemma}\label{lemma:case2}
Given $\varepsilon\in(0,1)$, $\delta\in(0,1)$, and $m$ $\varepsilon/8$-optimal policies for each $M^i$ respectively, in Algorithm \ref{alg:case2}, for each time step $t$, the output policy is $\varepsilon$-optimal for $M_t$ with probability at least $1-\delta$.
\begin{proof}
By Bernstein's inequality, we know that with $K:=\lceil\frac{192m|\cS|^2\gamma^2}{(1-\gamma)^4\varepsilon^2}\log\big(\frac{(1+|\cS|)|\cS|m}{\delta}\big)\rceil$ transition samples to approximate the $j$th row of $P^{\pi^i}$, the following holds
\begin{align}
\mathbb{P}(\|\bar P^{\pi^i}(j,\cdot)-P^{\pi^i}(j,\cdot)\|_2\leq \frac{(1-\gamma)^2}{8\gamma|\cS|}\varepsilon)\geq 1- (1+|\cS|)\exp\Big(\frac{-K(1-\gamma)^4\varepsilon^2}{192\gamma^2|\cS|^2}\Big)\geq 1-\delta/(|\cS|m).
\end{align}
By taking union bound over all rows, we further have that
\begin{align}
\mathbb{P}(\|\bar P^{\pi^i}-P^{\pi^i}\|_2 \leq \varepsilon':=\frac{(1-\gamma)^2}{8\gamma\sqrt{|\cS|}}\varepsilon) \geq 1- \delta/m.
\end{align}
Denote by $\bar{V}^{\pi^i}:= (I-\gamma\bar{P}^{\pi^i})^{-1}R^{\pi^i}$, then by Lemma \ref{lemma:nbound},
\begin{align}\label{eq:vbound}
\|\bar{V}^{\pi^i}-V^{\pi^i}\|_{\infty}&= \|(I-\gamma P^{\pi^i})^{-1}R^{\pi^i}-(I-\gamma\bar{P}^{\pi^i})^{-1}R^{\pi^i}\|_{\infty}\\
&=\|(I+\gamma\bar{P}^{\pi^i}+(\gamma\bar{P}^{\pi^i})^2+\dots)R^{\pi^i}-(I+\gamma P^{\pi^i} + (\gamma P^{\pi^i})^2+\dots)R^{\pi^i}\|_{\infty}\\
&=\|\gamma(\bar{P}^{\pi^i}-P^{\pi^i})R^{\pi^i} + \gamma^2((\bar{P}^{\pi^i})^2-(P^{\pi^i})^2)R^{\pi^i} + \dots\|_{\infty}\\
&\leq \|\gamma(\bar{P}^{\pi^i}-P^{\pi^i})R^{\pi^i}\|_{\infty} + \|\gamma^2((\bar{P}^{\pi^i})^2-(P^{\pi^i})^2)R^{\pi^i}\|_{\infty} + \dots\\
&\leq \Big(\gamma\|\bar{P}^{\pi^i}-P^{\pi^i}\|_2+\gamma^2\|(\bar{P}^{\pi^i})^2-(P^{\pi^i})^2\|_{2}+\dots\Big)\|R^{\pi^i}\|_2\\
&\leq (\gamma\varepsilon' + 2\gamma^2\varepsilon' + \dots + n\gamma^n\varepsilon' \dots) \sqrt{|\cS|}\\
&\leq \frac{\sqrt{|\cS|}\gamma \varepsilon'}{(1-\gamma)^2}=\varepsilon/8 \text{ with prob. at least $1-\delta/m$}
\end{align}
Next, we show that after Line \ref{lin:endforcase2}, $\Pi'$ only contains $\varepsilon$-optimal policies for $M_t$. Define the following events:
 \begin{align}
 \cE &:=\{ \text{for all } i\in \{1,2,\dots,m\}, \|\bar{V}^{\pi^i}-V^{\pi^i}\|_{\infty} \leq \varepsilon/8\}\\
 \cE_j &:=\{\text{After } j\text{th iteration of Line \ref{lin:forcase2}}, \Pi' \text{ contains  an } \varepsilon/8-\text{optimal policy for } M_t\}.
  \end{align}
By Eqn \eqref{eq:vbound}, $\mathbb{P}(\cE)\geq 1-\delta$. Suppose right after the $j-1$th iteration of Line \ref{lin:forcase2}, $\Pi_{j-1}'=\{\pi^{j_1},\pi^{j_2},\dots, \pi^{j_k}\}$, where $k\leq m$, and that $\bar{V}^{\pi^{j_1}}_j = \bar{V}^{\text{max}}_j$. If a policy $\pi\in\Pi'_{j-1}$ satisfies that $\bar{V}^{\pi}_j<\bar{V}^{\text{max}}_j-\varepsilon/2$, then on $\cE$, we have that
\begin{align}
V^*_j-V^{\pi}_j&=V^*_j-V^{\pi^{j_1}}_j+V^{\pi^{j_1}}_j-\bar{V}^{\pi^{j_1}}_j+\bar{V}^{\pi^{j_1}}_j-\bar{V}^{\pi}_j+\bar{V}^{\pi}_j-V^{\pi}_j\\
&> 0 -\varepsilon/8 + \varepsilon/2 - \varepsilon/8 >\varepsilon/8.
\end{align}
The above equation tells that when $\cE$ happens, if a policy is eliminated from $\Pi'$, then it is not an $\varepsilon/8-optimal$ policy of $M_t$. It further implies that $\mathbb{P}(\cE_j|\cE)=1$, i.e. $\cE\subset \cE_j$ for any $j\in\{1,2,\dots,|\cS|\}$. If a policy $\pi\in\Pi'_{j-1}$ satisfies that $\bar{V}^{\pi}_j\geq \bar{V}^{\text{max}}_j-\varepsilon/2$, then on $\cE\cap\cE_{j-1}$, we have that
\begin{align}
V^*_j - V^{\pi}_j &=V^*_j-V^{\pi^{j_1}}_j+V^{\pi^{j_1}}_j-\bar{V}^{\pi^{j_1}}_j+\bar{V}^{\pi^{j_1}}_j-\bar{V}^{\pi}_j+\bar{V}^{\pi}_j-V^{\pi}_j\\
&\leq \max_{\pi'\in\Pi'} V^{\pi'}_j+\varepsilon/8-V^{\pi^{j_1}}_j+V^{\pi^{j_1}}_j-\bar{V}^{\pi^{j_1}}_j+\bar{V}^{\pi^{j_1}}_j-\bar{V}^{\pi}_j+\bar{V}^{\pi}_j-V^{\pi}_j\\
&\leq \max_{\pi'\in\Pi'} \bar{V}^{\pi'}_j +3\varepsilon/8 -V^{\pi^{j_1}}_j+V^{\pi^{j_1}}_j-\bar{V}^{\pi^{j_1}}_j+\bar{V}^{\pi^{j_1}}_j-\bar{V}^{\pi}_j+\bar{V}^{\pi}_j-V^{\pi}_j\\
&=\bar{V}^{\pi^{j_1}}_j + 3\varepsilon/8 -V^{\pi^{j_1}}_j+V^{\pi^{j_1}}_j-\bar{V}^{\pi^{j_1}}_j+\bar{V}^{\pi^{j_1}}_j-\bar{V}^{\pi}_j+\bar{V}^{\pi}_j-V^{\pi}_j\\
&\leq 3\varepsilon/8 + \varepsilon/2 + \varepsilon/8 <\varepsilon
\end{align}
where the second inequality is due to the existence of an $\varepsilon/8$-optimal policy. This implies that if a policy $\pi$ can last till the end of scanning, then for any state $s_j, V^*_j-V^{\pi}_j\leq\varepsilon$, so it is $\varepsilon$-optimal. Thus, once $\cE$ happening, $\Pi'$ only contains $\varepsilon$-optimal policies for $M_t$.
\end{proof}
\end{lemma}

\begin{proposition}\label{prop:case2}
Given $\varepsilon>0$ and $\delta>0$, for Algorithm \ref{alg:case2}, it takes $\cO(\frac{m^2|\cS|^3\gamma^2}{(1-\gamma)^4\varepsilon^2}\log\big(\frac{(1+|\cS|)|\cS|m}{\delta}\big))$ samples each time to return an $\varepsilon$-optimal policy with probability at least $1-\delta$ for $M_t$.
\end{proposition}

\subsubsection{Case 3}
In Case 3, a similar approach is taken as in Case 1. We first classify the model of $M_t$ then gathering the corresponding samples for training. See the algorithm in Algorithm \ref{alg:case3}.

\begin{algorithm}
  \caption{Case 3 of Finite MDPs}
  \label{alg:case3}
  \begin{algorithmic}[1]
  \State \textbf{Input:} $\varepsilon>0, \delta>0, K=\lceil\frac{192}{\alpha^2}\log(|\cS||\cA|/\delta)\rceil, L = \lceil\frac{|\cS||\cA|}{(1-\gamma)^3\varepsilon^2}\log(\frac{|\cS||\cA|}{\delta})/ K\rceil , L^i=0 ~(i\in{1,2,\dots,m}), T>0, $ a list of MDPs $\cM'\gets\emptyset$, and a policy $\pi^0$ with $\pi^0(s_i)=a_1$. 
 \While {$t < T$}
  \For {$i\gets1; i\leq |\cS|; i\gets i+1$}\label{lin:part1start}
  \For {$j\gets1; j\leq |\cA|; j\gets j+1$}
 \State Take $K$ transition samples from $\cS\cO$ with input $(s_i,a_j)$;
 \State Record reward $R(s_i,a_j)$ and compute $\bar{P}_{i,j,w}= \frac{\text{\#the next state is }s_w}{K}, w\in{1,\dots,|\cS|}$.
 \EndFor
 \EndFor
 \State Define $\bar{M}_t:=(\cS,\cA,\bar{P},R,\gamma)$ and let $q\gets 0$.\label{lin:part1end}
 \For {$p\gets 1;p\leq |\cM'|; p\gets p+1$} \label{lin:part2start}
 \If {the $p$th model $M^p$ in $\cM'$ is $\alpha/8-$close to $\bar{M_t}$ }
 \State $q\gets p;$
 \If {$L^q<N$}
 \State $M^q\gets (\cS,\cA,\frac{L^q}{L^q+1}P^q+\frac{1}{L^q+1}\bar{P}, R, \gamma)$;
 \ElsIf {$L^q=L$}
 \State Compute an $\varepsilon$-optimal policy $\pi^q$ of $M^q$ with any planning algorithm
 \EndIf
 \State $L^q\gets L^q+1$;
 \EndIf
 \EndFor
 \If {$q=0$ and $|\cM'|<m$}
 \State add $\bar{M_t}$ as the last element of $\cM'$ and let $L^{|\cM'|}=1$.
 \EndIf\label{lin:part2end}
 \State \textbf{Output:} $\pi^q$.
\EndWhile
\end{algorithmic}
\end{algorithm}

The algorithm can be separated into three parts: part 1 is from Line \ref{lin:part1start} to \ref{lin:part1end} where we construct an empirical approximation of $M_t$; part 2 is from Line \ref{lin:part2start} to \ref{lin:part2end} where we search over the empirical model collection $\cM'$ and check if there exists an old approximation which also corresponds to $\cM(M_t)$ and then merge them, if not, depending on the size of $\cM'$, we create a new model; part 3 is to output a policy. Next, we show that when $t>...$, the output policy is $\varepsilon$-optimal for $M_t$ with probability at least $1-\delta$.

\begin{lemma}\label{lemma:case3part1}
In Algorithm \ref{alg:case3}, for each time step $t$, the empirical approximation $\bar{M_t}$ is $\alpha/8-$close to $\cM(M_t)$ with probability at least $1-\delta$.
\begin{proof}
Denote by $P$ the transition probability of $\cM(M_t)$. By Bernstein's Inequality, for any state-action pair $(s_i,a_j)$, we have that
\begin{align}
\mathbb{P}(\|\bar{P}(\cdot\vert s_i,a_j)-P(\cdot\vert s_i,a_j)\|_2\leq \alpha/8) \geq 1-(1+|\cS|)\exp\big(\frac{-K\alpha^2}{192}\big)\geq 1-\frac{\delta}{|\cS||\cA|}.
\end{align}
By taking the union bound over all state-action pairs, we have that $\bar{M_t}$ is $\alpha/8-$close to $\cM(M_t)$ with probability at least $1-\delta$.
\end{proof}
\end{lemma}

Since we only take $K=\lceil\frac{192}{\alpha^2}\log(|\cS||\cA|/\delta)\rceil$ samples each time. To calculate an $\varepsilon$-optimal policy for $\cM(M_t)$, by Theorem \ref{thm:mdp}, we need to accumulate at least $L = \lceil\frac{|\cS||\cA|}{(1-\gamma)^3\varepsilon^2}\log(\frac{|\cS||\cA|}{\delta})/ K\rceil$ times. Next, we calculate the probability of successful classification.

Define event $\cE_t:=\{\text{for all }t'\leq t, M_{t'} \text{ is correctly classified}\}$. Here, being correctly classified means that if $\cM(M_{t_1})=\cM(M_{t_2})$, then $\bar{M}_{t_1}$ and $\bar{M}_{t_2}$ are merged in the same group in $\cM'$.

\begin{lemma}\label{lemma:case3part2}
$\mathbb{P}(\cE_t)\geq 1-t\delta$.
\begin{proof}
Since the models in $\cM$ are $\alpha-$distant, as long as all $\bar{M_{t'}}$ are $\alpha/8-$close to $\cM(M_{t'})$, $\cE_t$ will happen. By Lemma \ref{lemma:case3part1} and taking the union bound, we have the desired result.
\end{proof}
\end{lemma}

Next, we show that when $t$ is great enough, all models can collect samples for $N$ times with high probability. We first define a special integer $$T_0(\delta):=\min\{x\in \NN \vert x\geq \frac{L-1}{\log(1/(1-p_{\min}))}\log x + \frac{|\log(\delta/cm)|}{\log(1/(1-p_{\min}))}\}$$.
\begin{lemma}\label{lemma:case3part3}
When $t>T_0$, all models in $\cM$ have appeared for at least $L$ times with probability at least $1-\delta$.
\begin{proof}
Denote by $\cE^i$ the event $\{$ when $t>T_0$, $M^i\in\cM$ have appeared for at least $L$ times$\}$. Then for any $i\in{1,2,\dots,m}$, we have that
\begin{align}\label{eq:probET}
\mathbb{P}[\cE^i] &= 1-\sum_{n=0}^{L-1} {t\choose n}(p^i)^n(1-p^i)^{T_0-n}\geq 1-\sum_{n=0}^{N-1} {T_0\choose n}(p^i)^n(1-p^i)^{T_0-n}\\
&\geq 1-\sum_{n=0}^{L-1} \big(\frac{eT_0}{n}\big)^n (p^i)^n(1-p^i)^{T_0-n}\\
&\geq 1-(1-p_{\text{min}})^{T_0} \Big(\sum_{n=0}^{N-1} \big(\frac{ep_{\text{max}}T_0}{n(1-p_{\text{max}})}\big)^n\Big)\\
&\geq 1-c(1-p_{\text{min}})^{T_0} T_0^{N-1},\label{eq:probETLAST}
\end{align}
where $c:=L\Big(\frac{ep_{\text{max}}}{(1-p_{\text{max}})(L-1)}\Big)^{L-1}$. Thus, by the definition of $T_0$, we have that $\mathbb{P}[\cE^i]>1-\delta/m$ for any $i\in\{1,2,\dots,m\}$. Taking the union bound over all $\cE^i$, we get the desired result.
\end{proof}
\end{lemma}
\begin{proposition}\label{prop:case2}
Given $\varepsilon\in(0,1)$ and $\delta'\in(0,1)$, Let $\delta=\delta'/(T_0+1)$ in Algorithm \ref{alg:case3}, for time step $t>T_0$, the output policy is $\varepsilon$-optimal for $M_t$ with probability at least $1-\delta'$.
\begin{proof}
By Lemma \ref{lemma:case3part3}, we know that when $t>T_0$, with probability at least $1-\delta$, all models will appear at least $L$ times. By Lemma \ref{lemma:case3part2}, we know that with probability at least $1-T_0\delta$, all samples are classified correctly. Combining these two results, when $t>T_0$, with probability at least $1-(T_0+1)\delta = 1-\delta'$, the returned policy is $\varepsilon$-optimal for $M_t$.
\end{proof}
\end{proposition}

We compare the one-time sample complexity of three cases with the general setting:
\begin{center}
\begin{tabular}{  c | c }
\hline
\textbf{Setting} & \textbf{$N_t$}\\
\hline
General & $\cO(\frac{|\cS||\cA|}{(1-\gamma)^3\varepsilon^2}\log(|\cS||\cA|/\delta))$\\
\hline
Case 1 of Finite MDPs & $\cO(\frac{m}{\alpha^2}\log(m/\delta))$\\
\hline
Case 2 of Finite MDPs & $\cO(\frac{m^2|\cS|^3\gamma^2}{(1-\gamma)^4\varepsilon^2}\log\big((1+|\cS|)|\cS|m/\delta\big))$\\
\hline
Case 3 of Finite MDPs & $\cO(\frac{|\cS||\cA|}{\alpha^2}\log(|\cS||\cA|/\delta))$\\
\hline
\end{tabular}
\end{center}}

\bibliographystyle{apalike}
\bibliography{references}  

\end{document}

\cut{\section{Time-varying Reinforcement Learning under New Assumptions}
In the previous section, we showed a sample complexity lower bound for any $(\varepsilon, \delta)$-correct RL algorithm to learn $M^t$, which indeed meets an upper bound of independent learning (i.e. without any information of previous models) as shown in \citealt{azar2013minimax}. This implies that when the accuracy requirement is high, i.e. when $\varepsilon$ is small enough, the information from an approximate model does not help reduce the sample complexity. There are several possible reasons:
\begin{itemize}
\item mean value being close or the transition probability being close does not imply a smaller range of transition samples. Thus, we are unable to have a tighter control on sample values.
\item Such a one-time information is not enough if one wants to try variance reduction technique, since the baseline value cannot be improved.
\end{itemize}

Due to the above reasons, we need to propose further assumptions on the time-varying RL problem such that the previous information can be effectively utilized and the sample complexity upper bound can be smaller than general scenarios.

\subsection{Finite MDPs}
The first assumption we propose is that all $M_t$ belongs to a finite set of MDPs $\cM:=\{M^1,M_2,\dots,M^m\}$, where $M^i:=(\cS,\cA,P^i,R^i,\gamma)$. Besides that $M^i$ and $M^j$ are $\beta-$close, we also have that any two MDPs in $\cM$ are $\alpha-$distant, where $\alpha>0$. At each time step $t$, $M_t=M^i$ with probability $p^i_t$, where $0<p_{\text{min}}\leq p^i_t\leq p_{\text{max}}<1$. Compared with the general setting, this new assumption provides more usable information and new ways of learning.
Depending on how much information we know about $\cM$, we separate the analysis in three cases:
\begin{enumerate}
\item For each $M^i$, we know $P^i$ and $R^i$.
\item For each $M^i$, we know an $\varepsilon/8$-optimal deterministic policy $\pi^i$;
\item We know nothing about the models except that they are $\alpha-$distant.
\end{enumerate}
Given $M_t$, we denote by $\cM(M_t)$ the model $M_t$ is inside of $\cM$.
\subsubsection{Case 1}
In this case, we only need to identify $\cM(M_t)$, then return a corresponding $\varepsilon$-optimal policy. Since any two models are $\alpha-$distant, given $M^i$ and $M^j$, there exists a state-action-state tuple $(s_{(i,j)}, a_{(i,j)}, s'_{(i,j)})$ such that
\begin{equation}\label{eq:alpha}
|P^i(s'_{(i,j)}\vert s_{(i,j)}, a_{(i,j)})-P^j(s'_{(i,j)}\vert s_{(i,j)}, a_{(i,j)})|\geq\alpha.
\end{equation}
Based on this, we can call $\cS\cO$ with input $(s_{(i,j)}, a_{(i,j)})$ enough times to construct an empirical probability and use concentration inequality to rule out wrong candidates. Specifically, given any state-action-state pair $(s,a,s')$, if calling $\cS\cO$ for $K$ times with input $(s,a)$, the empirical probability is defined as:
\begin{equation}\label{eq:Pbar}
\bar{P}(s' \vert s, a):=\frac{\text{\# getting } s' \text { as next state when calling } \cS\cO(s, a)}{K}.
\end{equation}
The Algorithm is presented in Algorithm \ref{alg:case1}.

\vspace*{.5cm}
\begin{algorithm}
  \caption{Case 1 of Finite MDPs}
  \label{alg:case1}
  \begin{algorithmic}[1]
  \State \textbf{Input:} $\cM, \varepsilon>0, \delta>0, K=\lceil\frac{32}{\alpha^2}\log(m/\delta)\rceil, T>0$.
  \State For each $M^i$, compute an $\varepsilon$-optimal policy $\pi^i$ with any planning algorithm
  \State Select $m(m-1)/2$ tuples $(s_{(i,j)}, a_{(i,j)}, s'_{(i,j)})_{1\leq i<j\leq m}$ such that \eqref{eq:alpha} holds.
  \While {$t < T$}
 \State $i\gets 1;~j\gets 2$.
 \While {$j\leq m$}
 \State Take $K$ transition samples from $\cS\cO$ with input $(s_{(i,j)},a_{(i,j)})$;
 \State Compute $\bar{P}(s'_{(i,j)} \vert s_{(i,j)}, a_{(i,j)})$ as in \eqref{eq:Pbar}.
 \If {$|\bar{P}(s'_{(i,j)} \vert s_{(i,j)}, a_{(i,j)})-P^i(s'_{(i,j)} \vert s_{(i,j)}, a_{(i,j)})|\leq \alpha/8$}
 \State $j=j+1$;
 \ElsIf {$|\bar{P}(s'_{(i,j)} \vert s_{(i,j)}, a_{(i,j)})-P^j(s'_{(i,j)} \vert s_{(i,j)}, a_{(i,j)})|\leq \alpha/8$}
 \State $i=j; ~j=j+1$;
 \Else
 \State $i=j+1$;~ $j=j+2;$
 \EndIf
 \EndWhile
 \State \textbf{Output:} $\pi^i$ for $M_t$. \label{lin:case1output}
\EndWhile
\end{algorithmic}
\end{algorithm}

Next, we show the correctness of Algorithm \ref{alg:case1} and calculate the sample complexity.

\begin{lemma}\label{lemma:case1}
Given $\varepsilon>0$ and $\delta>0$, in Algorithm \ref{alg:case1}, for each time step $t$, the output policy is $\varepsilon$-optimal for $M_t$ with probability at least $1-\delta$.
\begin{proof}
We only need to show that at Line \ref{lin:case1output} of Algorithm \ref{alg:case1}, index $i$ labels the correct model for $M_t$.
Denote by $P_t$ the transition probability of $M_t$. By Hoeffding's inequality,
\begin{align}\label{eq:hoeffdingcase1}
\mathbb{P}(|\bar{P}(s'_{(i,j)} \vert s_{(i,j)}, a_{(i,j)})-P_t(s'_{(i,j)} \vert s_{(i,j)}, a_{(i,j)})|\leq \frac{\alpha}{8})\geq 1-\exp(-2K\frac{\alpha^2}{64})\geq 1-\frac{\delta}{n}
\end{align}
Due to Equations \eqref{eq:alpha} and \eqref{eq:hoeffdingcase1}, at least one model among $M^i$ and $M^j$ will be at least $\alpha/4-$distant from $\bar{P}(s'_{(i,j)} \vert s_{(i,j)}, a_{(i,j)})$, which makes it not a candidate with high probability. During the inner loop, at each iteration, we pass at least one model which is not $\cM(M_t)$ with probability at least $1-\delta/n$. In total, by union bound, the final left model, which is labelled by $i$, is the correct model for $M_t$ with probability at least $1-\delta$
\end{proof}
\end{lemma}

\begin{proposition}\label{prop:case1}
Given $\varepsilon>0$ and $\delta>0$, for case 1 of Finite MDPs, it takes $\cO(\frac{m}{\alpha^2}\log(m/\delta))$ samples each time to return an $\varepsilon$-optimal policy with probability at least $1-\delta$ for $M_t$.
\begin{proof}
By Lemma \ref{lemma:case1}, the sample complexity is $\cO((m-1)K)=\cO(\frac{m}{\alpha^2}\log(m/\delta))$.
\end{proof}
\end{proposition}

\subsubsection{Case 2}
For case 2, identifying $\cM(M_t)$ is impossible without model knowledge. Thus, we take a different approach by directly evaluating all policies under $M_t$, then select the one with the highest value. Given a deterministic policy $\pi$, the value vector $V^{\pi}\in \RR^{|\cS|}$ satisfies the following Bellman equation:
\begin{align}\label{eq:bellman}
V^{\pi}=R^{\pi}+\gamma P^{\pi}V^{\pi},
\end{align}
where $R^{\pi}_j=R(s_j,\pi(s_j))$ and $P^{\pi}_{jw}=P(s_w\vert s_j,\pi(s_j))$. Due to the Equation \eqref{eq:bellman}, we first approximate $P^{\pi}$ by sample average, then solve the corresponding linear system to obtain $V^{\pi}$. The algorithm is summarized in Algorithm \ref{alg:case2} and the complexity is shown in Prop. \ref{prop:case2}.

\begin{algorithm}[h]
  \caption{Case 2 of Finite MDPs}
  \label{alg:case2}
  \begin{algorithmic}[1]
  \State \textbf{Input:} $ \varepsilon>0, \delta>0, \Pi:=\{\pi^i\}_{i=1}^m, K=\lceil\frac{192m|\cS|^2\gamma^2}{(1-\gamma)^4\varepsilon^2}\log\big(\frac{(1+|\cS|)|\cS|^2}{\delta}\big)\rceil, T>0$.
  \While {$t < T$}
  \State $\Pi'\gets\Pi$;
  \For {$i=1;i\leq m ;i=i+1$}
  \For {$j=1;j\leq|\cS|;j=j+1$}
 \State Take $K$ transition samples from $\cS\cO$ with input $(s_j,\pi^i(s_j))$;
 \State Record reward $R(s_j,\pi^i(s_j))$ and compute $\bar{P}^{\pi^i}_{j,w}= \frac{\text{\# next state is }s_w}{K}$;
 \EndFor
 \State Compute $\bar{V}^{\pi^i} := (I-\gamma\bar{P}^{\pi^i})^{-1}R^{\pi^i};$
 \EndFor
 \For {$j=1;j\leq |\cS|;j=j+1$}\label{lin:forcase2}
 \State $\bar{V}^{\text{max}}_j\gets\max_{\pi\in\Pi'} \bar{V}^{\pi}_j$;
 \State Check all policies in $\Pi'$ and eliminate $\pi\in\Pi'$ if $\bar{V}^{\pi}_j<\bar{V}^{\text{max}}_j-\varepsilon/2$;
 \EndFor\label{lin:endforcase2}
 \State \textbf{Output:} Any policy from $\Pi'$.
\EndWhile
\end{algorithmic}
\end{algorithm}
\begin{lemma}\label{lemma:nbound}
Given two transition matrix $P^1$ and $P^2$, if $\|P^1-P^2\|_2\leq\epsilon$, then $\|P^1^n-P^2^n\|_2\leq n\epsilon$.
\begin{proof}
where the second inequality comes from
\begin{align}
\|P^1^n-P^2^n\|_2& \leq\|P^1-P^2\|_2\|P^1^{n-1}+P^1^{n-2}P^2+\dots+P^2^{n-1}\|_2\\
&\leq \|P^1-P^2\|_2 \Big(\sum_{j=0}^{n-1} \| (P^1^jP_2^{n-1-j}\|_2\Big)\\
&\leq \|P^1-P^2\|_2 \Big(\sum_{j=0}^{n-1} \| P^1\|^j_2\|P^2\|^{n-1-j}_2\Big)\leq n\epsilon.
\end{align}
\end{proof}
\end{lemma}

\begin{lemma}\label{lemma:case2}
Given $\varepsilon\in(0,1)$, $\delta\in(0,1)$, and $m$ $\varepsilon/8$-optimal policies for each $M^i$ respectively, in Algorithm \ref{alg:case2}, for each time step $t$, the output policy is $\varepsilon$-optimal for $M_t$ with probability at least $1-\delta$.
\begin{proof}
By Bernstein's inequality, we know that with $K:=\lceil\frac{192m|\cS|^2\gamma^2}{(1-\gamma)^4\varepsilon^2}\log\big(\frac{(1+|\cS|)|\cS|m}{\delta}\big)\rceil$ transition samples to approximate the $j$th row of $P^{\pi^i}$, the following holds
\begin{align}
\mathbb{P}(\|\bar P^{\pi^i}(j,\cdot)-P^{\pi^i}(j,\cdot)\|_2\leq \frac{(1-\gamma)^2}{8\gamma|\cS|}\varepsilon)\geq 1- (1+|\cS|)\exp\Big(\frac{-K(1-\gamma)^4\varepsilon^2}{192\gamma^2|\cS|^2}\Big)\geq 1-\delta/(|\cS|m).
\end{align}
By taking union bound over all rows, we further have that
\begin{align}
\mathbb{P}(\|\bar P^{\pi^i}-P^{\pi^i}\|_2 \leq \varepsilon':=\frac{(1-\gamma)^2}{8\gamma\sqrt{|\cS|}}\varepsilon) \geq 1- \delta/m.
\end{align}
Denote by $\bar{V}^{\pi^i}:= (I-\gamma\bar{P}^{\pi^i})^{-1}R^{\pi^i}$, then by Lemma \ref{lemma:nbound},
\begin{align}\label{eq:vbound}
\|\bar{V}^{\pi^i}-V^{\pi^i}\|_{\infty}&= \|(I-\gamma P^{\pi^i})^{-1}R^{\pi^i}-(I-\gamma\bar{P}^{\pi^i})^{-1}R^{\pi^i}\|_{\infty}\\
&=\|(I+\gamma\bar{P}^{\pi^i}+(\gamma\bar{P}^{\pi^i})^2+\dots)R^{\pi^i}-(I+\gamma P^{\pi^i} + (\gamma P^{\pi^i})^2+\dots)R^{\pi^i}\|_{\infty}\\
&=\|\gamma(\bar{P}^{\pi^i}-P^{\pi^i})R^{\pi^i} + \gamma^2((\bar{P}^{\pi^i})^2-(P^{\pi^i})^2)R^{\pi^i} + \dots\|_{\infty}\\
&\leq \|\gamma(\bar{P}^{\pi^i}-P^{\pi^i})R^{\pi^i}\|_{\infty} + \|\gamma^2((\bar{P}^{\pi^i})^2-(P^{\pi^i})^2)R^{\pi^i}\|_{\infty} + \dots\\
&\leq \Big(\gamma\|\bar{P}^{\pi^i}-P^{\pi^i}\|_2+\gamma^2\|(\bar{P}^{\pi^i})^2-(P^{\pi^i})^2\|_{2}+\dots\Big)\|R^{\pi^i}\|_2\\
&\leq (\gamma\varepsilon' + 2\gamma^2\varepsilon' + \dots + n\gamma^n\varepsilon' \dots) \sqrt{|\cS|}\\
&\leq \frac{\sqrt{|\cS|}\gamma \varepsilon'}{(1-\gamma)^2}=\varepsilon/8 \text{ with prob. at least $1-\delta/m$}
\end{align}
Next, we show that after Line \ref{lin:endforcase2}, $\Pi'$ only contains $\varepsilon$-optimal policies for $M_t$. Define the following events:
 \begin{align}
 \cE &:=\{ \text{for all } i\in \{1,2,\dots,m\}, \|\bar{V}^{\pi^i}-V^{\pi^i}\|_{\infty} \leq \varepsilon/8\}\\
 \cE_j &:=\{\text{After } j\text{th iteration of Line \ref{lin:forcase2}}, \Pi' \text{ contains  an } \varepsilon/8-\text{optimal policy for } M_t\}.
  \end{align}
By Eqn \eqref{eq:vbound}, $\mathbb{P}(\cE)\geq 1-\delta$. Suppose right after the $j-1$th iteration of Line \ref{lin:forcase2}, $\Pi_{j-1}'=\{\pi^{j_1},\pi^{j_2},\dots, \pi^{j_k}\}$, where $k\leq m$, and that $\bar{V}^{\pi^{j_1}}_j = \bar{V}^{\text{max}}_j$. If a policy $\pi\in\Pi'_{j-1}$ satisfies that $\bar{V}^{\pi}_j<\bar{V}^{\text{max}}_j-\varepsilon/2$, then on $\cE$, we have that
\begin{align}
V^*_j-V^{\pi}_j&=V^*_j-V^{\pi^{j_1}}_j+V^{\pi^{j_1}}_j-\bar{V}^{\pi^{j_1}}_j+\bar{V}^{\pi^{j_1}}_j-\bar{V}^{\pi}_j+\bar{V}^{\pi}_j-V^{\pi}_j\\
&> 0 -\varepsilon/8 + \varepsilon/2 - \varepsilon/8 >\varepsilon/8.
\end{align}
The above equation tells that when $\cE$ happens, if a policy is eliminated from $\Pi'$, then it is not an $\varepsilon/8-optimal$ policy of $M_t$. It further implies that $\mathbb{P}(\cE_j|\cE)=1$, i.e. $\cE\subset \cE_j$ for any $j\in\{1,2,\dots,|\cS|\}$. If a policy $\pi\in\Pi'_{j-1}$ satisfies that $\bar{V}^{\pi}_j\geq \bar{V}^{\text{max}}_j-\varepsilon/2$, then on $\cE\cap\cE_{j-1}$, we have that
\begin{align}
V^*_j - V^{\pi}_j &=V^*_j-V^{\pi^{j_1}}_j+V^{\pi^{j_1}}_j-\bar{V}^{\pi^{j_1}}_j+\bar{V}^{\pi^{j_1}}_j-\bar{V}^{\pi}_j+\bar{V}^{\pi}_j-V^{\pi}_j\\
&\leq \max_{\pi'\in\Pi'} V^{\pi'}_j+\varepsilon/8-V^{\pi^{j_1}}_j+V^{\pi^{j_1}}_j-\bar{V}^{\pi^{j_1}}_j+\bar{V}^{\pi^{j_1}}_j-\bar{V}^{\pi}_j+\bar{V}^{\pi}_j-V^{\pi}_j\\
&\leq \max_{\pi'\in\Pi'} \bar{V}^{\pi'}_j +3\varepsilon/8 -V^{\pi^{j_1}}_j+V^{\pi^{j_1}}_j-\bar{V}^{\pi^{j_1}}_j+\bar{V}^{\pi^{j_1}}_j-\bar{V}^{\pi}_j+\bar{V}^{\pi}_j-V^{\pi}_j\\
&=\bar{V}^{\pi^{j_1}}_j + 3\varepsilon/8 -V^{\pi^{j_1}}_j+V^{\pi^{j_1}}_j-\bar{V}^{\pi^{j_1}}_j+\bar{V}^{\pi^{j_1}}_j-\bar{V}^{\pi}_j+\bar{V}^{\pi}_j-V^{\pi}_j\\
&\leq 3\varepsilon/8 + \varepsilon/2 + \varepsilon/8 <\varepsilon
\end{align}
where the second inequality is due to the existence of an $\varepsilon/8$-optimal policy. This implies that if a policy $\pi$ can last till the end of scanning, then for any state $s_j, V^*_j-V^{\pi}_j\leq\varepsilon$, so it is $\varepsilon$-optimal. Thus, once $\cE$ happening, $\Pi'$ only contains $\varepsilon$-optimal policies for $M_t$.
\end{proof}
\end{lemma}

\begin{proposition}\label{prop:case2}
Given $\varepsilon>0$ and $\delta>0$, for Algorithm \ref{alg:case2}, it takes $\cO(\frac{m^2|\cS|^3\gamma^2}{(1-\gamma)^4\varepsilon^2}\log\big(\frac{(1+|\cS|)|\cS|m}{\delta}\big))$ samples each time to return an $\varepsilon$-optimal policy with probability at least $1-\delta$ for $M_t$.
\end{proposition}

\subsubsection{Case 3}
In Case 3, a similar approach is taken as in Case 1. We first classify the model of $M_t$ then gathering the corresponding samples for training. See the algorithm in Algorithm \ref{alg:case3}.

\begin{algorithm}
  \caption{Case 3 of Finite MDPs}
  \label{alg:case3}
  \begin{algorithmic}[1]
  \State \textbf{Input:} $\varepsilon>0, \delta>0, K=\lceil\frac{192}{\alpha^2}\log(|\cS||\cA|/\delta)\rceil, L = \lceil\frac{|\cS||\cA|}{(1-\gamma)^3\varepsilon^2}\log(\frac{|\cS||\cA|}{\delta})/ K\rceil , L^i=0 ~(i\in{1,2,\dots,m}), T>0, $ a list of MDPs $\cM'\gets\emptyset$, and a policy $\pi^0$ with $\pi^0(s_i)=a_1$. 
 \While {$t < T$}
  \For {$i\gets1; i\leq |\cS|; i\gets i+1$}\label{lin:part1start}
  \For {$j\gets1; j\leq |\cA|; j\gets j+1$}
 \State Take $K$ transition samples from $\cS\cO$ with input $(s_i,a_j)$;
 \State Record reward $R(s_i,a_j)$ and compute $\bar{P}_{i,j,w}= \frac{\text{\#the next state is }s_w}{K}, w\in{1,\dots,|\cS|}$.
 \EndFor
 \EndFor
 \State Define $\bar{M}_t:=(\cS,\cA,\bar{P},R,\gamma)$ and let $q\gets 0$.\label{lin:part1end}
 \For {$p\gets 1;p\leq |\cM'|; p\gets p+1$} \label{lin:part2start}
 \If {the $p$th model $M^p$ in $\cM'$ is $\alpha/8-$close to $\bar{M_t}$ }
 \State $q\gets p;$
 \If {$L^q<N$}
 \State $M^q\gets (\cS,\cA,\frac{L^q}{L^q+1}P^q+\frac{1}{L^q+1}\bar{P}, R, \gamma)$;
 \ElsIf {$L^q=L$}
 \State Compute an $\varepsilon$-optimal policy $\pi^q$ of $M^q$ with any planning algorithm
 \EndIf
 \State $L^q\gets L^q+1$;
 \EndIf
 \EndFor
 \If {$q=0$ and $|\cM'|<m$}
 \State add $\bar{M_t}$ as the last element of $\cM'$ and let $L^{|\cM'|}=1$.
 \EndIf\label{lin:part2end}
 \State \textbf{Output:} $\pi^q$.
\EndWhile
\end{algorithmic}
\end{algorithm}

The algorithm can be separated into three parts: part 1 is from Line \ref{lin:part1start} to \ref{lin:part1end} where we construct an empirical approximation of $M_t$; part 2 is from Line \ref{lin:part2start} to \ref{lin:part2end} where we search over the empirical model collection $\cM'$ and check if there exists an old approximation which also corresponds to $\cM(M_t)$ and then merge them, if not, depending on the size of $\cM'$, we create a new model; part 3 is to output a policy. Next, we show that when $t>...$, the output policy is $\varepsilon$-optimal for $M_t$ with probability at least $1-\delta$.

\begin{lemma}\label{lemma:case3part1}
In Algorithm \ref{alg:case3}, for each time step $t$, the empirical approximation $\bar{M_t}$ is $\alpha/8-$close to $\cM(M_t)$ with probability at least $1-\delta$.
\begin{proof}
Denote by $P$ the transition probability of $\cM(M_t)$. By Bernstein's Inequality, for any state-action pair $(s_i,a_j)$, we have that
\begin{align}
\mathbb{P}(\|\bar{P}(\cdot\vert s_i,a_j)-P(\cdot\vert s_i,a_j)\|_2\leq \alpha/8) \geq 1-(1+|\cS|)\exp\big(\frac{-K\alpha^2}{192}\big)\geq 1-\frac{\delta}{|\cS||\cA|}.
\end{align}
By taking the union bound over all state-action pairs, we have that $\bar{M_t}$ is $\alpha/8-$close to $\cM(M_t)$ with probability at least $1-\delta$.
\end{proof}
\end{lemma}

Since we only take $K=\lceil\frac{192}{\alpha^2}\log(|\cS||\cA|/\delta)\rceil$ samples each time. To calculate an $\varepsilon$-optimal policy for $\cM(M_t)$, by Theorem \ref{thm:mdp}, we need to accumulate at least $L = \lceil\frac{|\cS||\cA|}{(1-\gamma)^3\varepsilon^2}\log(\frac{|\cS||\cA|}{\delta})/ K\rceil$ times. Next, we calculate the probability of successful classification.

Define event $\cE_t:=\{\text{for all }t'\leq t, M_{t'} \text{ is correctly classified}\}$. Here, being correctly classified means that if $\cM(M_{t_1})=\cM(M_{t_2})$, then $\bar{M}_{t_1}$ and $\bar{M}_{t_2}$ are merged in the same group in $\cM'$.

\begin{lemma}\label{lemma:case3part2}
$\mathbb{P}(\cE_t)\geq 1-t\delta$.
\begin{proof}
Since the models in $\cM$ are $\alpha-$distant, as long as all $\bar{M_{t'}}$ are $\alpha/8-$close to $\cM(M_{t'})$, $\cE_t$ will happen. By Lemma \ref{lemma:case3part1} and taking the union bound, we have the desired result.
\end{proof}
\end{lemma}

Next, we show that when $t$ is great enough, all models can collect samples for $N$ times with high probability. We first define a special integer $$T_0(\delta):=\min\{x\in \NN \vert x\geq \frac{L-1}{\log(1/(1-p_{\min}))}\log x + \frac{|\log(\delta/cm)|}{\log(1/(1-p_{\min}))}\}$$.
\begin{lemma}\label{lemma:case3part3}
When $t>T_0$, all models in $\cM$ have appeared for at least $L$ times with probability at least $1-\delta$.
\begin{proof}
Denote by $\cE^i$ the event $\{$ when $t>T_0$, $M^i\in\cM$ have appeared for at least $L$ times$\}$. Then for any $i\in{1,2,\dots,m}$, we have that
\begin{align}\label{eq:probET}
\mathbb{P}[\cE^i] &= 1-\sum_{n=0}^{L-1} {t\choose n}(p^i)^n(1-p^i)^{T_0-n}\geq 1-\sum_{n=0}^{N-1} {T_0\choose n}(p^i)^n(1-p^i)^{T_0-n}\\
&\geq 1-\sum_{n=0}^{L-1} \big(\frac{eT_0}{n}\big)^n (p^i)^n(1-p^i)^{T_0-n}\\
&\geq 1-(1-p_{\text{min}})^{T_0} \Big(\sum_{n=0}^{N-1} \big(\frac{ep_{\text{max}}T_0}{n(1-p_{\text{max}})}\big)^n\Big)\\
&\geq 1-c(1-p_{\text{min}})^{T_0} T_0^{N-1},\label{eq:probETLAST}
\end{align}
where $c:=L\Big(\frac{ep_{\text{max}}}{(1-p_{\text{max}})(L-1)}\Big)^{L-1}$. Thus, by the definition of $T_0$, we have that $\mathbb{P}[\cE^i]>1-\delta/m$ for any $i\in\{1,2,\dots,m\}$. Taking the union bound over all $\cE^i$, we get the desired result.
\end{proof}
\end{lemma}
\begin{proposition}\label{prop:case2}
Given $\varepsilon\in(0,1)$ and $\delta'\in(0,1)$, Let $\delta=\delta'/(T_0+1)$ in Algorithm \ref{alg:case3}, for time step $t>T_0$, the output policy is $\varepsilon$-optimal for $M_t$ with probability at least $1-\delta'$.
\begin{proof}
By Lemma \ref{lemma:case3part3}, we know that when $t>T_0$, with probability at least $1-\delta$, all models will appear at least $L$ times. By Lemma \ref{lemma:case3part2}, we know that with probability at least $1-T_0\delta$, all samples are classified correctly. Combining these two results, when $t>T_0$, with probability at least $1-(T_0+1)\delta = 1-\delta'$, the returned policy is $\varepsilon$-optimal for $M_t$.
\end{proof}
\end{proposition}

We compare the one-time sample complexity of three cases with the general setting:
\begin{center}
\begin{tabular}{  c | c }
\hline
\textbf{Setting} & \textbf{$N_t$}\\
\hline
General & $\cO(\frac{|\cS||\cA|}{(1-\gamma)^3\varepsilon^2}\log(|\cS||\cA|/\delta))$\\
\hline
Case 1 of Finite MDPs & $\cO(\frac{m}{\alpha^2}\log(m/\delta))$\\
\hline
Case 2 of Finite MDPs & $\cO(\frac{m^2|\cS|^3\gamma^2}{(1-\gamma)^4\varepsilon^2}\log\big((1+|\cS|)|\cS|m/\delta\big))$\\
\hline
Case 3 of Finite MDPs & $\cO(\frac{|\cS||\cA|}{\alpha^2}\log(|\cS||\cA|/\delta))$\\
\hline
\end{tabular}
\end{center}}}

\bibliographystyle{apalike}
\bibliography{references}  
\newpage
\appendix
\section{Proofs of the Main Results}\label{app:proof}
\subsection{Proofs of the Upper Bound}
We first prove the upper bound result in Theorem \ref{thm:main_upper}. 
\cut{We update the definition of $\cA^s_{\cM}(c)$ in Equation \eqref{eq:As} as
\begin{align}\label{eq:As_new}
    \cA^s_{\cM}(c):=\begin{cases}
    \{a~|~\max_{a'} Q_{\cM}^*(s,a') - Q_{\cM}^*(s,a) < c\}, &c>0;\\
    \argmax_a Q^*_{\cM}(s,a), &c\leq 0.
    \end{cases}
\end{align}
Then for any value of $c$, $\cA^s_{\cM}(c)$ is well-defined and always contains all optimal actions for state $s$ in $\cM$. Now we prove the key lemmas for the upper bound result.}
\begin{proof}[Proof of Lemma \ref{lemma:Qbound}]
Since $R(s,a,s')\in[0,1]$, the bound $1/(1-\gamma)$ is a direct result. For any policy $\pi$, denote by $Q^{\pi}_0$ and $Q^{\pi}$ the action-value functions and $V^{\pi}_0$ and $V^{\pi}$ the state-value functions of running $\pi$ in $\cM_0$ and $\cM$, respectively. Then, by definition, for every $(s,a)\in\cS\times\cA$,
\begin{align}
    \big\vert Q_0^{\pi}(s,a)-Q^{\pi}(s,a)\big\vert &= \big\vert r_0(s,a) + \gamma p_0(\cdot|s,a)^\top V_0^{\pi} - r(s,a)-\gamma p(\cdot|s,a)^\top V^{\pi}\big\vert\\
    &\leq \big\vert r_0(s,a)-r(s,a)\big\vert + \gamma \big\vert p_0(\cdot|s,a)^{\top} V_0^{\pi}-p(\cdot|s,a)^\top V_0^{\pi}\big\vert \\
    &\quad + \gamma\cdot \big\vert p(\cdot|s,a)^{\top} V_0^{\pi}-p(\cdot|s,a)^\top V^{\pi}\big\vert\\
    &\leq \beta + \gamma\cdot \big\|p_0(\cdot|s,a)-p(\cdot|s,a)\big\|_1\cdot \big\|V_0^{\pi}\big\|_{\infty} + \gamma \cdot \big\|V_0^{\pi}-V^{\pi}\big\|_{\infty}\\
    &\leq \beta + \gamma\beta/(1-\gamma)+\gamma\big\|Q_0^{\pi}-Q^{\pi}\big\|_{\infty}.
\end{align}
Thus, $\big\|Q^{\pi}_0-Q^{\pi}\big\|_\infty\leq \beta/(1-\gamma)+\gamma\cdot \big\|Q^{\pi}_0-Q^{\pi}\big\|_\infty$ and $\big\|Q^{\pi}_0-Q^{\pi}\big\|_\infty\leq \beta/(1-\gamma)^2$. Denote by $\pi^*_0$ an optimal policy of $\cM_0$ and $\pi^*$ an optimal policy of $\cM$. Then
\begin{align}
    Q^{\pi^*}_0-Q^{\pi^*}\leq Q^*_0-Q^*\leq Q^{\pi^*_0}_0-Q^{\pi^*_0}.
\end{align}
Thus, $\big\|Q^*_0-Q^*\big\|_{\infty}\leq \max\Big\{~\big\|Q^{\pi^*}_0-Q^{\pi^*}\big\|_{\infty},~ \big\|Q^{\pi^*_0}_0-Q^{\pi^*_0}\big\|_{\infty}~\Big\}\leq \min\big\{\frac{1}{1-\gamma}, \frac{\beta}{(1-\gamma)^2}\big\}$.
\end{proof}

\begin{proof}[Proof of Lemma \ref{lemma:epsQ}]
Since $\pi(s)=a^s$ for all $s\in\cS$, we have
\begin{align}
V^*(s)-V^{\pi}(s) &= \max_{a\in\cA^s}Q^*(s,a) - Q^{\pi}(s,a^s)\\
&= \max_{a\in\cA^s}Q^*(s,a)-Q^*(s,a^s) + Q^*(s,a^s) - Q^{\pi}(s,a^s)\\
&\leq \epsilon(1-\gamma) + \gamma\sum_{s'\in\cS}p(s'|s,a^s)(V^*(s')-V^{\pi}(s'))\\
&\leq \epsilon(1-\gamma) + \gamma\cdot \epsilon(1-\gamma) + \dots\leq \sum_{n=0}^\infty\epsilon(1-\gamma)\gamma^n=\epsilon,
\end{align}
where the last line is an induction step. Note that the above inequality holds for all states. Thus, $\pi$ is an $\epsilon$-optimal policy for $\cM$.
\end{proof}
\begin{proof}[Proof of Lemma \ref{lemma:aselect}]
Following Lemma \ref{lemma:epsQ}, for every state $s$ we need to find an action $a^s\in\cA^s$ such that \begin{align}\label{eq:Qcondition}
    Q^*(s,a^s)\geq \max_a Q^*(s,a)-\epsilon(1-\gamma).
\end{align} 
Denote by $c_0:=\min\{1/(1-\gamma), \beta/(1-\gamma)^2\}$. If $$\max_{a\in\cA^s}Q^*(s,a)-\epsilon(1-\gamma)\leq \max_{a'\in\cA^s}Q^*_0(s,a')-c_0,$$
then by Lemma \ref{lemma:Qbound}, any optimal action for state $s$ in $\cM_0$ satisfies Equation \eqref{eq:Qcondition}. Since $\cA^s_{\cM_0}(C(\beta,\epsilon))$ contains all these optimal actions, it is a valid searching set. If 
$$\max_{a\in\cA^s}Q^*(s,a)-\epsilon(1-\gamma) > \max_{a'\in\cA^s}Q^*_0(s,a')-c_0,$$
then denoting by $a^*$ an action such that $Q^*(s,a^*)= \max_{a}Q^*(s,a)$, we have
$$Q^*_0(s,a^*)\geq Q^*(s,a^*)-c_0>\max_{a'}Q^*_0(s,a')-2c_0+\epsilon(1-\gamma),$$
i.e., $a^*\in\cA^s_{\cM_0}(C(\beta,\epsilon))$. Combining the above results, there exists an $\epsilon(1-\gamma)$-optimal action for state $s$ in $\cM$ in the set $\cA^s_{\cM_0}(C(\beta,\epsilon))$.
\end{proof}

Combining the above Lemmas, we can easily obtain Theorem \ref{thm:main_upper}. 

\subsection{Proofs of the Lower Bound}
Next, we prove the lower bound result.
We first introduce the following lemma.
\begin{lemma}\label{lemma:L_k}
Following the notation in Lemma \ref{lemma:checkM}, we have $$\big\{l\in[L]~\big\vert ~\vert~p_0^k+\alpha_2^k-p_0(x_k,a_l)~\vert\leq\beta/2\big\}=[L_k], ~\forall k\in[K].$$
\end{lemma}
\begin{proof}[Proof of Lemma \ref{lemma:L_k}]
Since $p_0(x_k,a_1)\geq p_0(x_k,a_2)\geq\dots\geq p_0(x_k,a_l)$, the set $\big\{l\in[L]~\big\vert ~\vert~p_0^k+\alpha_2^k-p_0(x_k,a_l)~\vert\leq\beta/2\big\}$ contains consecutive integers. Now we only need to show that $l=1$ is contained in the set. By definition of $p_0^k$, it holds that $ p_0(x_k,a_1)-\beta/2\leq p_0^k<p_0(x_k,a_1)$. When $\varepsilon \in (0, \varepsilon_0)$, we have $0<\alpha_2^k < \beta/2$. Thus, $l=1$ is in the set.
\end{proof}
Based on Lemma \ref{lemma:L_k}, we can prove Lemma \ref{lemma:checkM}.
\begin{proof}[Proof of Lemma \ref{lemma:checkM}]
We first verify $\cM_1\in B_{\text{TV}}(\cM_0, \beta)$. When $\varepsilon\in(0, \varepsilon_0)$, by definition, we have $0<\alpha_1^k<\alpha_2^k<\beta/2$ and $p_0(x_k,a_1)-\beta/2\leq p_0^k<p_0(x_k,a_1)$. Then it holds that $$p_0(x_k,a_1)-\beta/2<p_0^k+\alpha_1^k < p_0(x_k,a_1)+\alpha_1^k < p_0(x_k,a_1)+\beta/2.$$ Thus, $|p_0^k+\alpha_1^k-p_0(x_k,a_1)|\leq \beta/2.$
For $2\leq l \leq L_k$, if $p_0(x_k,a_l)\geq p_0^k$, then $p_0(x_k,a_l)-p_0^k\leq p_0(x_k,a_1)-p_0^k\leq \beta/2$; otherwise, $p_0^k-p_0(x_k,a_l)\leq p_0^k+\alpha^k_2-p_0(x_k,a_l)\leq \beta/2$ (Lemma \ref{lemma:L_k}). Hence, $\cM_1\in B_{\text{TV}}(\cM_0, \beta)$. For each $\cM_{k,l}$, the validity directly follows Lemma \ref{lemma:L_k} and $\cM_1\in B_{\text{TV}}(\cM_0, \beta)$.
\end{proof}
Now we prove Lemma \ref{lemma:newversion}. Before that, we need to give some definitions. Let 
$$t^* = \frac{c_1}{(1-\gamma)^3\varepsilon^2}\log\Big(\frac{1}{4\delta}\Big),$$
where $c_1>0$ is to be determined later. We denote by $T_{k,l}$ the number of samples that algorithm $\sA$ calls from the generative model with input state $y_1(x_{k},a_{l})$ till $\sA$ stops (these sample calls are not necessarily consecutive). For every $k\in[K], 1<l\leq L_k$, we define the following events:
\begin{align}
A_{k,l} &= \{T_{k,l}\leq 4t^*\}, \quad E_{k,l} = \{ \sA \text{ outputs a policy } \pi \text{ with } \pi(x_k) = a_1\},\\
C_{k,l} &= \Big\{\max_{1\leq T_{k,l}\leq 4t^*} \big\vert p_0^k\cdot T_{k,l}- S_{k,l}(T_{k,l})\big\vert \leq\sqrt{16t^* \cdot p_0^k\cdot (1-p_0^k)\log(1/4\delta)}\Big\},
\end{align}
where $S_{k,l}(T_{k,l})$ is the sum of rewards (non-discounted) by calling the generative model $T_{k,l}$ times with input state $y_1(x_k,a_l)$. For these events, we have the following lemmas.

\begin{lemma}\label{lemma:eventA}
For any $k\in[K], 1<l\leq L_k$, if ~$\mathbb{E}_1[T_{k,l}]\leq t^*$, $\mathbb{P}_1(A_{k,l})> 3/4$.
\end{lemma}
\begin{proof}
$$t^*\geq \mathbb{E}_1[T_{k,l}]> 4t^*\mathbb{P}_1(T_{k,l}>4t^*)=4t^*(1-\mathbb{P}_1(T_{k,l}\leq 4t^*)).$$
Thus, $\mathbb{P}_1(A_{k,l})> 3/4$.
\end{proof} 

\begin{lemma}\label{lemma:eventC}
For any $k\in[K], 1<l\leq L_k$, $\mathbb{P}_1(C_{k,l})> 3/4$.
\end{lemma}
\begin{proof}
Let $\epsilon:=\sqrt{16t^* \cdot p_0^k\cdot (1-p_0^k)\log(1/4\delta)}$. When $1<l\leq L_k$, under hypothesis $\cM_1$, $p_{\cM_1}(x_k,a_l)=p_0^k$. By definition, the instant rewards from state $y_1(x_k,a_l)$ are i.i.d. Bernoulli$(p_0^k)$ random variables and
$p_0^k\cdot T_{k,l}-S_{k,l}(T_{k,l})$ is a martingale. Using Doob's inequality (\cite[Theorem 4.4.2]{durrett2019probability}), we have the following bound:
\begin{align}
    &\mathbb{P}_1\Big(\max_{1\leq T_{k,l}\leq 4t^*} \big\vert p_0^k T_{k,l}- S_{k,l}(T_{k,l})\big\vert \geq \sqrt{16t^*p_0^k (1-p_0^k)\log\frac{1}{4\delta}}\Big)
    \leq\frac{\mathbb{E}_1\Big[\big(4t^*\cdot p_0^k- S_{k,l}(4t^*)\big)^2\Big]}{16t^* \cdot p_0^k(1-p_0^k)\log(1/4\delta)}.
\end{align}
Since $\mathbb{E}_1[(4t^*\cdot p_0^k- S_{k,l}(4t^*))^2]=4t^*p_0^k(1-p_0^k)$ and $\delta<1/40$, we obtain that
\begin{align}
    \mathbb{P}_1(C_{k,l})\geq 1-1/(4\log(1/4\delta))>3/4.
\end{align}
\end{proof}

Since $\sA$ is $(\cM_0, \beta, \varepsilon, \delta)$-correct, it should return a policy $\pi$ such that when $\cM=\cM_1$, $\pi(x_{k})=a_1$ for every $k\in[K]$ with probability at least $1-\delta$, i.e. $\mathbb{P}_1\big(~E_{k,l}, ~\text{for all } k\in[K], 1<l\leq L_k~ \big)\geq 1-\delta>3/4$. We define the event $\cE_{k,l}:=A_{k,l}\cap E_{k,l} \cap C_{k,l}$. Combining the results above, it holds that
\begin{align}
\mathbb{P}_1(\cE_{k,l})>1-3/4=1/4, \quad \forall~ k\in[K],~ 1<l\leq L_k.
\end{align}
Based on the above results, we show the correctness of Lemma \ref{lemma:newversion}.
\begin{proof}[Proof of Lemma \ref{lemma:newversion}]
Given $k\in[K]$ and $1<l\leq L_k$, we denote by $W$ the length-$T_{k,l}$ random sequence of the instant rewards by calling the generative model $T_{k,l}$ times with the input state $y_1(x_k,a_l)$. If $\cM=\cM_1$, this is an i.i.d. Bernoulli$(p_0^k)$ sequence; if $\cM=\cM_{k,l}$, this is an i.i.d Bernoulli$(p_0^k+\alpha_2^k)$ sequence. We define the likelihood function $L_{k,l}$ as
$$L_{k,l}(w) = \mathbb{P}_{k,l}(W=w)$$
for every possible realization $w$. We simplify the previous notation $S_{k,l}(T_{k,l})$ as $S_{k,l}$. Then we compute the following likelihood ratio
\begin{align}
\frac{L_{k,l}(W)}{L_1(W)}=& \frac{(p_0^k+\alpha_2^k)^{S_{k,l}}(1-p_0^k-\alpha_2^k)^{T_{k,l}-S_{k,l}}}{(p_0^k)^{S_{k,l}}(1-p_0^k)^{T_{k,l}-S_{k,l}}}= \left(1+\frac{\alpha_2^k}{p_0^k}\right)^{S_{k,l}}\left(1-\frac{\alpha_2^k}{1-p_0^k}\right)^{T_{k,l}-S_{k,l}}\\
=& \left(1+\frac{\alpha_2^k}{p_0^k}\right)^{S_{k,l}}\left(1-\frac{\alpha_2^k}{1-p_0^k}\right)^{S_{k,l}\frac{1-p_0^k}{p_0^k}}\left(1-\frac{\alpha_2^k}{1-p_0^k}\right)^{T_{k,l}-S_{k,l}/p_0^k}.
\end{align}
Since $\gamma>0.4$ and $p_0^k \geq \frac{4\gamma-1}{3\gamma}$, we have $p_0^k>1/2$. By our choice of $\alpha_2^k$ and $\varepsilon$, it holds that $\alpha_2^k/(1-p_0^k)\leq \beta/4\in(0, 1/2]$ and $\alpha_2^k/p_0^k\leq \beta(1-p_0^k)/(4p_0^k)\in(0,1/2)$. With the fact that $\log (1-u) \geq-u-u^{2}$ for $u\in[0,1/2]$ and $\exp (-u) \geq 1-u$ for $u\in[0,1]$, we have that
\begin{align}
\bigg(1-\frac{\alpha_2^k}{1-p_0^k}\bigg)^{\frac{1-p_0^k}{p_0^k}}& \geq \exp \left(\frac{1-p_0^k}{p_0^k}\left(-\frac{\alpha_2^k}{1-p_0^k}-\Big(\frac{\alpha_2^k}{1-p_0^k}\Big)^{2}\right)\right)\\ &\geq\left(1-\frac{\alpha_2^k}{p_0^k}\right)\left(1-\frac{(\alpha_2^k)^{2}}{p_0^k(1-p_0^k)}\right).
\end{align}
Thus,
\begin{align} \frac{L_{k,l}(W)}{L_{1}(W)} & \geq\left(1-\frac{(\alpha_2^k)^2}{(p_0^k)^2}\right)^{S_{k,l}}\left(1-\frac{(\alpha_2^k)^{2}}{p_0^k\cdot(1-p_0^k)}\right)^{S_{k,l}}\bigg(1-\frac{\alpha_2^k}{1-p_0^k}\bigg)^{T_{k,l}-S_{k,l}/p_0^k} \\ & \geq\left(1-\frac{(\alpha_2^k)^2}{(p_0^k)^2}\right)^{T_{k,l}}\left(1-\frac{(\alpha_2^k)^{2}}{p_0^k\cdot (1-p_0^k)}\right)^{T_{k,l}}\bigg(1-\frac{\alpha_2^k}{1-p_0^k}\bigg)^{T_{k,l}-S_{k,l}/p_0^k}
\end{align}
due to $S_{k,l}\leq T_{k,l}$. Next, we proceed on the event $\cE_{k,l}$. By definition, if $\cE_{k,l}$ occurs, $A_{k,l}$ also occurs. Using $\log(1-u)\geq -2u$ for $u\in[0,1/2]$, it follows that
\begin{align}
   \left(1-\frac{(\alpha_2^k)^2}{(p_0^k)^2}\right)^{T_{k,l}} \geq& \left(1-\frac{(\alpha_2^k)^2}{(p_0^k)^2}\right)^{4t^*} \geq \exp \left(-8t^* \frac{(\alpha_2^k)^2}{(p_0^k)^2}\right)\\
   = &\exp \left(\frac{-8c_1\log(1/4\delta)}{(1-\gamma)^3\varepsilon^2}\cdot\frac{16(1-\gamma p_0^k)^4\varepsilon^2}{\gamma^2(p_0^k)^2}\right)\\
\geq &\exp \Big(-128 c_1 \log(1/4\delta) \frac{(1-\frac{4\gamma-1}{3})^4}{(1-\gamma)^3\gamma^2(\frac{4\gamma-1}{3\gamma})^2} \Big)\\
= &\exp \Big(-128 c_1 \log(1/4\delta)\frac{256(1-\gamma)}{9(4\gamma-1)^2} \Big)\\
\geq &\exp \Big(-128 c_1 \log(1/4\delta)\frac{256}{9*0.6} \Big)\geq\left(4\delta\right)^{6100c_1},
\end{align}
where the second line follows $p_0^k\geq \frac{4\gamma-1}{3\gamma}$. Using $\log(1-u)\geq -2u$ for $u\in[0,1/2]$, we also obtain
\begin{align}
    \left(1-\frac{(\alpha_2^k)^{2}}{p_0^k\cdot(1-p_0^k)}\right)^{T_{k,l}} &\geq \left(1-\frac{(\alpha_2^k)^{2}}{p_0^k\cdot(1-p_0^k)}\right)^{4t^*}\geq \exp \left(-8t^* \frac{(\alpha_2^k)^{2}}{p_0^k(1-p_0^k)}\right)\\
    &=\exp \left(-8\frac{c_1}{(1-\gamma)^3\varepsilon^2}\log(1/4\delta)\frac{16(1-\gamma p_0^k)^4\varepsilon^2}{\gamma^2p_0^k(1-p_0^k)}\right)\\
    &\geq\exp \Big(-128 c_1 \log(1/4\delta) \frac{(1-\frac{4\gamma-1}{3})^4}{(1-\gamma)^3\gamma^2(\frac{4\gamma-1}{3\gamma})\min\{\frac{1-\gamma}{3\gamma}, \beta/2\}} \Big)\\
    &=\exp \Big(-128 c_1 \log(1/4\delta)\frac{256}{27\gamma(4\gamma-1)}\cdot \frac{1-\gamma}{\min\{\frac{1-\gamma}{3\gamma}, \beta/2\}}\Big)\\
    &\geq \exp \Big(-128 c_1 \log(1/4\delta)\frac{256\cdot 20}{27\gamma(4\gamma-1)}\Big)\\
    &\geq \exp \Big(-128 c_1 \log(1/4\delta)\frac{5120}{27*0.4*0.6} \Big)\geq\left(4\delta\right)^{102000c_1},
\end{align}
where the third line follows $1-p_0^k\geq \min\{1-\frac{4\gamma-1}{3\gamma}, 1-p_0(x_k,a_1)+\beta/2\}\geq \min\{\frac{1-\gamma}{3\gamma}, \beta/2\}$ and the fifth line is due to $\frac{1-\gamma}{\min\{\frac{1-\gamma}{3\gamma}, \beta/2\}}\leq\max\{3\gamma, 20\}=20$ (since $\gamma> 1-2\beta$).
Further, when $\cE_{k,l}$ occurs, $A_{k,l}$ and $C_{k,l}$ both occur. Therefore, following similar steps, we have
\begin{align}
\left(1-\frac{\alpha_2^k}{1-p_0^k}\right)^{T_{k,l}-S_{k,l}/{p_0^k}} & \geq\left(1-\frac{\alpha_2^k}{1-p_0^k}\right)^{\max_{1\leq T_{k,l}\leq 4t^*} |T_{k,l}-S_{k,l}/p_0^k|}\\
&\geq\left(1-\frac{\alpha_2^k}{1-p_0^k}\right)^{\sqrt{16t^*\frac{1-p_0^k}{p_0^k}\log(1/4\delta)}} \\
& \geq \exp \left(-\sqrt{64\frac{(\alpha_2^k)^2}{p_0^k(1-p_0^k)}t^*\log(1/4\delta)}\right)
\geq\left(4\delta \right)^{\sqrt{810000c_1}}.
\end{align}
In total, we have $L_{k,l}(W)/L_{1}(W)\geq(4\delta)^{108100c_1+\sqrt{810000c_1}}$.
By taking $c_1$ small enough, e.g. $c_1=5e^{-7}$, we have $L_{k,l}(W)/L_{1}(W) > 4\delta.$
By a change of measure,
\begin{equation}\label{eq:p2}
\mathbb{P}_{k,l}(E_{k,l})\geq \mathbb{P}_{k,l}(\cE_{k,l})=\mathbb{E}_{k,l}[\mathbf{1}_{\cE_{k,l}}]=\mathbb{E}_1\left[\frac{L_{k,l}(W)}{L_1(W)}\mathbf{1}_{\cE_{k,l}}\right]> 4\delta*1/4=\delta.
\end{equation}
\end{proof}
Finally, we can show Theorem \ref{thm:main_lower}.
\begin{proof}[Proof of the Lower Bound in Theorem \ref{thm:main_lower}]
Since $\sA$ is $(\cM^0,\beta,\varepsilon, \delta)$-correct, under hypothesis $\cM_{k,l}$, $\sA$ should produce a policy $\pi$ such that $\pi(x_k)=a_l$ with probability $\geq 1-\delta$. Thus, we should have $\mathbb{P}_{k,l}(E_{k,l})<\delta$ for all $k\in[K], 1<l\leq L_k$. From Lemma~\ref{lemma:newversion}, it requires $\mathbb{E}_1[T_{k,l}]>t^*$ for all $k\in[K], 1<l\leq L_k$. In total, we need $\Omega\left(\frac{\sum_{k\in[K]}L_k}{(1-\gamma)^3\varepsilon^2}\log(1/\delta)\right)$ samples. By definition of $L_k$, for every $k\in[K]$, we have
\begin{align}\label{eq:L_k}
\{a_l, l\leq L_k\} &= \cA^{x_k}_{\cM_0}\Big(\frac{1}{1-\gamma p_0(x_k,a_1)}-\frac{1}{1-\gamma (p_0^k+\alpha_2^k-\beta/2)}\Big)\\
&=\cA^{x_k}_{\cM_0}\Big(V^*(x_k) - \frac{1}{1-\gamma p_0^k - 4(1-\gamma p_0^k)^2\varepsilon+\beta\gamma/2}\Big).
\end{align}
If $p_0^k=\frac{4\gamma-1}{3\gamma}$, i.e. $p_0(x_k,a_1)-\beta/2\leq\frac{4\gamma-1}{3\gamma}$, then 
\begin{align}\label{eq:Lk1}
    L_k &= \Big\vert\cA^{x_k}_{\cM_0}\Big(V^*(x_k) - \frac{9}{12(1-\gamma)-64(1-\gamma)^2\varepsilon+4.5\beta\gamma}\Big) \Big\vert.
    \end{align}
If $p_0^k=p_0(x_k,a_1)-\beta/2$, i.e. $p_0(x_k,a_1)-\beta/2 > \frac{4\gamma-1}{3\gamma}$, then
 \begin{align}\label{eq:Lk2}
    L_k &= \Big\vert\cA^{x_k}_{\cM_0}\Big(V^*(x_k) - \frac{1}{1-\gamma p_0(x_k,a_1)+\gamma\beta - 4\varepsilon(1-\gamma p_0(x_k,a_1)+\gamma\beta/2)^2}\Big)\Big\vert\\
    &=\Big\vert\cA^{x_k}_{\cM_0}\Big(V^*(x_k) - \frac{V^*(x_k)^2}{V^*(x_k)+\gamma\beta (V^*(x_k))^2 - 4\varepsilon(1+\gamma\beta V^*(x_k)/2)^2}\Big)\Big\vert
    \end{align}   
  Combining Equation \eqref{eq:Lk1} and \eqref{eq:Lk2}, we have
$$\Omega\left(\frac{\sum_{k\in[K]}L_k}{(1-\gamma)^3\varepsilon^2}\log(1/\delta)\right)=\Omega\left(\frac{\sum_{s\in\cS'}|\cA^{s}_{\cM_0}(\underline{C}_s)|}{(1-\gamma)^3\varepsilon^2}\log(1/\delta)\right),$$
which concludes our proof of the lower bound result in Theorem \ref{thm:main_lower}. 
\end{proof}
In Corollary \ref{coro:worst}, we have the worst scenario, where $L_k=|\cA|$. This is possible if there is no gap between the values in $\cM_0$. Then, for any $c>0$, $|\cA^{x_k}_{\cM_0}(c)|=|\cA|$.

\end{document}